\documentclass{article}

\usepackage[preprint, nonatbib]{neurips_2023}
\usepackage[numbers]{natbib}

\usepackage[utf8]{inputenc} % allow utf-8 input
\usepackage[T1]{fontenc}    % use 8-bit T1 fonts
\usepackage[colorlinks=true,linkcolor=purple,citecolor=purple,pagebackref=true]{hyperref}       % hyperlinks
\renewcommand*{\backrefalt}[4]{%
    \ifcase #1 \footnotesize{(Not cited.)}%
    \or        \footnotesize{(Cited on page~#2.)}%
    \else      \footnotesize{(Cited on pages~#2.)}%
    \fi}
\usepackage{url}            % simple URL typesetting
\usepackage{booktabs}       % professional-quality tables
\usepackage{amsfonts}       % blackboard math symbols
\usepackage{nicefrac}       % compact symbols for 1/2, etc.
\usepackage{microtype}      % microtypography
\usepackage{xcolor}         % colors
\usepackage{times}
\usepackage{booktabs}       % professional-quality tables
\usepackage{amssymb,amsmath,amsthm}
\usepackage{graphicx}
\usepackage{rotating}
\usepackage[ruled]{algorithm2e}
\usepackage{pifont} % checkmark and cross
\usepackage{makecell}
\usepackage{capt-of} % for captionof
\usepackage{tcolorbox} % for box around the algorithms
\usepackage{bbm}
\usepackage{times}
\usepackage[utf8]{inputenc}
\usepackage[T1]{fontenc}
\usepackage{amsfonts}
\usepackage{nicefrac}
\usepackage{natbib}
\usepackage{amsmath,centernot}
\usepackage{amsthm}
\usepackage{amssymb}
\usepackage{bbm}
\usepackage{dsfont}
\usepackage{color}
\definecolor{yale}{RGB}{14,77,146}
\usepackage{natbib}
\usepackage{multirow}
\usepackage{mdframed}
\usepackage{accents}
\newcommand{\ubar}[1]{\underaccent{\bar}{#1}}
\usepackage{pifont}% http://ctan.org/pkg/pifont
\usepackage{babel}
\makeatletter
\DeclareRobustCommand*\cal{\@fontswitch\relax\mathcal}
\makeatother

\usepackage{pgf,tikz,pgfplots}
\pgfplotsset{compat=1.15}
\usepackage{mathrsfs}
\usetikzlibrary{arrows}

\definecolor{ududff}{rgb}{0.30196078431372547,0.30196078431372547,1}
\definecolor{xdxdff}{rgb}{0.49019607843137253,0.49019607843137253,1}
\definecolor{uuuuuu}{rgb}{0.26666666666666666,0.26666666666666666,0.26666666666666666}

\theoremstyle{plain}
\newtheorem{thm}{\protect\theoremname}
\theoremstyle{plain}
\newtheorem{cor}[thm]{\protect\corname}
\theoremstyle{plain}

\theoremstyle{plain}
\newtheorem{prop}[thm]{\protect\propositionname}
\theoremstyle{plain}
\newtheorem{lem}[thm]{\protect\lemmaname}
\theoremstyle{plain}

\newtheorem{rem}{\protect\remarkname}
\theoremstyle{plain}

\theoremstyle{plain}

\providecommand{\algorithmname}{Algorithm}
\providecommand{\assumptionname}{Assumption}
\providecommand{\lemmaname}{Lemma}
\providecommand{\theoremname}{Theorem}
\providecommand{\corname}{Corollary}
\providecommand{\propositionname}{Proposition}
\providecommand{\remarkname}{Remark}
\providecommand{\definitionname}{Definition}

\newcommand{\cA}{\mathcal{A}}
\newcommand{\cB}{\mathcal{B}}

\newcommand{\cE}{\mathcal{E}}
\newcommand{\cF}{\mathcal{F}}
\newcommand{\cG}{\mathcal{G}}

\newcommand{\cO}{\mathcal{O}}

\newcommand{\cX}{\mathcal{X}}

\newcommand{\cZ}{\mathcal{Z}}
\newcommand{\EE}{\mathbb{E}}
\newcommand{\NN}{\mathbb{N}}
\newcommand{\PP}{\mathbb{P}}
\newcommand{\RR}{\mathbb{R}}

\newcommand{\diag}{\textrm{diag}}
\newcommand{\indicator}{\mathds{1}}

\newcommand{\norm}[1]{\left\lVert#1\right\rVert}

\usepackage{mylivemacros}
\RequirePackage{amsthm}
\RequirePackage{mathtools} % incl. amsmath
\RequirePackage{amssymb} % incl. amsfonts
\RequirePackage{bbm}
\RequirePackage{bm}
\RequirePackage{xfrac}

\usepackage{lineno}
\title{Spectral Entry-wise Matrix Estimation  
for Low-Rank Reinforcement Learning}

\author{%
  Stefan Stojanovic\\
  EECS\\
  KTH, Stockholm, Sweden\\
  \texttt{stesto@kth.se} \\
  \And
  Yassir Jedra\\
  EECS\\
  KTH, Stockholm, Sweden\\
  \texttt{jedra@kth.se} \\
  \AND
  Alexandre Proutiere\\
  EECS\\
    KTH, Stockholm, Sweden\\
  \texttt{alepro@kth.se} \\
}

\usepackage{mylivemacros}

\begin{document}
\maketitle

\begin{abstract}
We study matrix estimation problems arising in reinforcement learning (RL) with low-rank structure. In low-rank bandits, the matrix to be recovered specifies the expected arm rewards, and for low-rank Markov Decision Processes (MDPs), it may for example characterize the transition kernel of the MDP. In both cases, each entry of the matrix carries important information, and we seek estimation methods with low entry-wise error. Importantly, these methods further need to accommodate for inherent correlations in the available data (e.g. for MDPs, the data consists of system trajectories). We investigate the performance of  simple spectral-based matrix estimation approaches: we show that they efficiently recover the singular subspaces of the matrix and exhibit nearly-minimal entry-wise error. These new results on low-rank matrix estimation make it possible to devise reinforcement learning algorithms that fully exploit the underlying low-rank structure. We provide two examples of such algorithms: a regret minimization algorithm for low-rank bandit problems, and a best policy identification algorithm for reward-free RL in low-rank MDPs. Both algorithms yield state-of-the-art performance guarantees.
\end{abstract}
\section{Introduction}\label{sec:intro}

Learning succinct representations of the reward function or of the system state dynamics in bandit and RL problems is empirically known to significantly accelerate the search for efficient policies \cite{laskin20a,stooke21a,chandak2023representations}. It also comes with interesting theoretical challenges. The design of algorithms learning and leveraging such representations and with provable performance guarantees has attracted considerable attention recently, but remains largely open. In particular, significant efforts have been made towards such design when the representation relies on a low-rank structure. In bandits, assuming such a structure means that the arm-to-reward function can be characterized by a low-rank matrix \cite{kveton2017stochastic,jun2019,bayati2022speed,jain2022online}. In MDPs, it implies that the reward function, the $Q$-function or the transition kernels are represented by low-rank matrices \cite{sun2019,agarwal2020flambe,Modi21,uehara2022,ren2022spectral}. In turn, the performance of algorithms exploiting low-rank structures is mainly determined by the accuracy with which we are able to estimate these matrices. 

In this paper, we study matrix estimation problems arising in low-rank bandit and RL problems. Two major challenges are associated with these problems. (i) The individual entries of the matrix carry important operational meanings  (e.g. in bandits, an entry could correspond to the average reward of an arm), and we seek estimation methods with low entry-wise error. Such requirement calls for a fine-grained analysis, typically much more involved than that needed to only upper bound the spectral or Frobenius norm of the estimation error \cite{fan2018eigenvector, eldridge2018unperturbed, cape2019two, abbe2020entrywise, srebro2005rank, chen2020noisy, shamir2014matrix, negahban2012restricted}. (ii) Our estimation methods should further accommodate for inherent correlations in the available data (e.g., in MDPs, we have access to system trajectories, and the data is hence Markovian). We show that, essentially, spectral methods successfully deal with these challenges. 

{\bf Contributions.} 1) We introduce three matrix estimation problems. The first arises in low-rank bandits. The second corresponds to scenarios in RL where the learner wishes to estimate the (low-rank) transition kernel of a Markov chain and to this aim, has access to a generative model. The last problem is similar but assumes that the learner has access to system trajectories only, a setting referred to as the forward model in the RL literature. For all problems, we establish strong performance guarantees for simple spectral-based estimation approaches: these efficiently recover the singular subspaces of the matrix and exhibit nearly-minimal entry-wise error. To prove these results, we develop and combine involved leave-one-out arguments and Poisson approximation techniques (to handle the correlations in the data).

2) We apply the results obtained for our first matrix estimation problem to devise an efficient regret-minimization algorithm for low-rank bandits. We prove that the algorithm enjoys finite-time performance guarantees, with a regret at most roughly scaling as $(m+n)\log^3(T)\bar{\Delta}/\Delta_{\min}^2$ where $(m,n)$ are the reward matrix dimensions, $T$ is the time horizon, $\bar{\Delta}$ is the average of the reward gaps between the best arm and all other arms, and $\Delta_{\min}$ is the minimum of these gaps. 

3) Finally, we present an algorithm for best policy identification in low-rank MDPs in the reward-free setting. The results obtained for the second and last matrix estimation problems imply that our algorithm learns an $\epsilon$-optimal policy for any reward function using only a number of samples scaling as $O({nA/\epsilon^2})$ up to logarithmic factors, where $n$ and $A$ denote the number of states and actions, respectively. This sample complexity is mini-max optimal \cite{jedra2023nearly}, and illustrates the gain achieved by leveraging the low-rank structure (without this structure, the sample complexity would be $\Omega(n^2A/\epsilon^2)$).

{\bf Notation.} For any matrix $A\in \mathbb{R}^{m\times n}$, $A_{i,:}$ (resp. $A_{:,j}$) denotes its $i$-th row (resp. its $j$-th column), $A_{\min}=\min_{(i,j)}A_{i,j}$ and $A_{\max}=\max_{(i,j)}A_{i,j}$. We consider the following norms for matrices: $\|A\|$ denotes the spectral norm, $\| A\|_{1\to\infty}=\max_{i\in [m]} \| A_{i,:}\|_1$, $\| A\|_{2\to\infty}=\max_{i\in [m]} \| A_{i,:}\|_2$, and finally $\| A\|_{\infty}=\max_{(i,j)\in [m]\times [n]}|A_{i,j}|$. If the SVD of $A$ is $U \Sigma V^\top$, we denote by $\mathrm{sgn}(A)=UV^\top$ the matrix sign function of $A$ (see Definition 4.1 in \cite{chen2021spectral}). ${\cal O}^{r\times r}$ denotes the set of $(r\times r)$ real orthogonal matrices. For any finite set ${\cal S}$, let ${\cal P}({\cal S})$ be the set of distributions over ${\cal S}$. The notation $a(n,m,T) \lesssim b(n,m,T)$ (resp. $a(n,m,T)=\Theta(b(n,m,T))$) means that there exists a universal constant $C>0$ (resp. $c,C>0$) such that $a(n,m,T) \le C b(n,m,T)$ (resp. $cb(n,m,T)\le a(n,m,T) \le C b(n,m,T)$) for all $n,m,T$. Finally, we use $a\wedge b = \min(a,b)$ and $a\vee b = \max(a,b)$.
%For any vector $v$, $v^{-1}$ denotes the vector formed by the entrywise inverse of $v$.

\section{Models and Objectives}\label{sec:sampling}

Let $M\in \mathbb{R}^{m\times n}$ be an unknown rank $r$ matrix that we wish to estimate from $T$ noisy observations of its entries. We consider matrices arising in two types of learning problems with low-rank structure, namely low-rank bandits and RL. The SVD of $M$ is $U\Sigma V^\top$ where the matrices $U\in\RR^{m\times r}$ and $V\in\RR^{n\times r}$ contain the left and right singular vectors of $M$, respectively, and $\Sigma=\diag(\sigma_1,\dots,\sigma_r)$. We assume without loss of generality that the singular values have been ordered, i.e., $\sigma_1 \ge \ldots \ge \sigma_r$. The accuracy of our estimate $\hM$ of $M$ will be assessed using the following criteria:
\begin{itemize}
\item[{\it (i)}] {\it Singular subspace recovery.} Let the SVD of $\hM$ be $\hU\widehat{\Sigma}\hV^\top$. To understand how well the singular subspaces of $M$ are recovered, we will upper bound $\min_{O\in {\cal O}^{r\times r}}\| U -\hU O\|_{2\to\infty}$ and $\min_{O\in {\cal O}^{r\times r}}\| V -\hV O\|_{2\to\infty}$ (the $\min_{O\in {\cal O}^{r\times r}}$ problem corresponds to the orthogonal Procrustes problem and its solution aligns $\hU$ and $U$ as closely as possible, see Remark 4.1 in \cite{chen2021spectral}).
\item[{\it (ii)}] {\it Matrix estimation.} To assess the accuracy of $\hM$, we will upper bound the row-wise error $\| \hM -M\|_{1\to\infty}$ or $\| \hM -M\|_{2\to\infty}$, as well as the entry-wise error $\| \hM -M\|_\infty$ (the spectral error $\| \hM -M\|$ is easier to deal with and is presented in appendix only). 
\end{itemize}
We introduce two classical quantities characterizing the heterogeneity and incoherence of the matrix $M$ \cite{candes2010power, recht2011simpler}. Let $\kappa=\sigma_1/\sigma_r$, and let $\mu(U)=\sqrt{m/r}\|U\|_{2\to\infty}$ (resp. $\mu(V)=\sqrt{n/r}\|V\|_{2\to\infty}$) denote the row-incoherence (resp. column-incoherence) parameter of $M$. Let $\mu=\max\{\mu(U),\mu(V)\}$. Next, we specify the matrices $M$ of interest in low-rank bandits and RL, and the way the data used for their estimation is generated.  

{\bf Model I: Reward matrices in low-rank bandits.} For bandit problems, $M$ corresponds to the average rewards of various arms. To estimate $M$, the learner has access to data sequentially generated as follows. In each round $t=1,\ldots, T$, an arm $(i_t,j_t)\in [m]\times [n]$ is randomly selected (say uniformly at random for simplicity) and the learner observes $M_{i_t,j_t} + \xi_t$, an unbiased sample of the corresponding entry of $M$. $(\xi_t)_{t\ge 1}$ is a sequence of zero-mean and bounded random variables. Specifically, we assume that for all $t\ge 1$, $|\xi_t |\le c_1\|M\|_\infty$ a.s., for some constant $c_1>0$.

%\begin{assumption}[Incoherence]
 %  Define the incoherence constants $\mu(U) = \Vert U\Vert_{2 \to \infty} \sqrt{n/r}$ and $\mu(V) = \Vert V\Vert_{2 \to \infty} \sqrt{n/r}$. There exists $\mu > 0$, such that
  %  \begin{align*}
   %     \max (\mu(U), \mu(V)) \le \mu
   % \end{align*}
   % with $r$ being the rank of $M^\star$.
%\end{assumption}

\paragraph{Model II: Transition matrices in low-rank MDPs.} In low-rank MDPs, we encounter Markov chains whose transition matrices have low rank $r$ (refer to Section \ref{sec:mdps} for details). Let $P\in \mathbb{R}^{n\times n}$ be such a transition matrix. We assume that the corresponding Markov chain is irreducible with stationary distribution $\nu$. The objective is to estimate $P$ from the data consisting of samples of transitions of the chain. More precisely, from the data, we will estimate the {\it long-term frequency matrix} $M=\mathrm{diag}(\nu)P$ ($M_{ij}$ is the limiting proportion of transitions from state $i$ to state $j$ as the trajectory grows large). Observe that $M$ is of rank $r$, and that $P_{i,:}=M_{i,:}/\| M_{i,:}\|_1$. To estimate $M$, the learner has access to the data $(x_1,\ldots,x_T)\in [n]^T$ generated according to one of the following two models.
%, classically used in RL.
\begin{itemize}
\item[(a)] In the {\it generative} model, for any $t\in [T]$, if $t$ is odd, $x_t$ is selected at random according to some distribution $\nu_0$, and $x_{t+1}$ is sampled from $P_{x_t,:}$.
\item[(b)] In the {\it forward} model, the learner has access to a trajectory $(x_1,\ldots, x_T)$ of length $T$ of the Markov chain, where $x_1\sim \nu_0$ and for any $t\ge 1$, $x_{t+1}\sim P_{x_t,:}$. 
\end{itemize}

%\begin{rem}[Correlated noise] In classical matrix completion problems \cite{candes2010noise,keshavan2010matrix,chen2020}, the learner observes $M_{ij}+\xi_{ij}$ with probability $p$ for any entry $(i,j)$, where $\xi_{ij}$ are independent across entries. Even for our reward matrix estimation problem, the observation noise is correlated across entries (simply because the noise level for a given entry $(i,j)$ depends on the number of times this entry has been sampled and observed, which in turn depends on the numbers of times the other entries are sampled). These correlations become stronger in our transition matrix estimation problem where the sequence of sampled entries forms a Markov chain.     
%\end{rem}
\section{Matrix Estimation via Spectral Decomposition}\label{sec:etimation}

In the three models (Models I, II(a) and II(b)), we first construct a matrix $\tM$ directly from the data, and from there, we build our estimate $\hM$, typically obtained via spectral decomposition, i.e., by taking the best rank-$r$ approximation of $\tM$. In the remaining of this section, we let $\hU\widehat{\Sigma}\hV^\top$ denote the SVD of $\hM$. Next, we describe in more details how $\hM$ is constructed in the three models, and analyze the corresponding estimation error.

\subsection{Reward matrices}

For Model I, for $t=1,\ldots,T$, we define $\tM_t=\big( (M_{i_t, j_t}+\xi_t)\indicator_{\lbrace (i, j)=(i_t, j_t) \rbrace}\big)_{i,j \in [m] \times [n]}$ and $\tM=\frac{nm}{T} \sum_{t=1}^T \tM_t$. Let $\hM$ denote the best rank-$r$ approximation of $\tM$. 

\begin{thm}\label{thm:reward}
Let $\delta>0$. We introduce: 
$$
\cB =  \sqrt{\frac{nm}{T}} \left( \sqrt{(n+m)\log\left(\frac{e(n+m)T}{\delta} \right)} + \log^{3/2}\left(\frac{e(n+m)T}{\delta} \right) \right).
$$     
Assume that $T \ge c \mu^4 \kappa^2 r^2 (n+m)\log^{3}\left( e(m+n)T/\delta\right)$ for some universal constant $c>0$. Then there exists a universal constant $C>0$ such that the following inequalities hold with probability at least $1-\delta$: 
\begin{align*}
        (i) & \qquad \max\left(\Vert U - \widehat{U} (\widehat{U}^\top U)\Vert_{2 \to \infty}, \Vert V - \widehat{V} (\widehat{V}^\top V) \Vert_{2 \to \infty} \right) \le C \frac{(\mu^3 \kappa^2 r^{3/2})}{\sqrt{mn (n \wedge m)}}\cB, \\
        (ii) & \qquad \Vert \widehat{M} - M \Vert_{2 \to \infty} \le C \frac{ (\mu^3 \, \kappa^2  r^{3/2})}{\sqrt{m\wedge n}} \Vert M \Vert_\infty  \cB, \\
        (iii) &  \qquad \Vert \widehat{M} - M \Vert_{ \infty} \le C \left( \mu^{11/2} \, \kappa^2 r^{1/2} + \mu^3 \kappa r^{3/2} \frac{m+n}{\sqrt{mn}} \right) \frac{1}{(n \wedge m)} \Vert  M \Vert_\infty \cB.
    \end{align*}
\end{thm}

\begin{cor}(Homogeneous reward matrix) 
\label{corr:homogeneous1}
When $m = \Theta(n)$, $\kappa = \Theta(1)$, $\mu = \Theta(1)$, $\Vert M \Vert_\infty = \Theta(1)$, $r = \Theta(1)$, we say that the reward matrix $M$ is homogeneous. In this case, for any $\delta>0$, when $T \ge c (n+m)\log^{3}\big( e(m+n)T/\delta\big)$ for some universal constant $c>0$, we have with probability at least $1-\delta$:
    \begin{align*}
        &\max\left(\Vert U - \widehat{U}(\widehat{U}^\top U) \Vert_{2 \to \infty}, \Vert V - \widehat{V} (\widehat{V}^\top V) \Vert_{2 \to \infty} \right)  \lesssim \frac{1}{\sqrt{T}}  \log^{3/2}\left( \frac{(n+m)T}{\delta}\right),   \\
        & \Vert \widehat{M} - M  \Vert_{2 \to \infty} \lesssim \frac{(n+m)}{\sqrt{T}} \log^{3/2}\left( \frac{(n+m)T}{\delta}\right),\\
        &\Vert \widehat{M} - M \Vert_{\infty}  \lesssim \sqrt{\frac{(n+m)}{T}} \log^{3/2}\left( \frac{(n+m)T}{\delta}\right).
    \end{align*}
\end{cor}

For a homogeneous reward matrix, $\|U\|_{2\to\infty}=\Theta(1/\sqrt{m})$ and $\| M\|_{\infty}=\Theta(1)$, and hence, from the above corollary, we obtain estimates whose relative errors (e.g., $\Vert \widehat{M} - M \Vert_{\infty}/\| M\|_{\infty}$) scale at most as $\sqrt{m/T}$ up to the logarithmic factor.

We may also compare the results of the above corollary to those of Theorem 4.4 presented in \cite{chen2021spectral}. There, the data consists for each pair $(i,j)$ of a noisy observation $M_{i,j}+E_{i,j}$. The $E_{i,j}$'s are independent across $(i,j)$. This model is simpler than ours and does not include any correlation in the data. But it roughly corresponds to the case where $T=nm$ in our Model I. Despite having to deal with correlations, we obtain similar results as those of Theorem 4.4: for example, $\Vert \widehat{M} - M \Vert_{\infty}\lesssim \sqrt{1/(n+m)}$ (up to logarithmic terms) with high probability.

\subsection{Transition matrices under the generative model}
\label{subsec:transition_theorem_generative}

For Model II(a), the matrix $\tM$ records the empirical frequencies of the transitions: for any pair of states $(i,j)$, $\tM_{i,j} =  \frac{1}{\lfloor T/2\rfloor} \sum_{k=1}^{\lfloor T/2\rfloor}\indicator_{ \{ (x_{2k-1},x_{2k})=(i,j)\} }$. $\hM$ is the best rank-$r$ approximation of $\tM$ and the estimate $\hP$ of the transition matrix $P$ is obtained normalizing the rows of $\hM$: for all $i\in [n]$,
\begin{align}
\hP_{i,:} = \begin{cases}
                (\hM_{i,:})_+/ \| (\hM_{i,:})_+ \|_1, \quad &\textif \ \| {(\hM_{i,:})_+}\|_1 > 0,\\
                \frac{1}{n} \onevec_n, \quad &\textif\ \| (\hM_{i,:})_+\|_1 = 0.
                \end{cases}
    \label{def:P_hat_MC_general}
\end{align}
where $(\cdot)_+$ is the function applying $\max(0,\cdot)$ component-wise and $\onevec_n$ is the $n$-dimensional vector of ones. The next theorem is a simplified version and a consequence of a more general and tighter theorem presented in App. \ref{subsec:appendix_tight_thm_generative}. To simplify the presentation of our results, we define\\ $g(M,T,\delta)=n \log(\frac{n \sqrt{T}}{\delta})\max \left\{ \mu^6 \kappa^6 r^3  , \frac{\log (\frac{n\sqrt{T}}{\delta}) \indicator_{\{ \exists \ell: T\| M_{\ell,:} \|_{\infty}\leq 1\}}}{\log(1+\frac{1}{T\| M\|_{\infty}})} \right\}$. 
%{ \color{blue} Alternatively we can write the sample complexity bound $ T \ge  \frac{C}{n\Vert M \Vert_\infty} \left( \mu^6 \kappa^6 r^3  \log\left( \frac{nT}{\delta}\right)  + \left(n^2 \Vert M \Vert_\infty\right) h( T \Vert M \Vert_\infty) \log^2\left( \frac{nT}{\delta}\right) \right)$ where we introduce $h(x) = \frac{\indicator{ \lbrace x \le 1 \rbrace}}{\log(1 + 1/x)}$. Then we clearly see that when $T \Vert M \Vert_\infty \ge 1$, the sample complexity simplifies to $T \ge  \frac{C\mu^6 \kappa^6 r^3}{n\Vert M \Vert_\infty}  \log\left( \frac{nT}{\delta}\right) $  and in the worst case $h(T\Vert M \Vert) \le 1/\log(2)$ which entails that the sample becomes $T \gtrsim n \left( \mu^6 \kappa^6 r^3  \log\left( \frac{nT}{\delta}\right)  +  \log^2\left( \frac{nT}{\delta}\right) \right) $}
%\\ 

\begin{thm}\label{thm:generative} Let $\delta>0$. Introduce ${\cal B} =  \mu\kappa \sqrt{(r\|M\|_{\infty} /T) \log ( {n\sqrt{T} / \delta})}$. Assume that we have  $(\nu_0)_{\min} = \min_{i\in [n]} (\nu_0)_i >0$. If (a) $n\ge c\log^2(n T^{3/2}/\delta)$ and (b) $T\ge cg(M,T,\delta)$ for some universal constant $c>0$, then there exists a universal constant $C>0$ such that the following inequalities hold with probability at least $1-\delta$: 
\begin{align*}
        (i)& \qquad \max\Big\{ \| U  - \hU (\hU^\top U)\|_{2\to\infty}, \| V - \hV (\hV^\top V)\|_{2\to\infty} \Big\} \le C\frac{\kappa \mu^2 r}{n\|M\|_{\infty}} {\cal B},\\
        (ii)& \qquad \| \hM - M\|_{2\to\infty} \le C\kappa {\cal B}, \ \ \ \| \hP - P\|_{1\to\infty} \le C  \frac{\kappa\sqrt{n}}{(\nu_0)_{\min}} {\cal B},\\
        (iii)& \qquad  \|\hM - M \|_{\infty} \le C\frac{\kappa  \mu^2 r}{\sqrt{n}}  {\cal B},\\
        (iv)& \qquad \|\hP - P\|_{\infty} \le  C \frac{{\cal B}}{(\nu_0)_{\min}} \left[ \sqrt{n}\kappa\frac{\|M\|_{\infty}}{ (\nu_0)_{\min}} + \left(1+ \frac{\kappa {\cal B}}{ \sqrt{n}\norm{M}_{\infty}}\right) \frac{\kappa\mu^2 r}{\sqrt{n}} \right],
    \end{align*}
where (iv) holds if in addition 
%{\color{blue} $T \ge c  \frac{\mu^2\kappa r \Vert P \Vert_\infty^2}{n\Vert M \Vert_\infty}  \left( \frac{(\nu_0)_{\max}}{(\nu_0)_{\min}}\right)^2  \log\left( \frac{nT}{\delta}\right)$}.
$T\geq c n \|M\|_{\infty}(\nu_0)_{\min}^{-2} r\mu^2 \kappa^4 \log(n\sqrt{T}/\delta)$
\end{thm}

Note that in theorem, the condition (a) on $n$ has been introduced just to simplify the expression of ${\cal B}$ (refer to App. \ref{subsec:appendix_tight_thm_generative} for a full statement of the theorem without this condition).

\begin{cor}(Homogeneous transition matrix)
\label{corr:homogeneous_model2}
When $\kappa = \Theta(1)$, $\mu = \Theta(1)$, $r= \Theta(1)$, $M_{\max} = \Theta(M_{\min})$, we say that the frequency matrix $M$ is homogeneous. If $T\ge c n\log(nT)$ for some universal constant $c>0$, then we have with probability at least $1- \min\lbrace n^{-2}, T^{-1} \rbrace$:
\begin{align*}
        &\qquad \max\Big\{  \|U  - \hU (\hU^\top U)\|_{2\to\infty}, \|V - \hV (\hV^\top V) \|_{2\to\infty} \Big\} \lesssim  \sqrt{{\log( nT)\over T}},\\
        &\qquad \|\hM - M\|_{2\to\infty} \lesssim \frac{1}{n}\sqrt{{\log( nT)\over T}},\ \| \hM - M\|_{\infty} \lesssim \frac{1}{n}\sqrt{{\log( nT)\over nT}},\\
        &\qquad \| \hP-P\|_{1\to\infty} \lesssim \sqrt{\frac{n\log( nT)}{T}},\ \| \hP-P\|_{\infty}\lesssim  \sqrt{{\log(nT)\over nT}}.
    \end{align*}
\end{cor}

For a homogeneous frequency matrix, $\|U\|_{2\to\infty}=\Theta(1/\sqrt{n})$, $\| M \|_{2\to\infty} =\Theta(1/{n\sqrt{n}})$, $\| M \|_{\infty}=\Theta(1/n^2)$, $\| P \|_{1\to\infty} =1 $, $\| P \|_{\infty}=\Theta(1/{n})$. Thus for all these metrics, our estimates achieve a relative error scaling at most as $\sqrt{n/T}$ up to the logarithmic factor.

\subsection{Transition matrices under the forward model}
\label{subsec:transition_theorem_forward}

For Model II(b), we first split the data into $\tau$ subsets of transitions: for $k=1,\ldots,\tau$, the $k$-th subset is $((x_{k},x_{k+1}),(x_{k+\tau},x_{k+1+\tau}),\dots,(x_{k+(T_{\tau}-1)\tau},x_{k+1+(T_{\tau}-1) \tau}))$ where $T_\tau=\lfloor T/\tau\rfloor$. By separating two transitions in the same subset, we break the inherent correlations in the data if $\tau$ is large enough. Now we let $\tM^{(k)}$ be the matrix recording the empirical frequencies of the transitions in the $k$-th subset: 
$\tM_{i,j}^{(k)}=  \frac{1}{T_{\tau}} \sum_{l=0}^{T_{\tau}-1} \indicator_{\{ (x_{k+l\tau},x_{k+1+l\tau})=(i,j)\} }$
for any pair of states $(i,j)$. Let $\hM^{(k)}$ be the best $r$-rank approximation of $\tM^{(k)}$. As in (\ref{def:P_hat_MC_general}), we define the corresponding $\hP^{(k)}$. Finally we may aggregate these estimates $\hM=\frac{1}{\tau}\sum_{k=1}^{\tau} \hM^{(k)}$ and $\hP=\frac{1}{\tau}\sum_{k=1}^{\tau} \hP^{(k)}$. We present below the performance analysis for the estimates coming from a single subset; the analysis of the aggregate estimates easily follows. 

For any $\varepsilon>0$, we define the $\varepsilon$-mixing time of the Markov chain with transition matrix $P$ as $\tau(\varepsilon) = \min \{ t\geq 1:\ \max_{1\leq i\leq n} \frac{1}{2} \| P^t_{i,:} - \nu^\top\|_1 \leq \varepsilon\}$, and its mixing time as $\tau^\star = \tau(1/4)$. The next theorem is a simplified version and a consequence of a more general and tighter theorem presented in App. \ref{subsec:appendix_tight_thm_forward}. To simplify the presentation, we define:\\
$
h(M,T,\delta)=n\tau^\star\log(\frac{n \sqrt{T}}{\delta}) \log(T\nu_{\min}^{-1} ) 
\max \left \{ \mu^6 \kappa^6 r^3  , \frac{\log^2( \frac{n\sqrt{T_{\tau}}}{\delta}) \indicator_{\{ \exists \ell: T_{\tau} \| M_{\ell,:} \|_{\infty}\leq 1 \}} }{\log^2(1+\frac{1}{T_{\tau}\| M\|_{\infty}}) } \right\}.
$

\begin{thm}
\label{thm:main_thm}
Let $\delta>0$. Assume that $\nu_{\min} = \min_{i\in [n]} \nu_i >0$ and that $\tau/(\tau^\star \log (T\nu_{\min}^{-1}))\in [c_1,c_2]$ for some universal constants $c_2> c_1\ge 2$. Introduce:
    \begin{align*}
        {\cal B} =  \mu\kappa \sqrt{\frac{r\tau^\star  \| M\|_{\infty}}{T} \log \left( \frac{n\sqrt{T_{\tau}}}{\delta}\right) \log\left(\frac{T}{\nu_{\min}}\right)}.
        %\label{eq:B_def_main_thm}
    \end{align*}
If (a) $n\ge c\tau^\star \log^{3/2}(n T^{3/2}/\delta) \log^{1/2}(T\nu_{\min}^{-1})$ and (b) $T\ge ch(M,T,\delta)$ for some universal constant $c>0$, then there exists a universal constant $C>0$ such that the following inequalities hold with probability at least $1-\delta$: 
\begin{align*}
        (i)& \qquad \max\Big\{ \| U  - \hU (\hU^\top U)\|_{2\to\infty}, \| V - \hV (\hV^\top V)\|_{2\to\infty} \Big\} \le C\frac{\kappa \mu^2 r}{n\|M\|_{\infty}} {\cal B},\\
        (ii)& \qquad \| \hM - M\|_{2\to\infty} \le C\kappa {\cal B}, \ \ \ \| \hP - P\|_{1\to\infty} \le C  \frac{\kappa\sqrt{n}}{\nu_{\min}} {\cal B},\\
        (iii)& \qquad  \|\hM - M \|_{\infty} \le C\frac{\kappa  \mu^2 r}{\sqrt{n}}  {\cal B},\\
        (iv)& \qquad \|\hP - P\|_{\infty} \le  C \frac{{\cal B}}{\nu_{\min}} \left[ \sqrt{n}\kappa\frac{\|M\|_{\infty}}{ \nu_{\min}} + \left(1+ \frac{\kappa {\cal B}}{ \sqrt{n}\norm{M}_{\infty}}\right) \frac{\kappa\mu^2 r}{\sqrt{n}} \right],
    \end{align*}
where (iv) holds if in addition $T\geq c n \|M\|_{\infty}\nu_{\min}^{-2} \tau^\star r\mu^2 \kappa^4 \log(n\sqrt{T}/\delta)\log(T\nu_{\min}^{-1})$.    
\end{thm}

Note that our guarantees hold when $\tau$ roughly scales as $\tau^\star \log (T\nu_{\min}^{-1})$. Hence to select $\tau$, one would need an idea of the latter quantity. It can be estimated typically using $\tau^{\star}\nu_{\min}^{-1}$ samples \cite{wolfer19} (which is small when compared to the constraint $T\ge ch(M,T,\delta)$ as soon as $\nu_{\min}=\Omega(1/n)$). Further observe that in the theorem, the condition (a) can be removed (refer to App. \ref{subsec:appendix_tight_thm_forward} for a full statement of the theorem without this condition).

% \begin{cor}(Homogeneous transition matrices)
% \label{corr:homogeneous_model3}
% Assume that $M$ is homogeneous (as defined in Corollary \ref{corr:homogeneous_model2}). Let $\tau=\log(Tn)$. If $\delta = \Omega(n^{-d})$ and $T\ge c n\log(n\sqrt{T}/\delta)\log(Tn)$ for some universal constants $c,d>0$, then we have with probability at least $1-\delta$:
% \begin{align*}
%     &\qquad \max\Big\{  \|U  - \hU (\hU^\top U)\|_{2\to\infty}, \|V - \hV (\hV^\top V) \|_{2\to\infty} \Big\} \lesssim \frac{1}{\sqrt{T}}\ell(n,T,\delta),\\
%     &\qquad \|\hM - M\|_{2\to\infty} \lesssim \frac{1}{n\sqrt{T}}\ell(n,T,\delta),\ \| \hM - M\|_{\infty} \lesssim \frac{1}{n\sqrt{nT}}\ell(n,T,\delta),\\
%     &\qquad \| \hP-P\|_{1\to\infty} \lesssim \sqrt{\frac{n}{T}}\ell(n,T,\delta),\ \| \hP-P\|_{\infty}\lesssim \frac{1}{\sqrt{nT}}\ell(n,T,\delta),
% \end{align*}
% where $\ell(n,T,\delta)=\sqrt{\log( \frac{n\sqrt{T}}{\delta}) \log(Tn)}$.
% \end{cor}
\begin{cor}(Homogeneous transition matrices)
\label{corr:homogeneous_model3}
Assume that $M$ is homogeneous (as defined in Corollary \ref{corr:homogeneous_model2}). Let $\tau=\log(Tn)$. If $T\ge c n\log^2(nT)$ for some universal constant $c>0$, then we have with probability at least $1-\min\{n^{-2},T^{-1}\}$:
\begin{align*}
    &\qquad \max\Big\{  \|U  - \hU (\hU^\top U)\|_{2\to\infty}, \|V - \hV (\hV^\top V) \|_{2\to\infty} \Big\} \lesssim \frac{1}{\sqrt{T}}\log(nT),\\
    &\qquad \|\hM - M\|_{2\to\infty} \lesssim \frac{1}{n\sqrt{T}}\log(nT),\ \| \hM - M\|_{\infty} \lesssim \frac{1}{n\sqrt{nT}}\log(nT),\\
    &\qquad \| \hP-P\|_{1\to\infty} \lesssim \sqrt{\frac{n}{T}}\log(nT),\ \| \hP-P\|_{\infty}\lesssim \frac{1}{\sqrt{nT}} \log(nT).
\end{align*}
\end{cor}

As for the generative model, for a homogeneous frequency matrix, our estimates achieve a relative error scaling at most as $\sqrt{n / T}$ up to the logarithmic factor for all metrics. Note that up to a logarithmic factor, the upper bound for 
$\| \hP-P\|_{1\to\infty}$ (and similarly for $\hM$) matches the minimax lower bound derived in \cite{zhang2019spectral}.

\subsection{Elements of the proofs}
\label{subsec:elements_proofs}
The proofs of the three above theorems share similar arguments. We only describe elements of the proof of Theorem \ref{thm:main_thm}, corresponding to the most challenging model. The most difficult result concerns the singular subspace recovery (the upper bounds (i) in our theorems), and it can be decomposed into the following three steps. The first two steps are meant to deal with the Markovian nature of the data. The third step consists in applying a leave-one-out analysis to recover the singular subspaces.  

{\it Step 1: Multinomial approximation of Markovian data.} We treat the matrix $\tM^{(k)}$ arising from one subset of data, and for simplicity, we remove the superscript $(k)$, i.e., $\tM=\tM^{(k)}$. Note that $T_\tau\tM$ is a matrix recording the numbers of transitions observed in the data for any pair of states: denote by $N_{i,j}$ this number for $(i,j)$. We approximate the joint distribution of $N=(N_{i,j})_{(i,j)}$ by a multinomial distribution with $n^2$ components and parameter $T_{\tau} M_{i,j}$ for component $(i,j)$. Denote by $Z=(Z_{i,j})_{(i,j)}$ the corresponding multinomial random variable. Using the mixing property of the Markov chain and the choice of $\tau$, we establish (see Lemma \ref{lemma:markov_to_multinomial} in App. \ref{app:poisson-approx}) that for any subset ${\cal Z}$ of $\{z\in \mathbb{N}^{n^2}: \sum_{(i,j)}z_{i,j}=T_\tau\}$, we have $\mathbb{P}[N\in {\cal Z}]\le 3\mathbb{P}[Z\in {\cal Z}]$. 

{\it Step 2: Towards Poisson random matrices with independent entries.} The random matrix $Z$ does not have independent entries. Independence is however a requirement if we wish to apply the leave-one-out argument. Consider the random matrix $Y$ whose entries are independent Poisson random variables with mean $T_\tau M_{i,j}$ for the $(i,j)$-th entry. We establish the following connection between the distribution of $Z$ and that of $Y$: for any ${\cal Z}\subset \mathbb{N}^{n^2}$, we have $\mathbb{P}[Z\in {\cal Z}]\le e\sqrt{T_{\tau}} \mathbb{P}[Y\in {\cal Z}]$. Refer to Lemma \ref{lemma:multinomial_to_poisson} in App. \ref{app:poisson-approx} for details. 

{\it Step 3: The leave-one-out argument for Poisson matrices.} Combining the two first steps provides a connection between the observation matrix $\tM$ and a Poisson matrix $Y$ with independent entries. This allows us to apply a leave-one-out analysis to $\tM$ as if it had independent entries (replacing $\tM$ by $Y$). The analysis starts by applying the standard dilation trick (see Section 4.10 in \cite{chen2021spectral}) so as to make $\tM$ symmetric. Then, we can decompose the error $\|U - \hU (\hU^\top U) \|_{2\to\infty}$ (see Lemma \ref{lemma:U_to_EU} in App. \ref{sec:appendix_singular_subspace_LOO}) into several terms. The most challenging of these terms is    $\|(M-\tM)(U - \hU (\hU^\top U))\|_{2\to\infty} = \max_{l\in[n]} \|(M_{l,:}-\tM_{l,:})(U - \hU (\hU^\top U))\|_{2}$ because of inherent dependence between $M-\tM$ and $U - \hU (\hU^\top U)$. The leave-one-out analysis allows us to decouple this statistical dependency. It consists in exploiting the row and column independence of matrix $\tM$ to approximate $\|(M_{l,:}-\tM_{l,:})(U - \hU (\hU^\top U))\|_{2}$ by $\|(M_{l,:}-\tM_{l,:})(U - \hUl ((\hUl)^\top U)\|_{2}$ where $\hUl$ is the matrix of eigenvectors of matrix $\tM^{(l)}$ obtained by zeroing the $l$-th row and column of $\tM$. By construction, $(M_{l,:}-\tM_{l,:})$ and $U - \hUl ((\hUl)^\top U)$ are independent, which simplifies the analysis. The proof is completed by a further appropriate decomposition of this term, combined with concentration inequalities for random Poisson matrices (see App. \ref{sec:appendix_conc_Poisson}).

\section{Regret Minimization in Low-Rank Bandits}\label{sec:bandits}

Consider a low-rank bandit problem with a homogeneous rank-$r$ reward matrix $M$. We wish to devise an algorithm $\pi$ with low regret. $\pi$ selects in round $t$ an entry $(i_t^\pi,j_t^\pi)$ based on previous observations, and receives as a feedback the noisy reward $M_{i_t^\pi,j_t^\pi}+\xi_t$. The regret up to round $T$ is defined by $R^\pi(T)=T M_{i^\star,j^\star} - \mathbb{E}[\sum_{t=1}^TM_{i_t^\pi,j_t^\pi}]$, where $(i^\star,j^\star)$ is an optimal entry. One could think of a simple Explore-Then-Commit (ETC) algorithm, where in the first phase entries are sampled uniformly at random, and where in a second phase, the algorithm always selects the highest entry of $\hM$ built using the samples gathered in the first phase and obtained by spectral decomposition. When the length of the first phase is $T^{2/3}(n+m)^{1/3}$, the ETC algorithm would yield a regret upper bounded by $O(T^{2/3}(n+m)^{1/3})$ for $T=\Omega((n+m)\log^3(n+m))$. 

To get better regret guarantees, we present SME-AE (Successive Matrix Estimation and Arm Elimination), an algorithm meant to identify the best entry as quickly as possible with a prescribed level of certainty. After the SME-AE has returned the estimated best entry, we commit and play this entry for the remaining rounds. The pseudo-code of SME-AE is presented in Algorithm \ref{algo:EE}. The algorithm runs in epochs: in epoch $\ell$, it samples $T_\ell$ entries uniformly at random among all entries (in $T_\ell$, the constant $C$ just depends on upper bounds of the parameters $\mu$, $\kappa$, and $\| M\|_\infty$, refer to App. \ref{app:bandits}); from these samples, a matrix $\hM^{(\ell)}$ is estimated and ${\cal A}_\ell$, the set of candidate arms, is pruned. The pruning procedure is based on the estimated gaps: $\widehat{\Delta}_{i,j}^{(\ell)}=\widehat{M}^{(\ell)}_{\star} - \widehat{M}^{(\ell)}_{i,j}$ where $\widehat{M}^{(\ell)}_{\star} = \max_{i,j} \widehat{M}^{(\ell)}_{i,j}$.

\begin{algorithm}[th]
    \SetAlgoLined
    \KwIn{Arms $[m] \times [n]$, confidence level $\delta$}
    $\ell = 1$ \;
    $\cA_1 = [m] \times [n]$\;
    \While{$\vert \cA_\ell \vert > 1$}{
        $\delta_\ell = \delta/\ell^2$\;
        $T_\ell = \left\lceil  C \left(2^{\ell + 2}\right)^2   (m+n) \log^3\left(2^{2\ell + 4}(m+n)/\delta_\ell \right) \right\rceil$ \;
        Sample uniformly at random $T_\ell$ entries from ${\cal A}_1$: $(M_{i_t,j_t}+\xi_t)_{t=1, \dots, T_\ell}$ \;
        Estimate $\widehat{M}^{(\ell)}$ via spectral decomposition as described in Section 3.1 \;
        $\cA_{\ell+1} = \left\lbrace (i, j) \in \cA_\ell: \widehat{\Delta}_{i,j}^{(\ell)} \le 2^{-(\ell+2)} \right\rbrace$;\
        $\ell = \ell + 1$\;
    }
    %Recommend the remaining arm\;
    \KwOut{Recommend the remaining pair  $(\hat{\imath}_\tau,\hat{\jmath}_\tau)$ in ${\cal A}_\ell$.}
     \label{algo:EE}
     \caption{\textbf{S}uccesive \textbf{M}atrix \textbf{E}stimation and \textbf{A}rm \textbf{E}limination (\textbf{SME-AE})}
\end{algorithm}

The following theorem characterizes the performance of SME-AE and the resulting regret. To simplify the notation, we introduce the gaps: for any entry $(i,j)$, $\Delta_{i,j} = (M_{i^\star,j^\star} - M_{i,j})$, $\Delta_{\min}=\min_{(i,j):\Delta_{i,j}>0}\Delta_{i,j}$, $\Delta_{\max}=\max_{(i,j)}\Delta_{i,j}$, and $\bar{\Delta}=\sum_{(i,j)} \Delta_{i,j}/(mn)$. We define the function $\psi(n,m,\delta) = \frac{c (m+n)\log\left(e/\Delta_{\min} \right)}{\Delta_{\min}^2}  \log^3\big(\frac{e(m+n)\log(e/\Delta_{\min})}{\Delta_{\min}\delta} \big)$ for some universal constant $c > 0$.
   
\begin{thm}\label{thm:regret}
(Best entry identification) For any $\delta \in (0,1)$, SME-AE($\delta$) stops at time $\tau$ and recommends arm $(\hat{\imath}_\tau, \hat{\jmath}_\tau)$ with the guarantee $\PP \big( (\hat{\imath}_\tau, \hat{\jmath}_\tau) = (i^\star,j^\star),  \tau \le   \psi(n,m,\delta) \big)  \ge  1-\delta$.
Moreover, for any $T \ge 1$ and $\alpha>0$, the sample complexity $\tau$ of SME-AE($1/T^\alpha$) satisfies $\EE[\tau \wedge T] \le \psi(n,m,T^{-\alpha}) + T^{1-\alpha}$.\\
(Regret) Let $T \ge 1$. Consider the algorithm $\pi$ that first runs SME-AE($1/T^2$) and then commits to its output $(\hat{\imath}_\tau, \hat{\jmath}_\tau)$ after $\tau$. We have: 
    %\begin{align*}
     $R^\pi(T) \le  \bar{\Delta}\left( \psi(n,m,T^{-2})+ 1 \right) + \frac{\Delta_{\max}}{T}$.
    %\end{align*}
\end{thm}

The proof of Theorem \ref{thm:regret} is given in App. \ref{app:bandits}. Note that the regret upper bounds hold for any time horizon $T\ge 1$, and that it scales as $O((m+n)\log^3(T)\bar{\Delta}/\Delta_{\min}^2)$ (up to  logarithmic factors in $m, n$ and $1/\Delta_{\min}$). The cubic dependence in $\log^3(T)$ is an artifact of our proof techniques. More precisely, it is due to the Poisson approximation used to obtain entry-wise guarantees. Importantly, for any time horizon, the regret upper bound only depends on $(m+n)$ rather than $mn$ (the number of arms / entries), and hence, the low-rank structure is efficiently exploited. If we further restrict our attention to problems with gap ratio $\Delta_{\max}/\Delta_{\min}$ upper bounded by $\zeta$, our regret upper bound becomes $O(\zeta (m+n)\log^3(T)/\Delta_{\min})$, and can be transformed into the  minimax gap-independent upper bound $O(\zeta((m+n)T)^{1/2}\log^{2}(T))$, see App. \ref{app:bandits}. Finally note that $\Omega(((m+n)T)^{1/2})$ is an obvious minimax regret lower bound for our low-rank bandit problem.

A very similar low-rank bandit problem has been investigated in \cite{bayati2022speed}. There, under similar assumptions (see Assumption 1 and Definition 1), the authors devise an algorithm with both gap-dependent and gap-independent regret guarantees. The latter are difficult to compare with ours. Their guarantees exhibit a better dependence in $T$ and $\Delta_{\min}$, but worse in the matrix dimensions $n$ and $m$. Indeed in our model, $b^\star$ in \cite{bayati2022speed} corresponds to $\| M\|_{2\to\infty}$ and scales as $\sqrt{n}$. As a consequence, the upper bounds in \cite{bayati2022speed} have a dependence in $n$ and $m$ scaling as $\sqrt{n}(n+m)$ in the worst case for gap-dependent guarantees and even $nm$ (through the constant $C_2$ in \cite{bayati2022speed}) for gap-independent guarantees.

%They obtain gap dependent and gap independent regret bounds. It is very hard to compare our results with theirs since they introduce this so called filtering resolution $h$ that their algorithm requires as input. They use this threshold to eliminate arms of larger gaps. Our algorithm does not require any input like that. In fact, our elimination scheme can be viewed as a mean to to adaptively select the minimal gap by which we eliminate arms. Their gap independent constant $C_2$ could scale as $mn$ while our results do not exhibit any dependent on $mn$ but only $m+n$. Their constant $b^\star$ could hide a factor $\sqrt{m+n}$ which can be an artifact of using Frobeinus norm error bounds on matrix estimation rather than entry-wise as we do.  

\section{Representation Learning in Low-Rank MDPs}\label{sec:mdps}

The results derived for Models II(a) and II(b) are instrumental towards representation learning and hence towards model-based or reward-free RL in low-rank MDPs. In this section, we provide an example of application of these results, and mention other examples in Section \ref{sec:conclusion}. A low-rank MDP is defined by $(\SSS,\AAA,\{P^{a}\}_{a\in\AAA},R,\gamma)$ where $\SSS$, $\AAA$ denote state and action spaces of cardinalities $n$ and $A$, respectively, $P^{a}$ denotes the rank-$r$ transition matrix when taking action $a$, $R$ is the reward function, and $\gamma$ is the discount factor. We assume that all rewards are in $[0,1]$. The value function of a policy $\pi:{\cal S}\to {\cal A}$ is defined as $V^\pi_R(x)=\EE[ \sum_{t=1}^{\infty} \gamma^{t-1} R(x_t^\pi,\pi_t(x_t^\pi)) \vert x_1^\pi=x]$ where $x_t^\pi$ is the state visited under $\pi$ in round $t$. We denote by $\pi^\star(R)$ an optimal policy (i.e., with the highest value function).

{\bf Reward-free RL.} In the reward-free RL setting (see e.g. \cite{kaufmann2021adaptive, jin2020reward, zhang2020nearly}), the learner does not receive any reward signal during the exploration process. The latter is only used to construct estimates $\{ \hP^a \}_{a\in\AAA}$ of $\{P^a \}_{a\in\AAA}$. The reward function $R$ is revealed at the end, and the learner may compute $\hat\pi(R)$ an optimal policy for the MDP $(\SSS,\AAA,\{\hP^{a}\}_{a\in\AAA},R,\gamma)$. The performance of this model-based approach is often assessed through $\Gamma=\sup_{R}\| V^{\pi^\star(R)}_{R} - V^{\pihat(R)}_{R} \|_{\infty}$. In tabular MDP, to identify an $\epsilon$-optimal policy for all reward functions, i.e., to ensure that $\Gamma\le \epsilon$, we believe that the number of samples that have to be collected should be $\Omega(\textrm{poly}({1\over 1-\gamma}){n^2A\over \epsilon^2})$ (the exact degree of the polynomial in $1/(1-\gamma)$ has to be determined). This conjecture is based on the sample complexity lower bounds derived for reward-free RL in episodic tabular MDP \cite{jin2020reward,menard21}. Now for low-rank MDPs, the equivalent lower bound would be $\Omega(\textrm{poly}({1\over 1-\gamma}){nA\over \epsilon^2})$ \cite{jedra2023nearly} (this minimax lower bound is valid for Block MDPs, a particular case of low-rank MDPs).

Leveraging our low-rank matrix estimation guarantees, we propose an algorithm matching the aforementioned sample complexity lower bound (up to logarithmic factors) at least when the frequency matrices $\{M^a \}_{a\in\AAA}$ are homogeneous. The algorithm consists of two phases: (1) in the model estimation phase, it collects $A$ trajectories, each of length $T/A$, corresponding to the Markov chains with transition matrices $\{ P^a\}_{a\in\AAA}$. From this data, it uses the spectral decomposition method described in \textsection\ref{sec:etimation} to build estimates $\{ \hP^a\}_{a\in\AAA}$. (2) In the planning phase, based on the reward function $R$, it computes the best policy $\hat{\pi}(R)$ for the MDP $(\SSS,\AAA,\{\hP^{a}\}_{a\in\AAA},R,\gamma)$. The following theorem summarizes the performance of this algorithm. To simplify the presentation, we only provide the performance guarantees of the algorithm for homogeneous transition matrices (guarantees for more general matrices can be derived plugging in the results from Theorem \ref{thm:main_thm}). 

% \begin{thm}
%    \label{lemma:V_bound_MDP_forward}
%     Assume that for any $a\in {\cal A}$, $M^a$ is homogeneous (as defined in Corollary \ref{corr:homogeneous_model2}), and let $\tau^\star$ denotes an upper bound of the mixing time of the Markov chains with transition matrices $\{ P^a\}_{a\in\AAA}$. If $\delta = \Omega(A n^{-d})$ and $T\ge c n\log(n\sqrt{TA}/\delta)\log(Tn/A)$ for some universal constants $c,d>0$, then  we have with probability at least $1-\delta$:
%     \begin{align*}
%         &\Gamma=\sup_R \| V^{\pi^\star(R)}_R - V^{\hat{\pi}(R)}_{R} \|_{\infty} 
%         \lesssim \frac{1}{(1-\gamma)^3}  \sqrt{\frac{nA}{T} \log( \frac{n\sqrt{TA}}{\delta}) \log(\frac{Tn}{A})}.
%     \end{align*}
% \end{thm}
\begin{thm}
   \label{lemma:V_bound_MDP_forward}
    Assume that for any $a\in {\cal A}$, $M^a$ is homogeneous (as defined in Corollary \ref{corr:homogeneous_model2}). If $T\ge c nA\log^2(nAT)$ for some universal constant $c>0$, then  we have with probability at least $1-\min\{n^{-2},T^{-1}\}$:
    %\begin{align*}
        $\Gamma=\sup_R \| V^{\pi^\star(R)}_R - V^{\hat{\pi}(R)}_{R} \|_{\infty} 
        \lesssim \frac{1}{(1-\gamma)^2}  \sqrt{\frac{nA}{T}} \log(nAT)$.
    %\end{align*}
\end{thm}

Theorem \ref{lemma:V_bound_MDP_forward} is a direct consequence of Corollary \ref{corr:homogeneous_model3} and of the fact that for any reward function $R$: $\| V^{\pi^\star(R)}_R - V^{\hat{\pi}(R)}_{R} \|_{\infty} \le {2\gamma\over (1-\gamma)^2} \max_{a\in {\cal A}}\| P^a - \hP^a\|_{1\to\infty}$, see App. \ref{app:prelim}. The theorem implies that if we wish to guarantee $\Gamma\le \epsilon$, we just need to collect $O({nA\over \epsilon^2(1-\gamma)^4})$ samples up to a logarithmic factor. This sample complexity is minimax optimal in $n$, $A$, and $\epsilon$ in view of the lower bound presented in \cite{jedra2023nearly}.

\section{Related Work}\label{sec:related}

{\bf Low-rank matrix estimation.} Until recently, the main efforts on low-rank matrix recovery were focused on 
guarantees w.r.t. the spectral or Frobenius norms, see e.g. \cite{davenport2016} and references therein. The first matrix estimation and subspace recovery guarantees in $\ell_{2\to\infty}$ and $\ell_{\infty}$ were established in \cite{fan2018eigenvector}, \cite{eldridge2018unperturbed} via a more involved perturbation analysis than the classical Davis-Kahan bound. An alternative approach based on a leave-one-out analysis was proposed in \cite{abbe2020entrywise}, and further refined in \cite{cai2018rate,cape2019two,chen2019spectral}, see \cite{chen2021spectral} for a survey. Some work have also adapted the techniques beyond the independent noise assumption \cite{lei2019unified,abbe2022,agterberg2022entrywise}, but for very specific structural dependence. We deal with a stronger dependence, and in particular with Markovian data (an important scenario in RL).

The estimation of low-rank transition matrices of Markov chains has been studied in \cite{zhang2019spectral,bi2023low} using spectral methods and in \cite{li2018estimation,zhu2022learning} using maximum-likelihood approaches. \cite{zhang2019spectral} does not conduct any fine-grained subspace recovery analysis (such as the leave-one-out), and hence the results pertaining to the $\|\cdot\|_{1\to\infty}$-guarantees are questionable; refer to App. \ref{app:related} for a detailed justification. All these papers do not present entry-wise guarantees.

It is worth mentioning that there exist other methods for matrix estimation that do not rely on spectral decompositions like ours, yet enjoy entry-wise matrix estimation guarantees \cite{shah2020lowrank, agarwal2019reinforcement,sam2023overcoming}. However, these methods require different assumptions than ours that may be too strong for our purposes, notably having access to the so-called anchor rows and columns. Moreover, we do not know if these methods also lead to guarantees for subspace recovery in the norm $\Vert \cdot  \Vert_{2 \to \infty}$, nor how to extend those results to settings with dependent noise.

{\bf Low-rank bandits.} Low-rank structure in bandits has received a lot of attention recently \cite{katariya17,kveton2017stochastic,jun2019,trinh2020,lu2021,bayati2022speed,kang2022efficient, jain2022online}. Different set-ups have been proposed (refer to App. \ref{app:related} for a detailed exposition, in particular, we discuss how the settings proposed in \cite{jun2019,bayati2022speed} are equivalent), and regret guarantees in an instance dependent and minimax sense have been both established. 

Typically minimax regret guarantees in bandits scale as $\sqrt{T}$, but the scaling in dimension may defer when dealing with a low rank structure \cite{jun2019, kang2022efficient, bayati2022speed}. In \cite{jun2019}, the authors also leverage spectral methods. They reduce the problem to a linear bandit of dimension $nm$ but where only roughly $n+m$ dimensions are relevant. This entails that a regret lower bound of order $(n+m)\sqrt{T}$ is inevitable. Actually, in their reduction to linear bandits, they only use a subspace recovery in Frobenius norm, which perhaps explains the scaling $(n+m)^{3/2}$ in their regret guarantees. It is worth noting that in \cite{kang2022efficient}, the authors manage to improve upon the work \cite{jun2019} and obtain a scaling order $(m+n)$ in the regret. Our algorithm leverages entry-wise guarantees which rely on a stronger subspace recovery guarantee. This allows us to obtain a scaling $\sqrt{n+m}$ in the regret. 
The work of \cite{jain2022online} is yet another closely related work to ours. There, the authors propose an algorithm achieving a regret of order $\textrm{polylog}(n+m)\sqrt{T}$ for a contextual bandit problem with low rank structure. However, their result only holds for rank 1 and their observation setup is different than ours because in their setting, the learner observes $m$ entries per round while in ours the learner only observes one entry per round.
In \cite{bayati2022speed}, the authors use matrix estimation with nuclear norm penalization to estimate the matrix $M$. Their regret guarantees are already discussed in \textsection\ref{sec:bandits}. 

Some instance-dependent guarantees with logarithmic regret for low rank bandits have been established in \cite{katariya17, kveton2017stochastic, trinh2020}. However, these results suffer what may be qualified as serious limitations. Indeed, \cite{katariya17, trinh2020} provide instance dependent regret guarantees but only consider low-rank bandits with rank $1$, and the regret bounds of \cite{katariya17} are expressed in terms of the so-called column and row gaps (see their Theorem 1) which are distinct from the standard gap notions. \cite{kveton2017stochastic} extend the results in \cite{katariya17} to rank $r$ with the limitation that they require stronger assumptions than ours. Moreover, the computational complexity of their algorithm depends exponentially on the rank $r$; they require a search over spaces of size ${m \choose r}$ and ${n \choose r}$. Our proposed algorithm does not suffer from such limitations.  

We wish to highlight that our entry-wise guarantees for matrix estimation are the key enabling tool that led us to the design and analysis of our proposed algorithm. In fact, the need for such guarantees arises naturally in the analysis of gap-dependent regret bounds (see Appendix \ref{app:bandits-1}). Therefore, we believe that such guarantees can pave the way towards better, faster, and efficient algorithms for bandits with low-rank structure.

{\bf Low-rank Reinforcement Learning.} RL with low rank structure has been recently extensively studied but always in the function approximation framework \citep{jiang2017,dann2018,DuKJAD19a,misra2020kinematic,foster2021,zhang2022BMDP, sun2019,agarwal2020flambe,Modi21, uehara2021representation, uehara2022, ren2022spectral}. There, the transition probabilities can be written as $\phi(x,a)^\top\mu(x')$ where the unknown feature functions $\phi(x,a), \mu(x')\in \mathbb{R}^r$ belong to some specific class ${\cal F}$ of functions. The major issue with algorithms proposed in this literature is that they rely on strong computational oracles (e.g., ERM, MLE), see \cite{kane2022computational, golowich2022learning, zhang2022making} for detailed discussions. In contrast, we do not assume that the transition matrices are constructed based on a given restricted class of functions, and our algorithms do not rely on any oracle and are computationally efficient. In \cite{shah2020lowrank, sam2023overcoming}, the authors also depart from the function approximation framework. There, they consider a low rank structure different than ours. Their matrix estimation method enjoys an entry-wise guarantee, but requires to identify a subset of rows and columns spanning the range of the full matrix. Moreover, their results are only limited the generative models, which allows to actually rely on independent data samples.  

%The work of \cite{sam2023overcoming} uses similar techniques as in \cite{shah2020lowrank}, therefore suffers similar limitations.    

%the proposed algorithms, typically achieve a sample complexity of order  $O(\textrm{poly}((1-\gamma)^{-1}) \log(\vert \cM \vert) / \epsilon^2)$ to identify an $\epsilon$-optimal policy, where $\cM$ is a finite function class describing the low rank structure and $\vert \cM\vert$ its corresponding cardinality. When $\cM$ is the set of low rank transition matrices, using a covering argument, one can argue that the resulting sample complexity achievable by these algorithms is $O(\textrm{poly}((1-\gamma)^{-1}) n / \epsilon^2)$ which is comparable to ours. Indeed, we can construct a covering $\cM$ of size $e^{\Omega(n)}$ for low rank matrices \cite{candes2011tight}. The major issue with existing algorithms that use function approximation is there reliance on strong computational oracles (e.g., ERM, MLE). These issues have been recently discussed in \cite{kane2022computational, golowich2022learning, zhang2022making}.  Therefore, our work is not to be compared with theirs since we do not rely on the function approximation setting. 

%\cite{pananjady2020instance} generative model, more complicated bounds (depending on variance etc), do not estimate transition matrix directly, \cite{duan2022policy} forward model, TD learning

\section{Conclusion and Perspectives}\label{sec:conclusion}

In this paper, we have established that spectral methods efficiently recover low-rank matrices even in correlated noise. We have investigated noise correlations that naturally arise in RL, and have managed to prove that spectral methods yield nearly-minimal entry-wise error. Our results for low-rank matrix estimation have been applied to design efficient algorithms in low-rank RL problems and to analyze their performance. We believe that these results may find many more applications in low-rank RL. They can be applied (i) to reward-free RL in episodic MDPs (this setting is easier than that presented in \textsection\ref{sec:mdps} since successive episodes are independent); (ii) to scenarios corresponding to offline RL \cite{yin2021optimal} where the data consists of a single trajectory generated under a given behavior policy (from this data, we can extract the transitions $(x,a,x')$ where a given action $a$ is involved and apply the spectral method to learn $\hP^a$); (iii) to traditional RL where the reward function $R$ has to be learnt (learning $R$ is a problem that lies in some sense between the inference problems in our Models I and II); (iv) to model-free RL where we would directly learn the $Q$ function as done in \cite{shah2020} under a generative model; (v) to low-rank RL problems with continuous state spaces (this can be done if the transition probabilities are smooth in the states, and by combining our methods to an appropriate discretization of the state space).

\section*{Acknowledgment}

This research was supported by the Wallenberg AI, Autonomous Systems and Software Program (WASP) funded by the Knut and Alice Wallenberg Foundation.

\newpage
\bibliographystyle{plainnat}
\bibliography{references}

\newpage
\tableofcontents
\newpage

\appendix
\section{Preliminaries}\label{app:prelim}

In this section, we present a few results that are used throughout our analysis.

\subsection{Matrix norms}

\begin{lem}
Let $A\in\RR^{n\times m},B\in\RR^{m\times r}$. Then:
\begin{align}
    \|AB\|_{2\to\infty} &\leq \|A\|_{1\to\infty} \|B\|_{2\to\infty}, \\
    \|AB\|_{2\to\infty} &\leq \|A\|_{2\to\infty} \|B\|, \label{eq:2_to_inf_2_norm_ineq}\\
    \Vert A B \Vert_\infty &\le \Vert A \Vert_{2 \to \infty} \Vert B^\top\Vert_{2\to\infty}.
    \label{eq:inf_ineq_2_to_inf_norm}
\end{align}
\end{lem}
\begin{proof}
The proof of the lemma directly follows from Hölder's inequality (see for example Proposition 6.5 in \cite{cape2019two}). 
\end{proof}

\subsection{Mixing time}

\begin{lem}(Lemma 5 in \cite{zhang2019spectral})
\label{lemma:Markov_mixing_rate}
    Let $\tau(\varepsilon)$ be the $\varepsilon$-mixing time of an irreducible Markov chain. Then if $\varepsilon \leq \delta < 1/2$,
    \begin{align*}
        \tau(\varepsilon) \leq \tau(\delta) \left(1+\ceil{\frac{\log(\delta/\varepsilon)}{\log(1/(2\delta))}}\right).
        %\label{eq:tau_mixing_mengdi}
    \end{align*}
\end{lem}

\subsection{Value difference lemmas}

The following lemmas are used in Section \ref{sec:mdps} to prove Theorem \ref{lemma:V_bound_MDP_forward}. Recall the definition of value function of a policy $\pi$: $V^\pi_R(x)=\EE[ \sum_{t=1}^{\infty} \gamma^{t-1} R(x_t^\pi,\pi_t(x_t^\pi)) \vert x_1^\pi=x]$. The (state, action) value function of $\pi$ is also defined as: for any state $x\in\SSS$ and action $a\in\AAA$, 
\begin{align*}
    Q^\pi_R(x,a)= R(x,a) + \gamma \EE_{x'\sim P(\cdot|x,a)} [V_R^\pi(x')].
\end{align*}
We denote by $\widehat{Q}^\pi_R$ the (state, action) value function of $
\pi$ in the MDP where $P$ is replaced by its estimate $\hP$, and let $\widehat{\pi}(R)$ be the optimal policy for this MDP.

\begin{lem}
    We have that
    \begin{align*}
    \|V_R^{\pi^\star(R)}-V_R^{\widehat{\pi}(R)}\|_{\infty} \leq 2\sup_{\pi}\|Q^{\pi}_R - \widehat{Q}^{\pi}_R\|_{\infty} 
    \end{align*}    
    %\label{lemma:agarwal_V_inf}
\end{lem}
\begin{proof}
We remove the subscript $R$ to simplify the notation. For any $s$, we have 
    \begin{align*}
         V^{\pi^\star}(s)-V^{\widehat{\pi}}(s) &= Q^{\pi^\star}(s,\pi^\star(s)) - Q^{\pihat}(s,\pihat(s)) 
         = [Q^{\pi^\star}(s,\pi^\star(s))  - \widehat{Q}^{\pi^\star}(s,\pi^\star(s))] \nonumber \\ &\quad + [\widehat{Q}^{\pihat}(s,\pihat(s))  - Q^{\pihat}(s,\pihat(s))] + [\widehat{Q}^{\pi^\star}(s,\pi^\star(s)) - \widehat{Q}^{\pihat}(s,\pihat(s))] \nonumber\\ &\leq 2\sup_{\pi} 
         \| Q^{\pi} - \widehat{Q}^{\pi}\|_{\infty},
    \end{align*}
    since $\widehat{Q}^{\pi^\star}(s,\pi^\star(s)) \leq \widehat{Q}^{\pihat}(s,\pihat(s))$ by definition of $\pihat$.
\end{proof}

\begin{lem}(Proposition 2.1 in \cite{agarwal2019reinforcement})
    For all policies $\pi$:
    \begin{align*}
        \|Q^\pi_R - \widehat{Q}^\pi_R \|_{\infty} 
        % = 
        % \| \gamma(I-\gamma\hP^{\pi})^{-1}(P-\hP)V^\pi\|_{\infty} 
        \leq \frac{\gamma}{(1-\gamma)^2} \max_{a\in\AAA} \|P^a - \hP^a\|_{1\to\infty}
    \end{align*}
    %\label{lemma:agarwal_Q_inf}
\end{lem}

Combining the two lemmas, we get:
$$
\| V^{\pi^\star(R)}_R - V^{\hat{\pi}(R)}_{R} \|_{\infty} \le {2\gamma\over (1-\gamma)^2} \max_{a\in {\cal A}}\| P^a - \hP^a\|_{1\to\infty}.
$$
This inequality is used in the proof of Theorem \ref{lemma:V_bound_MDP_forward}.

\newpage
\section{Statement and proofs of the main results}
\label{app:main}

In this appendix, we present the proofs of the main theorems. In Subsection \textsection\ref{sec:proofs-thms-reward}, we provide the proof of Theorem \ref{thm:reward} and Corollary \ref{corr:homogeneous1}. In Subsection \textsection\ref{subsec:appendix_tight_thm_generative}, we give a complete, non-simplified version of Theorem \ref{thm:generative} from which one can deduce Theorem \ref{thm:generative} and Corollary \ref{corr:homogeneous_model2} given in the main text. Finally, in Subsection \textsection\ref{subsec:appendix_tight_thm_forward}, we present a complete, non-simplified version of Theorem \ref{thm:main_thm} and from the latter, deduce Theorem \ref{thm:main_thm} and Corollary \ref{corr:homogeneous_model3}.

\subsection{Reward matrix estimation -- Model I} \label{sec:proofs-thms-reward}
In this subsection, we present the proofs of Theorem \ref{thm:reward}. The proof of Corollary \ref{corr:homogeneous1} is in fact immediate from Theorem \ref{thm:reward}. 

\begin{proof}[Proof of Theorem \ref{thm:reward}]
\textbf{Proof of (i)}. Recall the results from Lemma \ref{lem:ssr-rewards}: for all $\delta \in (0,1)$, if  
\begin{align}\label{eq:def_B_reward}
        \cB = \sqrt{\frac{nm}{T}} \left( \sqrt{(n+m) \log\left( \frac{e(n+m)T}{\delta}\right)  }  + \log^{3/2}\left( \frac{e(n+m)T}{\delta}\right) \right),
    \end{align}
    then for all $ T \ge  c_1 (\mu^4 \kappa^2 r^2 + 1) (m+n) \log^{3}\left(e^2(m+n)T / \delta
    \right)$,  the event  
    \begin{align*}
        \max(  \Vert U - \widehat{U} (\widehat{U}^\top U)\Vert, \Vert V - \widehat{V} (\widehat{V}^\top V)\Vert) \le C_1 \frac{\Vert M \Vert \Vert M\Vert_\infty}{\sigma_r(M)^2} \max(\Vert V\Vert_{2\to \infty} \Vert U \Vert_{2\to\infty}) \cB
    \end{align*}
    holds with probability at least $1-\delta$ for some universal constants $c_1, C_1 > 0$. To obtain the form presented in Theorem \ref{thm:reward}, we simply recall the definitions $\kappa = \Vert M \Vert/\sigma_r(M)$, $\mu = \max( \sqrt{m/r}\Vert U \Vert_{2\to \infty}, \sqrt{n/r} \Vert V \Vert_{2\to \infty} )$ and the bound $\Vert M \Vert_{\infty}/ \sigma_r(M) \le (\mu^2 \kappa r)/ \sqrt{mn}$ from Lemma \ref{lem:spikeness}. We then substitute in the upper bound above. Note that $\mu,\kappa$ and $r$ are larger than $1$ by definition.

\textbf{Proof of (ii).} To establish the desired bound, we use the decomposition error established in Lemma \ref{lemma:P_to_U}. Namely, under the event that $\Vert \widetilde{M} - M \Vert \le c_1 \sigma_r(M)$ for some universal constant $c_1> 0$ sufficiently small, there exists a universal constant $c_2>0$ such that
\begin{align}\label{eq:rew-upperbound-1}
    \|\hM - M\|_{2\to\infty} & \leq
    c_2\sigma_1(M) \left[
    \|U-\hU (\hU^\top U)\|_{2\to\infty} + \mu\sqrt{\frac{r}{m}}\frac{\|\tM-M\|}{\sigma_r(M)} \right].
\end{align}
Hence, we only need high probability bounds on $\|U-\hU (\hU^\top U)\|_{2\to\infty}$ which we established in (i), and on $\Vert \widetilde{M} - M \Vert$ which we also established in Proposition \ref{prop:reward-concentration-1} under the compound Poisson entries model described \eqref{eq:reward-poisson-matrix-plus-noise-model}. We can extend the latter result under our observation model using the Poisson approximation Lemma \ref{lem:poisson-approx-rewards}, and finally write that for all $\delta\in (0,1)$, using the same definition of $\cB$ as above in \eqref{eq:def_B_reward}, for all for all $
        T \ge  c_3 \log^3\left( (n+m)/\delta\right)$, the following statement
    \begin{align}\label{eq:rew-upperbound-2}
         \frac{\Vert \widetilde{M} - M \Vert}{\sigma_r(M)}  \le  \frac{C_3 \Vert M \Vert_\infty}{\sigma_r(M)}  \cB 
    \end{align}
    holds with probability at least $1-\delta$, 
    for some universal constants $c_3, C_3> 0$ large enough. Note that under the condition  $
        T \ge  c_4 \mu^4 \kappa^2 r^2 \log^3 \left(e (n+m)/\delta\right)$  for some universal constant $c_4$ large enough,  the high probability statement in \eqref{eq:rew-upperbound-2} holds and in addition we also have $\Vert \widetilde{M} - M \Vert \le c_1 \sigma_r(M)$. There, we used the result of Lemma \ref{lem:spikeness}. The statement (ii) in Theorem \ref{thm:reward} is obtained by first substituting in \eqref{eq:rew-upperbound-1}, the upper bound we get in (i) and that we get in  \eqref{eq:rew-upperbound-2}, and then, using  $\sigma_1(M) \le \sqrt{mn} \Vert M \Vert_\infty$  and the bound $\Vert M \Vert_{\infty}/ \sigma_r(M) \le (\mu^2 \kappa r)/ \sqrt{mn}$ from Lemma \ref{lem:spikeness}.

\textbf{Proof of (iii).} To establish the desired bound, we use the decomposition error established in Lemma \ref{lem:M-infinity}. Namely, under the event that $\Vert \widetilde{M} - M \Vert \le c_1 \sigma_r(M)$ for some universal constant $c_1> 0$ sufficiently small, there exists a universal constant $c_2>0$ such that
    \begin{align}
        \Vert \widehat{M} - M \Vert_{\infty}  \leq
        c_2 \| M\|_{2\to\infty} & \left(\frac{\|M-\tM\|}{\sigma_r(M)} \| V\|_{2\to\infty} + \|V-\hV\WhV \|_{2\to\infty} \right) \nonumber \\
        &+ c_2\| M - \hM \|_{2\to\infty} (\| V\|_{2\to\infty} + \|V-\hV\WhV \|_{2\to\infty} ). 
        \label{eq:error-3-reward}
    \end{align}
To upper bound the above error, we need to control: (a) $\|V-\hV\WhV \|_{2\to\infty}$, which we have already done in (i); (b) $\|M-\tM\|$, which follows from Lemma \ref{lem:poisson-approx-rewards} as established in the proof of (ii) (see the high probability statement \eqref{eq:rew-upperbound-2}); and (c) $\| M - \hM \|_{2\to\infty}$, which again we have already done in (ii). The statement (iii) in Theorem \ref{thm:reward} follows from  first substituting in \eqref{eq:error-3-reward}, the upper bounds we get from (a), (b) and (c), and then using $\Vert M \Vert_{\infty}/ \sigma_r(M) \le (\mu^2 \kappa r)/ \sqrt{mn}$, $\mu = \max( \sqrt{m/r}\Vert U \Vert_{2\to \infty}, \sqrt{n/r} \Vert V \Vert_{2\to \infty} )$, and the basic inequality $\Vert M \Vert_{2 \to \infty} \le \sqrt{m} \Vert M\Vert_\infty \le \sqrt{m+n} \Vert M\Vert_\infty$.
\end{proof}

\subsection{Transition matrix estimation under the generative model -- Model II(a)}
\label{subsec:appendix_tight_thm_generative}
% The estimate $\hM$ is obtained taking the rank-$r$ approximation of the matrix $\tM^{(1)}$ recording the empirical frequency of the transitions: $\tM_{i,j} =  \frac{1}{\lfloor T/2\rfloor} \sum_{k=1}^{\lfloor T/2\rfloor}\indicator_{ \{ (x_{2k-1},x_{2k})=(i,j)\} }$ for any pair of states $(i,j)$. The estimate $\hP$ of the transition matrix is then obtained normalizing the rows of $\hM$: for all $i\in [n]$,
In this subsection, we present a complete, non-simplified version of Theorem \ref{thm:generative}, from which one can deduce Theorem \ref{thm:generative} and Corollary \ref{corr:homogeneous_model2} given in the main text. 

First, let us define the function 
$g_{\delta}:\RR^{n\times n}\to\RR_+$ as
\begin{align}
    g_{\delta}(M)=\indicator_{\{\exists \ell: \| M_{\ell,:} \|_{\infty}\leq 1\}}\log &\left(\frac{ne}{\delta}\right) \log^{-1}\left( 1+\frac{1}{\norm{M}_{\infty}} \right) \nonumber\\
    &+ \indicator_{\{\forall \ell: \| M_{\ell,:} \|_{\infty}> 1\} }\log\left(\frac{\norm{M}_{\infty} ne}{\delta}\right) \sqrt{\norm{M}_{\infty}}.
\label{eq:g_def} 
\end{align}
%that characterizes Poisson tail and which we use in the following two subsections.
We also use the following notation:
$$
\left\{
\begin{array}{l}
{\cal A}=\frac{1}{\sqrt{T}} \sqrt{\|M\|_{1\to\infty}+\|M^\top\|_{1\to\infty}},\\
{\cal B}'=\mu\kappa \sqrt{\frac{r}{n}} \left( {\cal A}+ \frac{1 }{T}g_{\delta/\sqrt{T}}(TM)\log \left(\frac{n\sqrt{T}}{\delta}\right) \right)  + \sqrt{\frac{r  \|M\|_{\infty} }{T}\log \left( \frac{n\sqrt{T}}{\delta}\right)}. 
\end{array}
\right.
$$

We first recall a standard result quantifying how well $\tM$ approximates $M$.

\begin{lem} $\forall \delta \in (0, 1)$, w.p. at least $1-\delta$,
   % \begin{align}
     $\| \tM-M\| \leq C{\cal A}  + \frac{C}{T} g_{\delta/\sqrt{T}}(T M) \sqrt{\log( \frac{n\sqrt{T}}{\delta})}.$ 
    % \end{align}
    \label{lem:spectral_Multinomial}
\end{lem}
\begin{proof}
    The lemma follows directly from Lemma \ref{lemma:multinomial_to_poisson} (replacing $T_{\tau}$ by $T$) and Lemma \ref{lem:spectral_Poisson}.
\end{proof}

\begin{thm}
    \label{thm:multinomial_sampling_tight}
    Assume that $(\nu_0)_{\min} = \min_{i\in[n]}(\nu_0)_i>0$. For any $\delta>0$, if $\| \tM-M\| \leq c \sigma_r(M)$, $g_{\delta/\sqrt{T}}(T M)\log(n\sqrt{T}/\delta)\leq c T\sigma_r(M)$ and  $\| M\|_{\infty}\log (n\sqrt{T}/\delta)\leq c T\sigma_r^2(M)$ for some universal constant $c>0$, then there exists a universal constant $C>0$ such that with probability at least $1-\delta$ holds:
$$
\begin{array}{ll}
%\hbox{(subspace recovery)} 
(i) & \max\Big\{ \| U  - \hU (\hU^\top U)\|_{2\to\infty}, \|V - \hV (\hV^\top V)\|_{2\to\infty} \Big\} \leq C \frac{{\cal B}'}{\sigma_r(M)},\\
%\hbox{(row-wise error)} 
(ii) & \| \hM - M\|_{2\to\infty} \leq C \kappa {\cal B}', \quad \| \hP - P\|_{1\to\infty} \leq C \frac{\kappa\sqrt{n}}{(\nu_0)_{\min}} {\cal B}',\\
%\hbox{(entry-wise error)} 
(iii) & \| \hM - M \|_{\infty}  
\leq C \left( \frac{\Vert M \Vert_{2\to\infty} + \kappa {\cal B}'}{\sigma_r(M)} + \kappa  \mu \sqrt{\frac{r}{n}} \right) {\cal B}',\\
(iv) & \| \hP - P\|_{\infty} \leq C \frac{ {\cal B}'}{ (\nu_0)_{\min} } \left[ \sqrt{n}\kappa\frac{\| M\|_{\infty}}{ (\nu_0)_{\min} } + \frac{\|M\|_{2\to\infty} +  \kappa {\cal B}'}{ \sigma_r(M)} + \kappa\mu\sqrt{\frac{r}{n}} \right],
\end{array}
$$
where (iv) holds if in addition $\|  \hM -  M \|_{1 \to \infty} \le \frac{1}{2} (\nu_0)_{\min}$. 
\end{thm}
\begin{proof}
    The first statement of the theorem follows from Lemma \ref{lemma:multinomial_to_poisson} (with $T$ instead of $T_{\tau}$), Lemmas \ref{lem:spectral_Poisson} and \ref{lemma:U_to_EU}. The remaining bounds are consequences of $(i)$ and of the results presented in Appendix \ref{sec:different_norms_from_sing_subspace}.
\end{proof}

%\begin{rem}
%\label{rem:thm_gen_simplifying_bounds}
\begin{proof}[Proof of Theorem \ref{thm:generative}] Theorem \ref{thm:generative} follows
from Lemma \ref{lem:spectral_Multinomial} and Theorem \ref{thm:multinomial_sampling_tight} by simplifying the term ${\cal B}'$ using ${\cal B}$ given in Theorem \ref{thm:generative}. As a result of this simplification, as well as of the assumptions given in statement of Theorem \ref{thm:multinomial_sampling_tight}, we obtain bounds on $n,T$ required in Theorem \ref{thm:generative}. Furthermore, we use simple inequalities (check Lemma \ref{lem:spikeness}) to rewrite all terms depending on $M$ as functions of $\norm{M}_{\infty}$ and $(\nu_0)_{\min}$.
\end{proof}

\begin{rem}
\label{rem:thm_gen_corollary}
    It is worth noting that Corollary \ref{corr:homogeneous_model2} is a corollary of Theorem \ref{thm:multinomial_sampling_tight}, and that the lower bound on $n$ required in Theorem \ref{thm:generative} is not required for this corollary. Moreover, results presented in this corollary are valid for almost all $T\geq cn\log(nT)$ - in the case when $T \asymp [n^{2-\epsilon},n^2]$ for arbitrarily small $\epsilon>0$, bounds in Corollary \ref{corr:homogeneous_model2} contain additional $\log$ term, which is an artifact of our analysis (and splitting concentration into cases $T\lesssim n^2$ and $T\gtrsim n^2$). This discontinuity in the range of $T$ can be resolved, but at the price of a reduced readability.
\end{rem}

\subsection{Transition matrix estimation under the forward model -- Model II(b)}
\label{subsec:appendix_tight_thm_forward}

In this subsection, we present a complete, non-simplified version of Theorem \ref{thm:main_thm}, from which one can deduce Theorem \ref{thm:main_thm} and Corollary \ref{corr:homogeneous_model3} given in the main text.

% To estimate $M$, we first split the data into $\tau$ subsets of transitions: for $k=1,\ldots,\tau$, the $k$-th subset is $((x_{k},x_{k+1}),(x_{k+\tau},x_{k+1+\tau}),\dots,(x_{k+(T_{\tau}-1)\tau},x_{k+1+(T_{\tau}-1) \tau}))$ where $T_\tau=\lfloor T/\tau\rfloor$. By separating two transitions in the same subset, we break the inherent correlations in the data if $\tau$ is large enough. Now we let $\tM^{(k)}$ be the matrix recording the empirical frequencies of the transitions in the $k$-th subset: $\tM_{i,j}^{(k)}=  \frac{1}{T_{\tau}} \sum_{l=0}^{T_{\tau}-1} \indicator_{\{ (x_{k+l\tau},x_{k+1+l\tau})=(i,j)\} }$ for any pair of states $(i,j)$. Let $\hM^{(k)}$ be the $r$-rank approximation of $\tM^{(k)}$. As in (\ref{def:P_hat_MC_general}), we define the corresponding $\hP^{(k)}$. Finally we aggregate these estimates $\tM=\frac{1}{\tau}\sum_{k=1}^{\tau} \tM^{(k)}$, $\hM=\frac{1}{\tau}\sum_{k=1}^{\tau} \hM^{(k)}$ and $\hP=\frac{1}{\tau}\sum_{k=1}^{\tau} \hP^{(k)}$. 

% Before we state our results regarding $\hM$ and $\hP$, we define for any $\varepsilon>0$, $\tau(\varepsilon)$ as the $\varepsilon$-mixing of the Markov chain with transition matrix $P$: $\tau(\varepsilon)= \min \{ k\geq 1:\ \max_{1\leq i\leq n} \frac{1}{2} \|P^k_{i,:} - \nu^\top\|_1 \leq \varepsilon\}$. Let $\tau^\star= \tau(1/4)$ be its mixing time. 
 
Again, we use function the funcion $g_\delta$ defined in  \eqref{eq:g_def}, and we introduce:
\begin{align*}
{\cal B}' =  \mu\kappa\sqrt{\frac{r}{n}} \bigg( \sqrt{\frac{\| \nu\|_{\infty}\tau^\star }{T}\log\left(\frac{ne}{\delta}\right) \log(  T \nu_{\min}^{-1} ) } + \frac{ \tau^\star}{T} &g_{\delta/\sqrt{T_{\tau}}}(T_{\tau} M)\log \left(\frac{n\sqrt{T_{\tau}}}{\delta}\right)\log ( T\nu_{\min}^{-1} ) \bigg)  \\
        &+ \sqrt{\frac{r \tau^\star \|M\|_{\infty}}{T} \log \left(\frac{n \sqrt{T_{\tau}}}{\delta}\right) \log (  T \nu_{\min}^{-1} ) }. 
\end{align*}

Our analysis starts from the following lemma stating how well $\tM$ approximates $M$. 

\begin{lem}(Lemma 7 in \cite{zhang2019spectral}) For $\tau \geq 2 \tau^\star \log (T\nu_{\min}^{-1})$ and for any $\delta \in (0,1)$, we have with probability at least $1-\delta$: $\| \tM-M\| \leq  C\sqrt{\frac{\| \nu\|_{\infty}\tau }{T}\log\left(\frac{ne}{\delta}\right)} + C\frac{\tau}{T} \log\left(\frac{ne}{\delta}\right)$.
\end{lem}

\begin{thm}
\label{thm:main_thm_tight}
Assume that $\nu_{\min} = \min_{i\in[n]} \nu_i>0$ and that $\tau/(\tau^\star \log (T\nu_{\min}^{-1}))\in [c_1,c_2]$ for some universal constants $c_2> c_1\ge 2$. For any $\delta>0$, if $\| \tM-M\| \leq c \sigma_r(M)$, $g_{\delta/\sqrt{T}}(T_{\tau} M)\log(n\sqrt{T_{\tau}}/\delta)\leq c T_{\tau}\sigma_r(M)$ and $\| M\|_{\infty}\log (n\sqrt{T_{\tau}}/\delta)\leq c T_{\tau}\sigma_r^2(M)$ for some universal constant $c>0$, then there exists a universal constant $C>0$ such that with probability at least $1-\delta$,
$$
\begin{array}{ll}
%\hbox{(subspace recovery)} 
(i) & \max\Big\{ \| U  - \hU (\hU^\top U)\|_{2\to\infty}, \|V - \hV (\hV^\top V)\|_{2\to\infty} \Big\} \leq C \frac{{\cal B}'}{\sigma_r(M)},\\
%\hbox{(row-wise error)} 
(ii) & \| \hM - M\|_{2\to\infty} \leq C \kappa {\cal B}', \quad \| \hP - P\|_{1\to\infty} \leq C \frac{\kappa\sqrt{n}}{\nu_{\min}} {\cal B}',\\
%\hbox{(entry-wise error)} 
(iii) & \| \hM - M \|_{\infty}  
\leq C \left( \frac{\Vert M \Vert_{2\to\infty} + \kappa {\cal B}'}{\sigma_r(M)} + \kappa  \mu \sqrt{\frac{r}{n}} \right) {\cal B}',\\
(iv) & \| \hP - P\|_{\infty} \leq C \frac{ {\cal B}'}{ \nu_{\min} } \left[ \sqrt{n}\kappa\frac{\| M\|_{\infty}}{ \nu_{\min} } + \frac{\|M\|_{2\to\infty} +  \kappa {\cal B}'}{ \sigma_r(M)} + \kappa\mu\sqrt{\frac{r}{n}} \right],
\end{array}
$$
where (iv) holds if in addition $\|  \hM -  M \|_{1 \to \infty} \le \frac{1}{2} \nu_{\min}$. 
\end{thm}
\begin{proof}
    The first statement of the 
    theorem follows from Lemmas \ref{lemma:markov_to_multinomial}, \ref{lemma:multinomial_to_poisson} and \ref{lemma:U_to_EU}, whereas the next four bounds follow from $(i)$ and the bounds presented in Appendix \ref{sec:different_norms_from_sing_subspace}.
\end{proof}

As for the generative model, Theorem \ref{thm:main_thm} is a direct consequence of Theorem \ref{thm:main_thm_tight}, and it is obtained by simplifying the term ${\cal B}'$ to ${\cal B}$. Corollary \ref{corr:homogeneous_model3} is also easily derived from Theorem \ref{thm:main_thm_tight}.

\subsection{An additional lemma}

\begin{lem}\label{lem:spikeness}
    Let $M$ be matrix and $m\times n$ matrix with rank $r$, incoherence parameter $\mu > 0$, and condition number $\kappa > 0$. Then, we have  
    \begin{align*}
        \Vert M \Vert_\infty  \le   \sigma_1(M) \frac{ \mu^2 r  }{\sqrt{nm}} \le \sigma_r(M) \frac{ \mu^2 \kappa  r  }{\sqrt{nm}}.
    \end{align*}
\end{lem}

\begin{proof}[Proof of Lemma \ref{lem:spikeness}]
For all $(i,j) \in [m] \times [n]$, we have 
    \begin{align*}
        \vert M_{i,j} \vert & = \left\vert \sum_{\ell = 1}^r \sigma_\ell(M) u_{i,\ell}  v_{j, \ell} \right\vert \le  \sigma_1(M)  \sum_{\ell = 1}^r \vert u_{i,\ell}  v_{j, \ell} \vert \\
        & \le \sigma_1(M)  \Vert U \Vert_{2 \to \infty} \Vert V \Vert_{2 \to \infty} \le \sigma_1(M) \frac{ \mu r  }{\sqrt{nm}} \le \sigma_r(M)   \frac{ \mu \kappa  r  }{\sqrt{nm}}.
    \end{align*}
    The first inequality follows from the triangular inequality and the fact that $\sigma_1(M) \ge \sigma_2(M) \ge \dots \ge \sigma_r(M)$. The second inequality follows from Cauchy-Schwarz inequality. The last inequalities follow by definition of the incoherence parameter and of the condition number.
\end{proof}

\newpage
\section{Comparison inequalities and the Poisson approximation argument}\label{app:poisson-approx}

In this section, we state and prove the results related to the Poisson approximations used to handle the noise correlations in the data. 
We start by presenting some of the key tools behind the Poisson approximation argument. This argument comes in the form of comparison inequalities. The latter are applied and specified  first to Model I (reward matrix estimation), and then to Model II (transition matrix estimation).

\subsection{Preliminaries on Poisson approximation}

The Poisson approximation argument comes in the form of an inequality, which is presented in Lemma \ref{lemma:hetero_poisson_multi}. However, the key idea behind the argument is, roughly speaking, the equality in distribution between a multinomial distribution with $t$ trials and $n$ outcomes, and the joint distribution of $n$ in dependent Poisson random variables with properly chosen parameters, conditioned on some particular event. This equality of distribution is powerful for our purposes precisely because of the independence between the Poisson random variables. Below, we present Lemma \ref{lemma:hetero_poisson_matrix_equiv_multinomial} that represents this idea.

\begin{lem}
    (Heterogeneous analogue of Theorem 5.2 in \cite{mitzenmacher2017probability}) Let $Y_i^{(t)} \sim \mathrm{Poisson}(t p_i)$, $i=1,\dots,n$, be independent random variables with $\sum_{i=1}^{n} p_i = 1$. Moreover, let $(Z_1^{(t)},Z_2^{(t)},\dots,Z_n^{(t)})\sim\allowbreak \mathrm{Multinomial}(t,(p_1,\dots,p_n))$. Then distribution of $(Y_1^{(t)},\dots,Y_{n}^{(t)})$ conditioned on $\sum_{i=1}^{n} Y_i^{(t)} = s$ is the same as $(Z_1^{(s)},\dots,Z_{n}^{(s)})$ irrespective of $t$.
    \label{lemma:hetero_poisson_matrix_equiv_multinomial}
\end{lem}
\begin{proof}
    The proof follows similar steps as the proof of Theorem 5.2 in \cite{mitzenmacher2017probability}, but we provide it here for the sake of completeness.
    First, note that from the definition of multinomial distributions:
    \begin{align}
        \PP((Z_1^{(s)},\dots,Z_{n}^{(s)})=(a_1,\dots,a_{n})) = \frac{s!}{a_1!\cdots a_{n}!} p_1^{a_1} \cdots p_{n}^{a_{n}}
        \label{eq:multi_dist}
    \end{align}
    if $\sum_{i=1}^{n} a_i =s$, and $0$ otherwise. Since the sum of Poisson random variables is a Poisson random variable with parameter equal to the sum of parameters of the initial random variables, we get that the random variable $\sum_{i=1}^{n} Y_i^{(t)} \sim \mathrm{Poisson}(\sum_{i=1}^{n} t p_i) =\mathrm{Poisson}(t)$. Hence we have:
    \begin{align}
        \PP \left( (Y_1^{(t)},\dots,Y_{n}^{(t)}) = (a_1,\dots,a_{n}) \Bigg| \sum_{i=1}^{n} Y_i^{(t)} = s \right) & =
        \frac{\PP((Y_1^{(t)},\dots,Y_{n}^{(t)}) = (a_1,\dots,a_{n}))}{\PP(\sum_{i=1}^{n} Y_i^{(t)} = s)}\nonumber \\
        &= \frac{s!}{\exp(-t)t^s}  \prod_{i=1}^n \frac{(tp_i)^{a_i} \exp(-t p_i)}{a_i!} \nonumber  \\
        & = \frac{s!}{a_1! \cdots a_{n}!}  p_1^{a_1} \cdots p_{n}^{a_{n}}
        \label{eq:poisson_cond_dist}
    \end{align}
    where in the last step we used the independence of $Y_i^{(t)}$'s and $\sum_{i=1}^n a_i = s$. Note that equations \eqref{eq:multi_dist} and \eqref{eq:poisson_cond_dist} are exactly the same, which concludes the proof.
\end{proof}

\begin{lem}
    Consider the setting of Lemma \ref{lemma:hetero_poisson_matrix_equiv_multinomial} and let $f:\mathbb{R}^p\to\mathbb{R}_+$ be any non-negative function. Then:
    \begin{align*}
        \EE \left[f(Z_1^{(t)},\dots,Z_{p}^{(t)})\right] \leq e\sqrt{t} \EE \left[f(Y_1^{(t)},\dots,Y_{p}^{(t)})\right]
    \end{align*}
    \label{lemma:hetero_poisson_multi}
\end{lem}
\begin{proof}
    The proof is essentially the same as that of Theorem 5.7 in \cite{mitzenmacher2017probability} with the exception that we use Lemma \ref{lemma:hetero_poisson_matrix_equiv_multinomial} instead of Theorem 5.6 in \cite{mitzenmacher2017probability} and we repeat it here for the sake of completeness.
    \begin{align}
        \EE [f(Y_1^{(t)},\dots,Y_{p}^{(t)})] &=
        \sum_{k=0}^{\infty} \EE \left[ f(Y_1^{(t)},\dots,Y_{p}^{(t)}) \Big\vert \sum_{i=1}^p Y_i^{(t)}=k  \right] \PP \left( \sum_{i=1}^p Y_i^{(t)}=k \right)\nonumber \\
        &\geq
        \EE \left[ f(Y_1^{(t)},\dots,Y_{p}^{(t)}) \Big\vert \sum_{i=1}^p Y_i^{(t)}=t  \right] \PP \left( \sum_{i=1}^p Y_i^{(t)}=t \right) \nonumber \\
        &=
        \EE [f(Z_1^{(t)},\dots,Z_{p}^{(t)})] \PP \left( \sum_{i=1}^p Y_i^{(t)}=t\right)
        \label{eq:proof_mult_poiss_f_nonneg}
    \end{align}
    where in the second line we used non-negativeness of $f$, and in the last line we used Lemma \ref{lemma:hetero_poisson_matrix_equiv_multinomial}. Now, since $\sum_{i=1}^p Y_i^{(t)}$ is a Poisson random variable with mean $t$ we have $\PP (\sum_{i=1}^p Y_i^{(t)}=t) = \frac{t^t \exp(-t)}{t!}$ and using simple inequality $t! < e\sqrt{t} (\frac{t}{e})^t$ we can rewrite inequality \eqref{eq:proof_mult_poiss_f_nonneg} as follows:
    \begin{align}
        \EE [f(Y_1^{(t)},\dots,Y_{p}^{(t)})] \geq 
        \EE [f(Z_1^{(t)},\dots,Z_{p}^{(t)})] \frac{1}{e\sqrt{t}}
    \end{align}
    which gives statement of the lemma.
\end{proof}

\subsection{Poisson approximation for reward matrices -- Model I}

We recall from Section \ref{sec:bandits} that the definition of the empirical reward matrix $\widetilde{M}$ is given as follows 
\begin{align}\label{eq:bandit-matrix-noise-model}
\forall (i,j)\in [n] \times [m], \qquad \widetilde{M}_{i,j}  = \frac{nm}{T} \sum_{t=1}^T (M_{i_t, j_t} + \xi_{t} ) \indicator_{\lbrace (i_t, j_t) = (i,j)\rbrace}  
\end{align} 
where $(i_t, j_t)$ are sampled uniformly at random from $[n] \times [m]$. Due to independence between $(i_1, j_1), \dots, (i_T, j_T)$ and $\xi_1, \dots, \xi_T$, we note that the observation model \eqref{eq:bandit-matrix-noise-model} is equivalent in distribution to the following one 
\begin{align*}
    \forall (i,j)\in [n] \times [m], \qquad \widetilde{M}_{i,j}  = \frac{n m}{T} \sum_{t=1}^{Z_{i,j}} (M_{i,j} + \xi'_{i,j,t} ) 
\end{align*}
where we for all $(i,j) \in [n]\times [j]$, $(\xi_{i,j,t}')_{t\ge 1}$ is a sequence of i.i.d. random variables copies, say of $\xi_1$, and   
\begin{align*}
    Z_{i,j} & = \sum_{t=1}^T \indicator_{\lbrace (i_t, j_t) = (i,j)\rbrace}. 
\end{align*}
Observe that $Z=(Z_{i,j})_{(i,j)}$ is a multinomial random variable whose parameters are defined by the fact that for all $t\in [T]$, $\PP((i_t, j_t) = (i,j)) = 1/n m$). We denote $\PP$ the joint probability of the entries of $Z$ and sequences $(\xi_{i,j,t})_{t\ge 1}$, $(i,j)\in [n] \times [m]$. 

\medskip

\paragraph{Compound Poisson random matrix model.} We define a random matrix $Y \in \RR^{n \times m}$ generated by a Poisson model as follows: 
\begin{align*}
    Y_{i,j} \sim \textrm{Poisson}\left(T/nm\right), \qquad (i,j)\in [n] \times [m]
\end{align*}
and denote $\PP'$ the joint probability of the entries of $Y$ and the sequences $(\xi'_{i,j,t})_{t\ge 1}$, for $(i,j)\in [n] \times [m]$. We may then consider the matrix model 
\begin{align}\label{eq:reward-poisson-matrix-plus-noise-model}
    X_{i,j}  = \sum_{t=1}^{Y_{i,j}} (M_{i,j} + \xi'_{i,j,t}). 
\end{align}
We note that the entries of the matrix $X$ are distributed according to compound Poisson distributions. Below, we precise the Poisson approximation argument for the reward matrix model.

\begin{lem}[Poisson Approximation]\label{lem:poisson-approx-rewards}
    Let $(\Omega, \cF , \PP)$ (resp.  $(\Omega, \cF , \PP')$) be the probability space under the matrix-plus-noise model \eqref{eq:bandit-matrix-noise-model}  (resp. \eqref{eq:reward-poisson-matrix-plus-noise-model}). Then for any event $\cE \in \cF$, we have 
    \begin{align*}
        \PP \left( \cE  \right)  \le e \sqrt{T} \, \PP'\left(\cE \right).
    \end{align*}
\end{lem}

\begin{proof}[Proof of Lemma \ref{lem:poisson-approx-rewards}]
For convenience, we denote $X = ((\xi_{i,j,t})_{t\ge 1})_{i,j \in [n]\times [m]}$. 
    We set $f(Z, X) = \indicator_{\lbrace \cE\rbrace}$. Thanks to Lemma \ref{lemma:hetero_poisson_multi}, given that $Z$ is independent of $X$, we have 
    \begin{align*}
        \EE\left[   f(Z, X)   \vert X \right] \le e\sqrt{T}  \EE\left[   f(Y, X)   \vert X \right].
    \end{align*}
    We further take the expectation on $X$ and write 
    \begin{align*}
         \PP(\cE) = \EE\left[   f(Z, X)   \right] \le e\sqrt{T}  \EE\left[   f(Y, X)  \right] = e\sqrt{T} \;  \PP'(\cE).
    \end{align*}
\end{proof}

\subsection{Approximations for transition matrices -- Model II}

We restrict our attention to the forward model, Model II(b). The results for the generative model are simpler and can be easily deduced from those for the forward model. Recall from Section \ref{subsec:transition_theorem_forward} definition of matrix $\tM^{(k)}$ and in the following discussion we fix value of $k\in[\tau]$. Define a matrix $N =T_{\tau} \tM^{(k)}$ and note that it is equal to:
\begin{align}
    N_{i,j}=  \sum_{l=0}^{T_{\tau}-1} \indicator_{\{ (x_{k+l\tau},x_{k+1+l\tau})=(i,j)\} }, \qquad i,j=1,2,\dots,n
    \label{def:M_hat_MC_model}
\end{align}
Furthermore, let $\PP_1$ be joint probability distribution of entries of $N$. 
\subsubsection{Multinomial approximation}
Here we define a matrix $Z\in\RR^{n\times n}$ with entries: 
\begin{align}
    Z_{i,j} = \sum_{t=0}^{T_{\tau}-1} \indicator_{\{ (i_t,j_t)=(i,j) \} }, \qquad i,j=1,2,\dots,n,
    \label{eq:M_hat_multinomial_model}
\end{align}
where $\mathbb{P}((i_t,j_t) = (i,j)) = \nu_i P_{i,j}$ independently over $i,j\in[n]$ and $t\in[T_{\tau}]$. Denote by $\PP_2$ joint probability distribution of entries of $Z$. Then we have:

% \begin{lem} 
% \label{lemma:markov_to_multinomial}
%     Let $(\Omega,\mathcal{F},\PP_1)$ and $(\Omega,\mathcal{F},\PP_2)$ be the probability spaces under the models \eqref{def:M_hat_MC_model} and \eqref{eq:M_hat_multinomial_model}, respectively. Then for any event $\mathcal{E} \in \mathcal{F}$, we have
%     \begin{align*}
%         \PP_1(\mathcal{E}) \leq 3 \PP_2(\mathcal{E})
%     \end{align*}
% \end{lem}
\begin{lem} 
\label{lemma:markov_to_multinomial}
    Let $N$ and $Z$ be matrices obtained under the models \eqref{def:M_hat_MC_model} and \eqref{eq:M_hat_multinomial_model}, respectively. Then, for any subset ${\cal Z}$ of $\{z\in \mathbb{N}^{n^2}:\sum_{(i,j)}z_{i,j}=T_\tau\}$, we have $\mathbb{P}(N\in {\cal Z})\le 3\mathbb{P}(Z\in {\cal Z})$. 
\end{lem}

\begin{proof}
Note that by subsampling as explained in Section \ref{subsec:transition_theorem_forward}, for each $k$ we obtain a Markov chain with transition kernel 
\begin{align*}
    P_{\tau}((y,y')|(x,x')) = P^{\tau}(y|x')P(y'|y)    
\end{align*}
and initial distribution $\nu_0^{(k)}(x,x') = \nu_0^{(k)}(x)P(x'|x)$ with
\begin{align*}
    \nu_0^{(1)}(x) = \nu_0(x)\qquad\mathrm{and}\qquad
    \nu_0^{(k)}(x)= \sum_{y\in[n]} \nu_0(y)P^{k-1}(x|y)\quad \mathrm{for}\ k=2,\dots,\tau.
\end{align*}
Moreover, all chains share the same stationary distribution given by $\Pi\in\RR^{n\times n}$ with $\Pi_{x,x'} = \nu(x)P(x'|x)$, $x,x'\in[n]$. Now, recall definition of $\tau$ from Theorem \ref{thm:main_thm} and note that according to Lemma \ref{lemma:Markov_mixing_rate} with $\delta=\frac{1}{4}$ and $\varepsilon= \nu_{\min}/(eT)$ we have $\tau(\varepsilon)\leq \tau$ and thus:
\begin{align}
    &\max_{1\leq i\leq n} \|P^{\tau}_{i,:} - \nu^\top\|_1 \leq \frac{\nu_{\min}}{eT}.
    \label{eq:markov_total_variation_final}
\end{align}

Now, let $z= (z_{i,j})_{i,j=1}^{n}\in \{ z\in \NN^{n^2}: \sum_{i,j} z_{i,j} = T_{\tau} \}$ be a tuple of fixed integers. Define a set:
\begin{align*}
    \SSS(z) := \{ (a_{2l+1},a_{2l+2})_{l=0}^{T_{\tau}-1}\in ([n]\times [n])^{T_{\tau}}: \  \sum_{l=0}^{T_{\tau}-1} \indicator_{\{ (a_{2l+1}, a_{2l+2} ) = (i,j) \} } = z_{i,j}, \forall i,j\in[n] \}
\end{align*}
and note that $|\SSS(z)| = T_{\tau}! (\prod_{i,j=1}^n z_{i,j}!)^{-1}$. By definition of Markovian and multinomial models, we have:
\begin{align*}
    \PP(N=z) = \sum \nu_0^{(k)}(x_{k-1},x_{k})\prod_{l=1}^{T_{\tau}-1} P_{\tau} ((x_{k-1+l\tau},x_{k+l\tau})|(x_{k-1+(l-1)\tau},x_{k+(l-1)\tau}))
\end{align*}
where the sum is over ${(x_{k-1+l\tau}, x_{k+l\tau} )_{l=0}^{T_{\tau}-1} \in \SSS(z) }$, and
\begin{align*}
    \PP(Z=z) = \frac{T_{\tau}!}{\prod_{i,j=1}^n z_{i,j}!} \prod_{i,j=1}^{n} \Pi_{i,j}^{z_{i,j}}.
\end{align*}
Now we fix arbitrarily one of the summands in the expression for $\PP(N=z)$ and note that:
\begin{align*}
    \absBig{ &\nu_0^{(k)}(x_{k-1},x_{k})\prod_{l=1}^{T_{\tau}-1} P_{\tau} ((x_{k-1+l\tau},x_{k+l\tau})|(x_{k-1+(l-1)\tau},x_{k+(l-1)\tau})) - \prod_{i,j=1}^{n} \Pi_{i,j}^{z_{i,j}}} \allowdisplaybreaks\\
    % &= \absBig{
    % \nu_0^{(k)}(x_{k-1},x_{k}) \prod_{l=1}^{T_{\tau}-1} P^{\tau}(x_{k-1+l\tau}|x_{k+(l-1)\tau}) P(x_{k+l\tau}|x_{k-1+l\tau})  - \prod_{l=0}^{T_{\tau}-1} \nu(x_{k-1+l\tau})P(x_{k+l\tau}|x_{k-1+l\tau})}\allowdisplaybreaks\\
    &=
    \left( \prod_{l=0}^{T_{\tau}-1} P(x_{k+l\tau}|x_{k-1+l\tau}) \right) \absBig{ \nu_0^{(k)}(x_{k-1}) \prod_{l=1}^{T_{\tau}-1}  P^{\tau}(x_{k-1+l\tau}|x_{k+(l-1)\tau}) - \prod_{l=0}^{T_{\tau}-1} \nu(x_{k-1+l\tau})}\allowdisplaybreaks \\
    &\leq 
   \left( \prod_{l=0}^{T_{\tau}-1} P(x_{k+l\tau}|x_{k-1+l\tau}) \right)  \left( 
    \prod_{l=0}^{T_{\tau}-1} (\nu(x_{k-1+l\tau}) + \epsilon) - \prod_{l=0}^{T_{\tau}-1} \nu(x_{k-1+l\tau}) \right)\allowdisplaybreaks\\
    &\leq 
    \left(\prod_{i,j=1}^{n} \Pi_{i,j}^{z_{i,j}}\right) \sum_{j=1}^{T_{\tau}} \left( \frac{\epsilon}{\nu_{\min}} \right)^j \binom{T_{\tau}}{j}  
    \leq
    \left(\prod_{i,j=1}^{n} \Pi_{i,j}^{z_{i,j}}\right)\sum_{j=1}^{T_{\tau}} \left( \frac{e T_{\tau} \epsilon}{j \nu_{\min}} \right)^j \allowdisplaybreaks \leq 2\left(\prod_{i,j=1}^{n} \Pi_{i,j}^{z_{i,j}}\right)
\end{align*}
where in first inequality we used Equation \eqref{eq:markov_total_variation_final}, where we then used the bound on binomial coefficients $\binom{T_{\tau}}{j} \leq  (e T_{\tau}/j)^j$, and where in the last inequality, we used definition of $\epsilon$. Since this upper bound holds irrespective of the summand, we deduce that:
\begin{align*}
    \abs{\PP(N=z) - \PP(Z=z)} \leq 2\frac{T_{\tau}!}{\prod_{i,j=1}^n z_{i,j}!} \left(\prod_{i,j=1}^{n} \Pi_{i,j}^{z_{i,j}}\right) = 2 \PP(Z=z).
\end{align*}
% Now, let $\mathcal{E}$ be an event consisting of simple disjoint events $A_j\in\FFF$ (and note that $j_{\max}$ is finite for $n,T_{\tau}$ finite). Then we have:
% \begin{align*}
%     \PP_1(\mathcal{E}) =\PP_1(\cup_j A_j) =  \sum_{j} \PP_1(A_j) 
%     %= \sum_{f\in\FFF} P_{\mathrm{mult}}(A_f)(1+2) 
%     \leq 3 \sum_j \PP_2(A_j) = 3 \PP_2(\mathcal{E})
% \end{align*}
% as claimed by the lemma.
Now, let ${\cal Z}$ be any subset of $\{z\in \mathbb{N}^{n^2}:\sum_{(i,j)}z_{i,j}=T_\tau\}$. Then we have:
\begin{align*}
    \PP(N\in {\cal Z}) =  \sum_{z\in{\cal Z}} \PP(N=z) 
    %= \sum_{f\in\FFF} P_{\mathrm{mult}}(A_f)(1+2) 
    \leq 3 \sum_{z\in{\cal Z}} \PP(Z=z) = 3 \PP(Z\in{\cal Z})
\end{align*}
as claimed in the lemma.

\end{proof}

\subsubsection{Poisson approximation}
\label{subsec_proof_multinomial_to_poisson}
We define a matrix $Y\in\RR^{n\times n}$ generated by the Poisson model as follows:
\begin{align}
    Y_{i,j} \sim \mathrm{Poisson}(T_{\tau}M_{i,j}),\qquad i,j=1,2,\dots,n.
    \label{eq:M_hat_Poisson_model}
\end{align}
We show that rare random events occur with approximately equal probability for the Poisson and multinomial models: 
% \begin{lem}
% \label{lemma:multinomial_to_poisson}
%     Let $(\Omega,\mathcal{F},\PP_2)$ and $(\Omega,\mathcal{F},\PP_3)$ be the probability spaces under the models \eqref{eq:M_hat_multinomial_model} and \eqref{eq:M_hat_Poisson_model}, respectively. Then for any event $\mathcal{E} \in \mathcal{F}$, we have
%     \begin{align*}
%         \PP_2(\mathcal{E}) \leq e\sqrt{T_{\tau}} \PP_3(\mathcal{E})
%     \end{align*}
% \end{lem}
\begin{lem}
\label{lemma:multinomial_to_poisson}
    Let $Z$ and $Y$ be matrices obtained under the models \eqref{eq:M_hat_multinomial_model} and \eqref{eq:M_hat_Poisson_model}, respectively. Then for any ${\cal Z}\subset \mathbb{N}^{n^2}$, we have $\mathbb{P}(Z\in {\cal Z})\le e\sqrt{T_{\tau}} \mathbb{P}(Y\in {\cal Z})$.
\end{lem}
\begin{proof}
Proof of the lemma is a straightforward consequence of Lemma \ref{lemma:hetero_poisson_multi} with parameters $T_{\tau}$, $n^2$ and $f=\indicator_{\{\mathcal{Z}\}}$.
\end{proof}

\newpage
\section{Concentration of matrices with Poisson and compound Poisson entries}
\label{sec:appendix_conc_Poisson}

As mentioned in Appendix \ref{app:poisson-approx}, our analysis relies on a Poisson approximation argument. As a result, we will require tight concentration bounds for random matrices with entries distributed according to compound Poisson distributions (when estimating the reward matrix) and Poisson distributions (when estimating the transition matrices). In  \textsection\ref{subsec:appendix-conc-Poisson:prelim}, we present a few simple facts about Poisson and compound Poisson random variables, together with some other useful tools. In \textsection\ref{subsec:appendix-conc-Poisson:conc-cp}, we present two concentration results, required for the model with compound Poisson entries. Similarly, in \textsection\ref{subsec:appendix-conc-Poisson:conc-p}, we present two concentration results, required for the model with Poisson entries. These concentration results will be extensively used in the forthcoming analysis for the subspace recovery.

It is worth noting that our results in \textsection\ref{subsec:appendix-conc-Poisson:conc-p} are sharper than those in \textsection\ref{subsec:appendix-conc-Poisson:conc-cp} thanks to Bennett's inequality. As a consequence, our results for estimating the reward matrix exhibit a dependence in $\log^3(n+m)$ while in the estimation of the transitions, our results exhibit a dependence in $\log^2(n)$ and even $\log(n)$ in some regimes.

\subsection{Preliminaries}\label{subsec:appendix-conc-Poisson:prelim}

We first present Theorem \ref{theorem:matrix_bernstein}, which can be seen as a version of matrix Bernstein inequality. The theorem is borrowed from \cite{hopkins2016fast} and relies on a truncation trick. The proofs of our concentration results in \textsection\ref{subsec:appendix-conc-Poisson:conc-cp} and \textsection\ref{subsec:appendix-conc-Poisson:conc-p}  rely on this theorem. %In the following, $\| x\|$ indicates the Euclidian norm of vector $x$.

\begin{thm}(Proposition A.3 in \cite{hopkins2016fast}) \label{theorem:matrix_bernstein}
    Let $\{Z_t\}_{t=1}^T$ be a sequence of $m \times n$ independent zero-mean real random matrices. Suppose that for all $1\leq t\leq T$,
    \begin{align}
        (i) \ \ \PP \left( \left\|Z_t\right \| \ge \beta \right) \leq p,
        \qquad \text{and} \qquad  (ii) \ \
        \left\| \EE[Z_t \indicator_{ \{ \|Z_t\| > \beta \} } ] \right\| \leq q,
        \label{eq:q0_q1_cond_matrix_bernstein}
    \end{align}
    hold for some quantities $p \in (0,1)$, and $ q \geq 0$. Furthermore, assume there exists $v \ge 0$, such that   
    \begin{align}
        (iii) \ \  \max \left\{   \left\|\sum_{t=1}^T \EE \left[Z_t Z_t^\top \right] \right\|,  \left\|\sum_{t=1}^T \EE \left[Z_t^\top Z_t\right]\right\| \right\}
        \label{eq:variance_matrix_bernstein} \le v. 
    \end{align}
    Then, for all $u > 0$,
    \begin{align}
        \PP \left( \left\|\sum_{t=1}^T Z_t \right\| \geq   T q + u \right) \leq Tp + (n+m)\exp\left( - \frac{u^2/2}{v+ \beta u/3}\right).
        \label{eq:matrix_bernstein_final_bound}
    \end{align}
\end{thm}

To apply Theorem \ref{theorem:matrix_bernstein}, we need control of the tails of the entries of the random matrix we study. In the case of Poisson entries, we will simply use the following standard fact about Poisson random variables. It is a simple consequence of Bennett's inequality \cite{bennett1962probability}.

\begin{lem}\label{lem:poisson-rv}
    Let $Y$ be a Poisson random variable with mean $\lambda$. Then for, all $\theta \in \RR$, we have $ \EE[e^{\theta Y}] \le \exp( \lambda(e^{\theta}- 1))$. Furthermore, we have for all $ u > 0$
    \begin{align*}
        \PP( \vert Y - \lambda\vert >  u) \le 2 \exp\left( - \lambda h(u/\lambda)\right) \le 2\exp\left( - \frac{u^2/2}{ \lambda + u/3 }\right),
    \end{align*}
    where $h(u)  = (1+u) \log(1+u) - u$.
\end{lem}

In the case of compound Poisson entries, we do not have any result similar to Bennett's inequality. Instead, we derive a Bernstein-type concentration result on these random variables. 
%The result is presented below in Lemma \ref{lem:cpoisson-rv}.

\begin{lem}\label{lem:cpoisson-rv}
    Let $(\xi_t)_{t \ge 1}$ be a sequence of zero-mean, $\sigma^2$-subgaussian, i.i.d. random variables. Let $Y$ be a Poisson random variables with mean $\lambda$.  Let $M$ be a positive constant. Then, the moment generating function of the compound Poisson random variable $ Z = \sum_{i=1}^Y (M + \xi_i)$ satisfies the following:
    \begin{align*}
        \forall u > 0, \qquad  \PP(  \vert Z - \lambda M  \vert  > u ) & \le 2 \exp\left(- \min\left(\frac{u^2}{16e \lambda L^2 }, \frac{u}{4L} \right) \right), \nonumber \\
        \EE\left[  \vert Z - \lambda M  \vert^2  \right] & \le 18 \lambda L^2, 
    \end{align*}
    where $L = \max (M, \sigma)$.
\end{lem}

\begin{proof}[Proof of Lemma \ref{lem:cpoisson-rv}] First, we upper bound the moment generating function of $\sum_{i=10}^Y (M + \xi_i)$. Let $\theta > 0$,  we have 
\begin{align*}
     I_Z(\theta) \triangleq \EE\left[e^{\theta(\sum_{i=1}^Y (M + \xi_i))}\right] & \le \sqrt{\EE\left[ e^{2 \theta M Y}\right] \EE[e^{2 \theta \sum_{i=0}^Y \xi_i }]} \nonumber  \\
    & \le \exp\left(\frac{\lambda (e^{ 2 \theta M} - 1)}{2}\right) \sqrt{ \EE[e^{2 \theta \sum_{i=0}^Y \xi_i }] } \nonumber \\
    & \le \exp\left( \frac{\lambda}{2} \left( (2\theta M)^2 e^{2 \theta M }  + 2\theta M\right)\right)\sqrt{ \EE[e^{2 \theta \sum_{i=0}^Y \xi_i }] },
\end{align*}
    where in the first inequality, we use Cauchy-Schwarz inequality, in the second inequality, we use the well known bound on the moment generating function of a Poisson random variable (if $Y$ is a Poisson random variable with mean $\lambda$, then for all $\theta > 0$, $\EE[e^{\theta Y}] \le \exp( \lambda(e^{\theta}- 1))$), and in the last inequality, we use the elementary fact that $e^{x} - 1 \le x^2 e^{x} + x$ for all $x \in \RR$. Next, we have 
    \begin{align*}
        \EE\left[e^{2 \theta \sum_{i=1}^Y \xi_i }\right] & = \EE\left[ \sum_{k=1}^\infty \indicator_{\lbrace Y = k\rbrace } \exp\left({2 \theta \sum_{i=1}^k\xi_i}\right)\right] \nonumber \\
        & = \sum_{k=1}^\infty\PP(Y = k) \EE\left[\exp\left(2 \theta \sum_{i=1}^k \xi_i\right) \right] \nonumber \\
        & \le \sum_{k=1}^\infty \PP(Y = k) \exp(2k \theta^2\sigma^2) \nonumber \\
        & \le \exp \left( \lambda (e^{2\theta^2 \sigma^2} - 1)\right) \nonumber\\
        & \le \exp\left( \lambda \left(2\theta^2 \sigma^2  e^{ 2 \theta^2 \sigma^2  } \right) \right),
    \end{align*} 
    where we use the fact that the $\xi_i$ are $\sigma^2$-subgaussian r.v., and the elementary inequality $e^{x^2} - 1 \le x ^2 e^{x^2}$ for all $x \in \RR$.  We conclude that 
    for all $\theta > 0$, 
    \begin{align*}
        I_Z(\theta) \le \exp\left(  \lambda \left(2  \theta^2 M^2 e^{2\theta M} + 2 \theta^2 \sigma^2 e^{2 \theta^2 \sigma^2 }\right)  + \lambda \theta M  \right).
    \end{align*}
    Next, we introduce $L = \max(M, \sigma)$. Then, for all $\alpha  > 0$, we deduce that 
    \begin{align*}
        I_Z(\theta) \le \exp\left(  2 \lambda   \theta^2 L^2 \left(e^{\alpha} + e^{\alpha^2}\right)   + \lambda \theta M  \right), \qquad \forall \vert \theta \vert \le \frac{\alpha}{2 L}.
    \end{align*}
    By Markov inequality, and fixing $\alpha = 1$, we have 
    \begin{align*}
        \PP(  Z - \lambda M > u ) \le \inf_{\vert \theta \vert \le 1/(2L)}I_Z(\theta) e^{- \lambda \theta  M  - \theta u} \le \exp\left(- \min\left(\frac{u^2}{16e \lambda L^2 }, \frac{u}{4L} \right) \right). 
    \end{align*}
    Similarly, we have 
    \begin{align*}
            \PP(  \lambda M - Z > u ) \le \exp\left(- \min\left(\frac{u^2}{16e \lambda L^2 }, \frac{u}{4L} \right) \right).
    \end{align*}
    The final tail bound follows from a union bound.
    Finally, straightforward computations yield an upper bound on  $\EE[\vert \lambda M - Z \vert^2 ]$. Indeed, we have  
    \begin{align*}
        \EE[\vert \lambda M - Z \vert^2 ] & \le  2  \EE[ \vert Y - \lambda \vert^2  ]  M^2 + 2 \EE\left[\left(\sum_{i=1}^Y \xi_i\right)^2\right] \le 2  \lambda M^2 + 16  \lambda \sigma^2 \le 18 \lambda L^2.
    \end{align*}
\end{proof}

\subsection{Random matrices with compound Poisson entries}\label{subsec:appendix-conc-Poisson:conc-cp}

We list below the two main concentration results that we need for the forthcoming analysis. In Proposition \ref{prop:reward-concentration-1}, we provide a high probability guarantee on the error between the empirical mean reward matrix and the true matrix in operator norm. In Proposition \ref{prop:reward-concentration-2}, we establish another concentration result that will be instrumental in the subspace recovery analysis. The proofs of the two results are similar with slight differences and they both rely on Theorem \ref{theorem:matrix_bernstein}. The proofs are presented at the end of this subsection.

\begin{prop}\label{prop:reward-concentration-1}
    Under the random matrix model \eqref{eq:reward-poisson-matrix-plus-noise-model} with compound Poisson entries, for all $\delta \in (0,1)$, for all $
        T \ge  13 (n+m) \log^3\left( (n+m)/\delta\right)$, the following statement
    \begin{align*}
         \Vert \widetilde{M} - M \Vert  \le  36\sqrt{2}L  \sqrt{\frac{nm}{T}}\left( \sqrt{ (n+m)  \log\left(\frac{n+m}{\delta}\right)  }  + \log^{3/2}\left(\frac{n+m}{\delta}\right)   \right)  
    \end{align*}
    holds with probability at least $1-\delta$, where $L = \max(\Vert M \Vert_{\infty}, \sigma)$.
\end{prop}

\begin{prop}\label{prop:reward-concentration-2}
    Let $A$ be a $m \times 2r$ nonrandom matrix, and $B$ be a $n \times 2r$ nonrandom matrix. Then, under the random matrix model \eqref{eq:reward-poisson-matrix-plus-noise-model} with compound Poisson entries, and denoting $L = \max(\Vert M\Vert_\infty, \sigma)$, we have: 
    \begin{itemize}
        \item [(i)] for all $\ell \in [m]$, for all $\delta\in (0,1)$, for all $T \ge m \log^3(en/\delta)$, the following event   
        \begin{align}\label{prop:eq:statement}
            \Vert (\widetilde{M}_{\ell,:} - M_{\ell, :})  A \Vert \le 73\sqrt{2} L \Vert A \Vert_{2\to \infty} \sqrt{\frac{nm}{T}}\left(\sqrt{ n \log\left(\frac{en}{\delta}\right)} +  \log^{3/2}\left( \frac{en}{\delta}\right) \right) 
        \end{align}
        holds with probability at least $1-\delta$;
        \item[(ii)]  for all $k \in [n]$, for all $\delta\in (0,1)$, for all $T \ge n \log^3(em/\delta)$, the following event 
        \begin{align}\label{prop:eq:statement-transpose}
            \!\! \Vert (\widetilde{M}_{:,k} - M_{:,k})^\top  B \Vert \le 73\sqrt{2} L \Vert B \Vert_{2\to \infty}\sqrt{\frac{nm}{T}}\left(\sqrt{ m \log\left(\frac{em}{\delta}\right)} +  \log^{3/2}\left( \frac{em}{\delta}\right) \right) 
        \end{align}
        holds with probability at least $1-\delta$.
    \end{itemize}    
\end{prop}

\begin{proof}[Proof of Proposition \ref{prop:reward-concentration-1}]
    To simplify the notation, introduce the matrices $Z_{i,j} = (\widetilde{M}_{i,j} - M_{i,j}) e_i e_j^\top$, for all $(i,j)\in [m] \times [n]$, $\lambda = T/mn$, and $L = \max(\Vert M\Vert_{\infty}, \sigma)$. We remark that we can write
    \begin{align*}
        \widetilde{M} - M = \sum_{(i,j) \in [m] \times [n]} Z_{i,j}.
    \end{align*}
    Starting from the above expression, we will apply Theorem \ref{theorem:matrix_bernstein} to obtain the desired result. First, we note that for all $(i,j) \in [m] \times [n]$, $\Vert Z_{i,j}\Vert = \vert \widetilde{M}_{i,j} - M_{i,j}\vert$ and $\widetilde{M}_{i,j} - M_{i,j}$ is a centered and normalized compound Poisson random variable. Thus, we have by Lemma \ref{lem:cpoisson-rv}, for all $\delta \in (0,1)$,  $\PP(\Vert Z_{i,j}\Vert > \beta   ) \le  \delta/(2n^2m^2)$, where we define
    \begin{align*}
        \beta & = 4L \max\left( \sqrt{\frac{e}{\lambda} \log\left( \frac{4n^2m^2}{\delta}\right)}, \frac{1}{\lambda} \log\left( \frac{4n^2m^2}{\delta}\right)  \right), \nonumber \\
        & \le 4L \max\left( \sqrt{\frac{4e}{\lambda} \log\left( \frac{n+m}{\delta}\right)}, \frac{4}{\lambda} \log\left( \frac{n+m}{\delta}\right)  \right).
    \end{align*}
    Moreover, we have 
    \begin{align*}
        \EE\left[ \Vert Z_{i,j}\Vert \indicator_{\lbrace \Vert Z_{i,j] \Vert > \beta } \rbrace}\right] & \le \sqrt{ \EE\left[\Vert Z_{i,j}\Vert^2 \right]  \EE\left[  \indicator_{\lbrace \Vert Z_{i,j] \Vert > \beta } \rbrace} \right]} \nonumber \\
        & \le \sqrt{ \EE[\vert \widetilde{M}_{i,j} - M_{i,j} \vert^2]  \PP(  \Vert Z_{i,j} \Vert > \beta )  } \nonumber \\
        & \le \sqrt{  \frac{9 L^2 \delta}{\lambda n^2 m^2} },
    \end{align*}
    where the first inequality follows from Cauchy-Schwarz inequality, the second inequality follows from the expression of $Z_{i,j}$, and the third inequality follows from Lemma \ref{lem:cpoisson-rv}. Next, we have 
    \begin{align*}
        \left\Vert \sum_{(i,j)\in [m]\times [n]} \EE\left[ Z_{i,j} Z_{i,j}^\top \right]   \right\Vert & = \left\vert   \sum_{i\in [m]} \left( \sum_{j \in [n]} \EE\left[\left(\widetilde{M}_{i,j} - M_{i,j}\right)^2  \right] \right) e_i e_i^\top   \right\vert  \\
        & = \max_{i \in [m]} \sum_{j \in [n]} \EE\left[\left(\widetilde{M}_{i,j} - M_{i,j}\right)^2  \right] \\
        & \le \frac{18 n L^2}{\lambda}.
    \end{align*}
    By symmetry, we obtain similarly  
    \begin{align*}
        \left\Vert \sum_{(i,j)\in [m]\times [n]} \EE\left[  Z_{i,j}^\top  Z_{i,j}\right]   \right\Vert &  \le \frac{18 m L^2}{\lambda}.
    \end{align*}
    Let us set $v = 18(n\wedge m)L^2/\lambda$.
    We conclude using Theorem \ref{theorem:matrix_bernstein} that, for all $u > 0$,
    \begin{align*}
        \PP\left(  \Vert\widetilde{M} - M\Vert  > \sqrt{ \frac{9 L^2 \delta}{\lambda} } + u \right) & \le \frac{\delta}{2(nm)} + (n+m) \exp\left( - \frac{u^2/2}{v + \beta u / 3}\right) \\
        & \le  \frac{\delta}{2(nm)} + (n+m) \exp\left( -  \frac{1}{4}\min\left( \frac{u^2}{v}, \frac{3u}{\beta}
        \right) \right). 
    \end{align*}
    We re-parametrize by choosing $\delta = 2(n+m)\exp( - (1/4) \min ( u^2/v, 3u/\beta   )  )$, and write 
    \begin{align}\label{eq:prop-reward-prob}
        \PP\left(  \Vert\widetilde{M} - M\Vert  > \frac{3 L \sqrt{\delta}}{\sqrt{\lambda}}  + u  \right) \le \delta
    \end{align}
    with 
    \begin{align*}
        u & = \max \left(\sqrt{  4 v \log\left( \frac{2(n+m)}{\delta}\right) },  \frac{4 \beta}{3}\log\left( \frac{2(n+m}{\delta}\right) \right) \\
        & \le \max \left(\sqrt{  8 v \log\left( \frac{n+m}{\delta}\right) },  \frac{8 \beta}{3}\log\left( \frac{n+m}{\delta}\right) \right).
    \end{align*}
    By inspecting the definition of $\beta$ and $v$, we note that under the condition 
    \begin{align}\label{eq:prop-reward-condition}
        \lambda = \frac{T}{nm}\ge \frac{4^5}{3^4} \frac{1}{n\wedge m} \log^3\left( \frac{n+m}{\delta}\right) 
    \end{align}
    then 
    \begin{align*}
        u & \le \max \left(  \sqrt{ 8 v \log\left( \frac{n+m}{\delta}\right)}, \frac{16 L \sqrt{2e}}{3\sqrt{\lambda}}  \log^{3/2}\left( \frac{2(n+m)}{\delta}\right)   \right) \\
        & \le \frac{L}{\sqrt{\lambda}} \max\left( \sqrt{ 3^2 4^2 (n \wedge m) \log\left( \frac{n+m}{\delta}\right)}, \frac{ 4^2 2\sqrt{e}}{3} \log\left( \frac{n+m}{\delta}\right)\right).
    \end{align*}
    After using the upper bound on $u$ in \eqref{eq:prop-reward-prob}, and after upper bounding $\delta$ by $1$, we obtain, under the condition \eqref{eq:prop-reward-condition}, 
    \begin{align*}
        \Vert \widetilde{M} -  M \Vert & > \frac{L}{\sqrt{\lambda}}  \left( 3   + 12 \sqrt{ 2 (n+m)  \log\left(\frac{n+m}{\delta}\right)  }  + \frac{4^3\sqrt{e}}{3} \log^{3/2}\left(\frac{n+m}{\delta}\right)   \right) \\
        & > \frac{L}{\sqrt{\lambda}}  \left( 36 \sqrt{ 2 (n+m)  \log\left(\frac{n+m}{\delta}\right)  }  + 36\log^{3/2}\left(\frac{n+m}{\delta}\right)   \right) \\
        & > \frac{36\sqrt{2}L}{\sqrt{\lambda}}  \left( \sqrt{ (n+m)  \log\left(\frac{n+m}{\delta}\right)  }  + \log^{3/2}\left(\frac{n+m}{\delta}\right)   \right)
    \end{align*}
    with probability at most $\delta$. Noting that a stricter condition than \eqref{eq:prop-reward-condition} is  
    \begin{align*}
        T \ge  13 (n+m) \log^3\left( \frac{n+m}{\delta}\right),
    \end{align*}
    we complete the proof.
\end{proof}

\begin{proof}[Proof of Proposition \ref{prop:reward-concentration-2}]
   To simplify the notation, let us denote $Z_j = (\widetilde{M}_{\ell,j} - M_{\ell,j}) A_{j, :}$, $\lambda = mn/T$, and $L = \max(\Vert M \Vert_\infty, \sigma)$. We remark that we can write 
    \begin{align*}
        (\widetilde{M}_{\ell,:} - M_{\ell,:}) A = \sum_{j\in[n]} (\widetilde{M}_{\ell,j} - M_{\ell,j}) A_{j, :} =  \sum_{j\in[n]} Z_j.
    \end{align*}
    Starting from the above expression, we will apply \ref{theorem:matrix_bernstein} to obtain the desired result. First, we note that for all $j \in [n]$, $\Vert Z_j \Vert = \vert \widetilde{M}_{\ell,j} - M_{\ell,j}\vert  \Vert A_{j, :} \Vert$, and $\widetilde{M}_{\ell,j} - M_{\ell,j}$ is a centered and normalized compound Poisson random variable. Thus, we have by Lemma \ref{lem:cpoisson-rv}, for all $\delta \in (0, 1)$, $\PP\left(  \Vert Z_j\Vert > \Vert A \Vert_{2 \to \infty} \beta \right) \le \PP\left(  \Vert Z_j\Vert > \Vert A_{j, :} \Vert \beta \right) \le \delta/(2n^2)$, where we define 
    \begin{align*}
        \beta & = 4 L \max\left(   \sqrt{\frac{e}{\lambda} \log\left( \frac{4n^2}{\delta}\right)  }, \frac{1}{\lambda}  \log\left( \frac{4n^2}{\delta}\right) \right) \\
        & \le 4L \max\left(   \sqrt{\frac{2e}{\lambda} \log\left( \frac{en}{\delta}\right) }, \frac{2}{\lambda}  \log\left( \frac{en}{\delta}\right) \right).
    \end{align*}
    Moreover, we have 
    \begin{align*}
        \EE\left[ \Vert Z_{j} \Vert \indicator_{  \lbrace \Vert Z_j \Vert > \Vert A \Vert_{2 \to \infty} \beta \rbrace  }   \right] & \le \sqrt{\EE[ \Vert Z_{j} \Vert^2] \PP\left( \Vert Z_j \Vert > \Vert A \Vert_{2 \to \infty} \beta\right)  } \\
        & \le \Vert A \Vert_{2 \to \infty} \sqrt{   \frac{ \EE[\Vert \widetilde{M}_{\ell, :} - M_{\ell, :}\Vert^2]\delta}{2 n^2} } \\
        &  \le \Vert A \Vert_{2 \to \infty} \sqrt{ \frac{ 9 L^2\delta}{\lambda n^2}},
    \end{align*}
    where in the first inequality, we use Cauchy-Schwarz inequality, and in the third inequality, the result of Lemma \ref{lem:cpoisson-rv} to upper bound the variances. Next, we have
    \begin{align*}
        \left\Vert \EE\left[ \sum_{j\in [n]}  Z_j Z_j^\top    \right] \right\Vert & \le  \sum_{j\in [n]}  \EE\left[ (\widetilde{M}_{\ell,j} - M_{\ell, j} )^2  \right] \left\Vert  A_{j,:} \right\Vert^2 \\
        & \le \frac{18 L^2 \Vert A \Vert_F^2}{\lambda} \\
        & \le \frac{18 L^2 n \Vert A \Vert_{2 \to \infty}^2}{\lambda},  
    \end{align*}
    where we simply used the expressions of $Z_j$, $j \in [n]$, the triangular inequality, and Lemma \ref{lem:cpoisson-rv} to upper bound the variances. Similarly, we have 
    \begin{align*}
        \left\Vert \EE\left[ \sum_{j\in [n]}  Z_j^\top Z_j    \right] \right\Vert & \le \frac{18  L^2 n \Vert A \Vert_{2 \to \infty}^2}{\lambda}. 
    \end{align*}
    We set $v = 18 L^2 n \Vert A \Vert_{2 \to \infty}^2 / \lambda$. Now we are ready to apply Theorem \ref{theorem:matrix_bernstein}. We get: 
    \begin{align*}
        \PP\left(\Vert (\widetilde{M}_{\ell,:} - M_{\ell,:}) A \Vert  > \Vert A \Vert_{2 \to \infty} \sqrt{\frac{9 L^2 \delta}{\lambda}}  + u  \right) \le \frac{\delta}{2n} + n \exp\left( -\frac{1}{4} \min \left(\frac{u^2}{v}, \frac{3u}{\Vert A \Vert_{2 \to \infty} \beta}
        \right) \right).
    \end{align*}
    We re-parametrize by choosing $\delta = 2 n \exp(-(1/4)\min(u^2/v, 3u/(\Vert A \Vert_{2 \to \infty}\beta)  ))$ and we write 
    \begin{align*}
        \PP\left(\Vert (\widetilde{M}_{\ell,:} - M_{\ell,:}) A \Vert  > \Vert A \Vert_{2 \to \infty} \sqrt{\frac{9 L^2 \delta}{\lambda}}  + u  \right) \le \delta
    \end{align*}
    with 
    \begin{align*}
        u & = \max\left(  \sqrt{4v \log\left( \frac{2n}{\delta}\right)}, \frac{4 \Vert A \Vert_{2 \to \infty}\beta}{3} \log\left( \frac{2n}{\delta}\right)  \right) \\
        & \le \max\left(  \sqrt{4v \log\left( \frac{e n}{\delta}\right)}, \frac{4 \Vert A \Vert_{2 \to \infty}\beta}{3} \log\left( \frac{e n}{\delta}\right)  \right).
    \end{align*}
    By inspecting the definition of $\beta$ and $ v$, we note that when the condition 
    \begin{align}\label{eq: condition-prop-vec}
        \lambda = \frac{T}{mn} \ge \frac{4^3}{3^4} \frac{1}{n} \log^3\left(\frac{e n}{\delta}\right)
    \end{align}
    holds, then 
    \begin{align}
        u & \le \max\left( \sqrt{4v\log\left(\frac{en}{\delta}\right)}, \frac{16\sqrt{2e} L\Vert A \Vert_{2 \to \infty}  }{3\sqrt{\lambda}} \log^{3/2}\left( \frac{en}{\delta}\right)\right) \nonumber \\
        & \le \frac{L \Vert A \Vert_{2 \to \infty} }{\sqrt{\lambda}} \max\left( \sqrt{2^3 3^2 n \log\left(\frac{en}{\delta}\right)}, \frac{16\sqrt{2e} }{3} \log^{3/2}\left( \frac{en}{\delta}\right)\right) \nonumber \\
        & \le \frac{36 \sqrt{2} L \Vert A \Vert_{2 \to \infty}  }{\sqrt{\lambda}} \max\left( \sqrt{ n \log\left(\frac{en}{\delta}\right)},  \log^{3/2}\left( \frac{en}{\delta}\right)\right). 
        \label{eq:upper-bound-u-vec}
    \end{align}
    After using the upper bound in \eqref{eq:upper-bound-u-vec}, and upper bounding $\delta$ by 1, we obtain that,  under the condition \eqref{eq: condition-prop-vec},  
    \begin{align*}
        \Vert (\widetilde{M}_{\ell,:} - M_{\ell, :})  A \Vert > \frac{73\sqrt{2} L \Vert A \Vert_{2\to \infty}}{\sqrt{\lambda}} \left(\sqrt{ n \log\left(\frac{en}{\delta}\right)} +   \log^{3/2}\left( \frac{en}{\delta}\right) \right) 
    \end{align*}
    holds with probability at most $\delta$. We can also refine the condition \eqref{eq: condition-prop-vec} as follows 
    \begin{align*}
        T \ge  m \log^3\left( \frac{e n}{\delta}\right).
    \end{align*}
    This concludes the proof of the statement \eqref{prop:eq:statement} in the proposition. The statement \eqref{prop:eq:statement-transpose} follows similarly. Therefore, we omit it.
\end{proof}

\subsection{Random matrices with Poisson entries} \label{subsec:appendix-conc-Poisson:conc-p}

Recall from Section \ref{subsec:appendix_tight_thm_generative}, the definition of the function $g_\delta$ from \eqref{eq:g_def} and that ${\cal A}=\frac{1}{\sqrt{T}} \sqrt{\|M\|_{1\to\infty}+\|M^\top\|_{1\to\infty}}$. First we show the following lemma that provides an upper bound of the spectral norm. This lemma is used to derive Lemma \ref{lem:spectral_Multinomial}.

\begin{lem} Let $Y\in\RR^{n\times n}$ be a matrix with independent entries $Y_{i,j} \sim  T^{-1}\mathrm{Poisson}(TM_{ij})$, $i,j\in[n]$, and let $0\leq \delta \leq 1$. Then, w.p. at least $1-\delta$,
   % \begin{align}
     $\| Y-M\| \leq C{\cal A}  + \frac{C}{T} g_{\delta}(T M) \sqrt{\log( \frac{ne}{\delta})}.$ 
    % \end{align}
    \label{lem:spectral_Poisson}
\end{lem}
\begin{proof}
The proof follows from that of Lemma \ref{lemma:El_A_conc} and that of Lemma 4 in \cite{mcrae2021low}, which is based on a spectral bound from \cite{bandeira2016sharp}. We use that the random variables $|Y_{i,j}-M_{i,j}|$ concentrate well around $L = L_1 \indicator_{\{\exists \ell: T\| M_{\ell,:} \|_{\infty}\leq 1\}} + L_2 \indicator_{\{\forall \ell: T\| M_{\ell,:} \|_{\infty}> 1\}}  $ where $L_1 = 4 T^{-1} \log^{-1}(1+ (T\|M\|_{\infty})^{-1}\wedge n\delta^{-1} )\log (\frac{ne}{\delta})$ and $L_2= 4 \sqrt{ T^{-1} \| M\|_{\infty}} \log \left( T \| M\|_{\infty} \frac{ne}{\delta} \right)$ using exactly the same argument as in the first step of Lemma \ref{lemma:El_A_conc}. Moreover, we use upper bound on $\abs{ \EE [ (Y_{i,j}-M_{i,j}) \indicator_{ \{ \abs{Y_{i,j} - M_{i,j}}  < L \} }  ]}$ derived in the second step of Lemma \ref{lemma:El_A_conc}.
\end{proof}

We also derive upper bounds in the $\ell_{2\to\infty}$ norm. These bounds are used in the analysis of the singular subspace recovery in Lemma \ref{lemma:U_to_EU}, and therefore in the proofs of Theorems \ref{thm:generative} and \ref{thm:main_thm}. 

\begin{lem}
\label{lemma:El_A_conc}
Let $Y\in\RR^{n\times n}$ be a matrix with independent entries $Y_{i,j} \sim  T^{-1}\mathrm{Poisson}(TM_{ij})$, $i,j\in[n]$, for an arbitrary integer $T>0$. Let $0\leq \delta \leq 1$. Then, for any $1\leq l\leq n$ and any matrix $A\in \RR^{n\times p}$, with $p\leq n$, and independent of $Y_{l,:}$ we have, if $T\|M_{l,:}\|_{\infty} \leq 1$,
        \begin{align*}
        \|(Y_{l,:}-M_{l,:})  A\| \lesssim
         & \|A\|_F \frac{\sqrt{\|M_{l,:}\|_{\infty} \log \left(\frac{ne}{\delta}\right)}}{\sqrt{T}} + \|A\|_{2\to\infty}\frac{\log^2 \left(\frac{ne}{\delta}\right)}{T \log(1+ (T\|M_{l,:}\|_{\infty})^{-1}\wedge n\delta^{-1} ))} 
         \end{align*}
        else if $T\|M_{l,:}\|_{\infty}  > 1$,
      \begin{align*}
        \|(Y_{l,:}-M_{l,:})  A\| \lesssim
         &  \|A\|_F \frac{\sqrt{\|M_{l,:}\|_{\infty} \log \left(\frac{ne}{\delta}\right)}}{\sqrt{T}} + \|A\|_{2\to\infty}\frac{\sqrt{\|M_{l,:}\|_{\infty}}}{\sqrt{T}} \log\left(T\|M_{l,:}\|_{\infty}\frac{ ne}{\delta}\right)\log \left(\frac{ne}{\delta}\right) 
         \end{align*}
with probability at least $1-\delta/n$.
\end{lem}
\begin{proof}[Proof of Lemma \ref{lemma:El_A_conc}]
The lemma is an application of the truncated matrix Bernstein theorem i.e. Theorem \ref{theorem:matrix_bernstein}. In this theorem, $T$ corresponds to $n$ in Lemma \ref{lemma:El_A_conc}, $n$ in Theorem \ref{theorem:matrix_bernstein} corresponds to $1$ in Lemma \ref{lemma:El_A_conc}, and $m$ in Theorem \ref{theorem:matrix_bernstein} corresponds to $n$ in Lemma \ref{lemma:El_A_conc}. First note that for any $l$, we have $
(Y_{l,:}- M_{l,:})  A = \sum_{i=1}^n (Y_{l,i}- M_{l,i})A_{i,:} $. Moreover, since each of these $n$ summands are independent, zero-mean random vectors, we can identify $Z_i$'s from Theorem \ref{theorem:matrix_bernstein} with $(Y_{l,i} - M_{l,i}) A_{i,:} \in\RR^{1\times n}$ for $i\in[n]$. To apply Theorem \ref{theorem:matrix_bernstein}, we need to verify its assumptions. This is done below.

\textbf{Step 1: Showing \emph{(i)} in \eqref{eq:q0_q1_cond_matrix_bernstein}}
% First, recall that for $y \sim  \mathrm{Poisson}(\mu)$, we have according to Bennett's inequality for any $t>-1$:
% \begin{align*}
%     \PP (\abs{y - \mu} \geq t \mu) \leq 2\exp \left( -h(t)\mu \right)
% \end{align*}
% with $h(t) = (1+t)\log(1+t)-t$. 
First, recall Bennett's concentration inequality from Lemma \ref{lem:poisson-rv}, which in our case implies that for any $i,j\in[n]$:
\begin{align}
    \PP (\abs{Y_{i,j} - M_{i,j}} \geq t M_{i,j}) \leq 2\exp \left( -h(t) TM_{i,j} \right).
\label{eq:equivalence_subexponential_norm_prob}
\end{align}
Note that $\|Z_i\|$ in Theorem \ref{theorem:matrix_bernstein} in our case corresponds to:
\begin{align*}
    \|(Y_{l,i} - M_{l,i}) A_{i,:}\| = \abs{Y_{l,i} -  M_{l,i}}\|A_{i,:}\|
    \leq
    \abs{Y_{l,i} - M_{l,i}} \|A\|_{2\to\infty}.
\end{align*}
We consider two different cases:

    1. $T\|M_{l,:}\|_{\infty} \leq 1$: We let $\beta_1 = 4 T^{-1} \|A\|_{2\to\infty} \log^{-1}(1+ (T\|M_{l,:}\|_{\infty})^{-1}\wedge n\delta^{-1} )\log (\frac{ne}{\delta})$ and note that $h(t) \geq \frac{1}{2}t\log t$ for $t\geq 1$. Thus, from Equation \eqref{eq:equivalence_subexponential_norm_prob}, we have:
    \begin{align*}
        \PP \left( \abs{Y_{l,i} - M_{l,i}} \geq \frac{\beta_1}{\|A\|_{2\to\infty}}\right)    &\leq
        2\exp\bigg( - 2\log(\frac{ne}{\delta})\log^{-1}(1+ (T\|M_{l,:}\|_{\infty})^{-1}\wedge n\delta^{-1} ) \\ &\cdot\log \left( \frac{4\log(\frac{ne}{\delta})}{T\|M_{l,:}\|_{\infty} \log(1+(T\|M_{l,:}\|_{\infty})^{-1}\wedge n\delta^{-1}) } \right) \bigg)
        %\PP \left( \abs{m_{l,i} - \EE m_{l,i}} \geq c \log^{-1}\left( 1+{\frac{c_5}{\max_{i} \mu_{l,i}}} \right)\log n \right) 
        \leq \frac{\delta}{2n^2}.
    \end{align*}
    where, in the second inequality, we show using simple algebra that $\log^{-1}(1+ (T\|M_{l,:}\|_{\infty})^{-1}\wedge n\delta^{-1} )\log \left( \frac{4\log(\frac{ne}{\delta})}{T\|M_{l,:}\|_{\infty} \log(1+(T\|M_{l,:}\|_{\infty})^{-1}\wedge n\delta^{-1}) } \right) \geq 1$ for $\delta \leq 1$ and $T\|M_{l,:}\|_{\infty} \leq 1$.\\
    2. $T\|M_{l,:}\|_{\infty}  > 1$: Here we define $\beta_2:= 4\|A\|_{2\to\infty} \sqrt{ T^{-1} \| M_{l,:}\|_{\infty}} \log \left( T \| M_{l,:}\|_{\infty} \frac{ne}{\delta} \right)$. Then, according to Equation \eqref{eq:equivalence_subexponential_norm_prob} and the approximation $h(t)\geq \min \{t^2/4,t\}$ for $t\geq 0$, we have:
    \begin{align*}
        \PP &\left( \abs{Y_{l,i} - M_{l,i}} \geq \frac{\beta_2}{\|A\|_{2\to\infty}}\right) \\
        &\leq 
        2\exp \bigg( - 4\min\bigg\{ \log^2(T\|M_{l,:}\|_{\infty} \frac{ne}{\delta}) , \sqrt{T\|M_{l,:}\|_{\infty}}\log(T\|M_{l,:}\|_{\infty}\frac{ne}{\delta})\bigg\} \bigg) \\
        &\leq 
        2\exp \bigg( - 4 \log(T\|M_{l,:}\|_{\infty} \frac{ne}{\delta}) \bigg)
        \leq \frac{1}{2T\|M_{l,:}\|_{\infty}}\frac{\delta}{n^2}.
    \end{align*}
where, in the second inequality, we used that $\delta\leq 1$ and $T\|M_{l,:}\|_{\infty}  > 1$. Finally, we define $\beta = \beta_1 \indicator_{\{T\|M_{l,:}\|_{\infty}  \leq 1 \}} + \beta_2\indicator_{\{T\|M_{l,:}\|_{\infty}  > 1 \}}  $ and $p= \frac{\delta}{2n}\indicator_{\{T\|M_{l,:}\|_{\infty}  \leq 1 \}} + \frac{1}{2T\|M_{l,:}\|_{\infty}}\frac{\delta}{n}\indicator_{\{T\|M_{l,:}\|_{\infty}  > 1 \}}$ (since we took union bound over $i\in[n]$).

\noindent \textbf{Step 2: Showing \emph{(ii)} in \eqref{eq:q0_q1_cond_matrix_bernstein}}
In our case the l.h.s. corresponds to $\| \EE [(Y_{l,i} - M_{l,i})  A_{i,:} \indicator_{ \{ \| (Y_{l,i} - M_{l,i}) A_{i,:}\| > \beta\} } ] \| = \| \EE [(Y_{l,i} - M_{l,i})  A_{i,:} \indicator_{ \{ \| (Y_{l,i} - M_{l,i}) A_{i,:}\| \le \beta\} } ] \| $ which can be upper bounded by $\|A\|_{2\to\infty}
    |\EE [ (Y_{l,i}-M_{l,i}) \indicator_{\{\vert Y_{l,i} - M_{l,i} \vert   \le \frac{\beta}{\|A_{i,:}\|} \} }  ] | 
$. For some integers $\kappa_{\min},\kappa_{\max}$, let $Y_{l,i}\in \frac{1}{T}[\kappa_{\min},\kappa_{\max}]$ be interval of $Y_{l,i}$ for which indicator $\indicator_{\{\vert Y_{l,i} - M_{l,i} \vert   \le \frac{\beta}{\|A_{i,:}\|} \} } $ is active and note that this is a superset of interval for which $\indicator_{\{\vert Y_{l,i} - M_{l,i} \vert   \le \frac{\beta}{\|A\|_{2\to\infty}} \} } $ is active. Then from the definition of Poisson random variables and the bounds derived previously, we obtain:
% \begin{align*}
%     \absBig{ \EE [ (Y_{l,i}-M_{l,i}) &\indicator_{ \{ \abs{Y_{l,i} - M_{l,i}}  < \frac{L}{\|A_{i,:}\|} \} }  ]} = \absBig{\sum_{k=\kappa_{\min}}^{\kappa_{\max}} (k-M_{l,i}) \frac{\exp(-M_{l,i}) M_{l,i}^k}{k!}}\\
%     &=
%     M_{l,i} \absBig{  \sum_{k=\kappa_{\min}-1}^{\kappa_{\max}-1}  \frac{\exp(-M_{l,i})M_{l,i}^{k}}{k!} - \sum_{k=\kappa_{\min}}^{\kappa_{\max}} \frac{\exp(-M_{l,i}) M_{l,i}^k}{k!} }\\
%     &\leq 
%     M_{l,i} (\PP(Y_{l,i}= \kappa_{\min}-1) + \PP(Y_{l,i}= \kappa_{\max}) )
%     \leq
%     \frac{\delta}{n^2}\min\{\|M_{l,:}\|_{\infty},1 \}
% \end{align*}
\begin{align*}
    \absBig{ \EE [ (Y_{l,i}-M_{l,i}) &\indicator_{ \{ \abs{Y_{l,i} - M_{l,i}}  \le \frac{\beta}{\|A_{i,:}\|} \} }  ]} = \frac{1}{T} \absBig{\sum_{k=\kappa_{\min}}^{\kappa_{\max}} (k-TM_{l,i}) \frac{\exp(-T M_{l,i}) (TM_{l,i})^k}{k!}}\\
    &=
    M_{l,i} \absBig{  \sum_{k=\kappa_{\min}-1}^{\kappa_{\max}-1}  \frac{\exp(-TM_{l,i})(TM_{l,i})^{k}}{k!} - \sum_{k=\kappa_{\min}}^{\kappa_{\max}} \frac{\exp(-TM_{l,i}) (TM_{l,i})^k}{k!} }\\
    &\leq 
    M_{l,i} (\PP(T Y_{l,i}= \kappa_{\min}-1) + \PP(T Y_{l,i}= \kappa_{\max}) )
    \leq
    \frac{2\delta}{T n^2}\min\{T\|M_{l,:}\|_{\infty},1 \},
\end{align*}
where we assumed that $\kappa_{\min} \geq 1$, otherwise we keep just the second probability term above.
Thus, using previous two inequalities, we have:
\begin{align*}
    \| \EE[(Y_{l,i} - M_{l,i}) A_{i,:} \indicator_{ \{ \|(Y_{l,i} - M_{l,i}) A_{i,:}\| > \beta  \} } ] \| \leq  \|A\|_{2\to\infty} \frac{2\delta}{T n^2}\min\{T\|M_{l,:}\|_{\infty},1 \}
    \label{eq:exp_truncated_poisson}
\end{align*}

\noindent\textbf{Step 3: Showing \emph{(iii)} in \eqref{eq:variance_matrix_bernstein}}
Using our definition $Z_i = (Y_{l,i} - M_{l,i})A_{i,:} \in \RR^{1\times n}$, we have that
\begin{align*}
    Z_i Z_i^\top &= (Y_{l,i} - M_{l,i})^2 \|A_{i,:}\|^2,\\
    Z_i^\top Z_i &= (Y_{l,i} - M_{l,i})^2 A_{i,:}^\top A_{i,:}.
\end{align*}
Since $A$ and $Y_{l,:}$ are independent, we have:
\begin{align*}
    \|\sum_{i=1}^n \EE [Z_i Z_i^\top]\| = \sum_{i=1}^n \EE [Z_i Z_i^\top] =  \sum_{i=1}^n  \|A_{i,:}\|^2 \EE (Y_{l,i} - M_{l,i})^2 \leq \|A\|_F^2 \max_{i} \EE (Y_{l,i} - M_{l,i})^2
\end{align*}
and
\begin{align*}
    \|\sum_{i=1}^n \EE [Z_i^\top Z_i]\| =  \|\sum_{i=1}^n\EE(Y_{l,i} - M_{l,i})^2 A_{i,:}^\top A_{i,:}\|
    &\leq
     \sum_{i=1}^n \EE(Y_{l,i} - M_{l,i})^2 \| A_{i,:}^\top A_{i,:} \|\\
    &\leq \|A\|_F^2 \max_{i} \EE (Y_{l,i} - M_{l,i})^2.
\end{align*}
Now note that for $Y_{l,i}\sim T^{-1} \mathrm{Poisson}(T M_{l,i})$, $\Var(Y_{l,i}) = \EE (Y_{l,i} - M_{l,i})^2 = T^{-1} M_{l,i}$. Thus, by setting $v= T^{-1} \|A\|_F^2 \|M_{l,:}\|_{\infty} $, we get {\it (iii)}. 

Plugging in all obtained quantities into Equation \eqref{eq:matrix_bernstein_final_bound} finishes proof of the lemma.
\end{proof}

\newpage
\section{Singular subspace recovery via the leave-one-out argument}
\label{sec:appendix_singular_subspace_LOO}

In this section, we present Lemma \ref{lem:ssr-rewards} and Lemma \ref{lemma:U_to_EU} providing sharp guarantees for the singular subspace recovery in two-to-infinity norm. Obtaining such guarantees is not trivial and requires the use of a rather technical analysis, namely the leave-one-out technique \cite{abbe2020entrywise, chen2021spectral}. However, such technique heavily relies on independence between entries of the observed random matrix. We use the Poisson approximation argument to address this, which in turn requires to reproduce the leave-one-out analysis under a different random matrix observation models (see \eqref{eq:reward-poisson-matrix-plus-noise-model} and \eqref{eq:M_hat_Poisson_model}).

We wish to highlight that Farias et al. \cite{farias2021near}, like us, have also used the leave-one-out argument to obtain entry-wise guarantees for matrix estimation with sub-exponential noise. In our case, we use this argument as a sub-step of our analysis after performing the Poisson approximation. However, we believe that, our final results are richer, more precise and actually needed for our RL applications. Indeed, we are able to obtain guarantees in the norms $\Vert \cdot \Vert_{2\to \infty}$ and $\Vert \cdot \Vert_{1 \to \infty}$ (these are not provided in \cite{farias2021near}). Moreover, the entry-wise guarantees in  \cite{farias2021near} are only expressed in terms of the matrix dimensions $m$ and $n$. Our guarantees on the other hand exhibit dependencies on the dimensions $m, n$, the number of observation $T$ and the confidence level $\delta$. Having guarantees with an explicit dependence for all $T \ge 1$ and $\delta \in (0,1)$ is crucial in the design of our algorithm for low-rank bandits.

\subsection{Subspace recovery for reward matrices}

% The followoing corresponds to Lemma 4.16 in \cite{chen2021spectral}
% \begin{lem}[Lemma 4.16 in \cite{chen2021spectral}] Let $S$ and $ \widetilde{S}$ be symmetric matrices such that $S$ is of rank $r$. Coonsider the svd of $S$ and $\widetilde{S}$ 
%     \begin{align}
% \label{eq:U_UWU_lemma416_poisson_transitions}
%     \| Q - \widehat{Q} W_{\widehat{Q}} \|_{2\to\infty} 
%     \leq
%     \frac{1}{\sigma_r(S)} \bigg( \frac{4\| \widetilde{S} Q  \|_{2\to\infty} \| E\|}{\sigma_r(S)} + \| E Q \|_{2\to\infty} + 2\| \widetilde{S}( Q - \widehat{Q} W_{\widehat{Q}})\|_{2\to\infty} \bigg)
% \end{align}
% \end{lem}

\begin{lem}\label{lem:ssr-rewards}
Let $\delta \in (0,1)$. Define: 
    \begin{align*}
        \cB = \sqrt{\frac{nm}{T}} \left( \sqrt{(n+m) \log\left( \frac{e(n+m)T}{\delta}\right)  }  + \log^{3/2}\left( \frac{e(n+m)T}{\delta}\right) \right).
    \end{align*}
    For all $ T \ge  c (\mu^4 \kappa^2 r^2 + 1) (m+n) \log^{3}\left(e^2(m+n)T / \delta
    \right)$,  the event  
    \begin{align*}
        \max(  \Vert U - \widehat{U} (\widehat{U}^\top U)\Vert, \Vert V - \widehat{V} (\widehat{V}^\top V)\Vert) \le C \frac{\Vert M \Vert \Vert M\Vert_\infty}{\sigma_r(M)^2} \max(\Vert V\Vert_{2\to \infty} \Vert U \Vert_{2\to\infty}) \cB
    \end{align*}
    holds with probability at least $1-\delta$, for some universal constants $c, C > 0$.
\end{lem}

\begin{proof}[Proof of Lemma \ref{lem:ssr-rewards}] The proof follows similar steps as that of Theorem 4.2 in \cite{chen2021spectral}, which is based on the leave-one-out analysis. 

\textbf{Step 1: Dilation trick.}
In order to apply the leave-one-out analysis, we first use a dilation trick \cite{tropp2015introduction} to reduce the problem to that of symmetric matrices. Define: 
\begin{align*}
    S = \begin{bmatrix}
        0 & M \\
        M^\top & 0
    \end{bmatrix}
\end{align*}
and note that for matrix $M$ with SVD $M = U \Sigma V^\top$, we have: 
\begin{align*}
    S = \frac{1}{\sqrt{2}} \begin{bmatrix}
        U & U \\
        V & -V
    \end{bmatrix}\begin{bmatrix}
        \Sigma & 0 \\
        0 & -\Sigma
    \end{bmatrix}
    \frac{1}{\sqrt{2}}
    \begin{bmatrix}
        U & U \\
        V & -V
    \end{bmatrix}^\top := Q D Q^{\top}.
\end{align*}
We define, in a similar way, $\widetilde{S}$ using $\widetilde{M}$, and let $\widehat{Q} \in \RR^{(n+m)\times 2r}$ be the matrix of eigenvectors of the best $2r$-rank approximation of $\widetilde{S}$. Note that: 
\begin{align}
    \Vert Q - \widehat{Q} (\widehat{Q}^\top Q)\Vert_{2 \to \infty} = \max \left\lbrace \Vert U - \widehat{U} (\widehat{U}^\top U)\Vert_{2 \to \infty},   \Vert V - \widehat{V} (\widehat{V}^\top V)\Vert_{2 \to \infty} \right\rbrace. 
\end{align}
    To keep the notation simple, we will define $W_{\widehat{Q}} = \widehat{Q}^\top Q$. Further note that 
    \begin{align}
    \label{eq:properties_symm_matrix}
        \Vert \widetilde{S} - S \Vert = \Vert \widetilde{M} - M \Vert, \qquad \sigma_{1}(S) = \sigma_1(M), \qquad  \text{and} \qquad \sigma_{2r}(S) = \sigma_r(M). 
    \end{align}
    
    We start the analysis under the model (\ref{eq:reward-poisson-matrix-plus-noise-model}) and assume that $\widetilde{M}$ has independent entries with compound Poisson distributions. We will eventually invoke the Poisson approximation argument via Lemma \ref{lem:poisson-approx-rewards} to deduce the final result. 

    {\bf{Step 2: Error decomposition.}} We apply the decomposition in Lemma \ref{lem:decomp-2-infinity} to obtain: 
     \begin{align*}
%\label{eq:U_UWU_lemma416_poisson_rewards}
    \| Q - \widehat{Q} W_{\widehat{Q}} \|_{2\to\infty} 
    \leq
    \frac{1}{\sigma_{2r}(S)} \bigg( \frac{4\| \widetilde{S} Q  \|_{2\to\infty} \| E\|}{\sigma_{2r}(S)} + \| E Q \|_{2\to\infty} + 2\| \widetilde{S}( Q - \widehat{Q} W_{\widehat{Q}})\|_{2\to\infty} \bigg), 
\end{align*}
    where we set $E = \widetilde{S} - S$. We observe that when $\Vert E \Vert \le \sigma_{2r}(S)/2$, then 
         \begin{align}
    \| Q - \widehat{Q} W_{\widehat{Q}} \|_{2\to\infty} 
    \leq
    \frac{1}{\sigma_{2r}(S)} \bigg( \frac{4\| S Q  \|_{2\to\infty} \| E\|}{\sigma_{2r}(S)} + 3\| E Q \|_{2\to\infty} + 2\| \widetilde{S}( Q - \widehat{Q} W_{\widehat{Q}})\|_{2\to\infty} \bigg).
    \label{eq:U_UWU_lemma416_poisson_rewards}
\end{align}
Furthermore, we also have 
\begin{align*}
    \Vert \widetilde{S} (Q - \widehat{Q} W_{ \widehat{Q} })\Vert_{2\to \infty} & \le \Vert E (Q - \widehat{Q} W_{ \widehat{Q} })\Vert_{2\to \infty} + \Vert S (Q - \widehat{Q} W_{ \widehat{Q} })\Vert_{2\to \infty} \\ 
    & \le \Vert E (Q - \widehat{Q} W_{ \widehat{Q} })\Vert_{2\to \infty} + \Vert S Q \Vert_{2\to \infty} \Vert \sin(Q, \widehat{Q}) \Vert^2 \\
    & \le \Vert E (Q - \widehat{Q} W_{ \widehat{Q} })\Vert_{2\to \infty} + \frac{\Vert S Q \Vert_{2\to \infty} \Vert E \Vert^2 }{ \sigma_{2r}(S)^2} \\
    & \le \Vert E (Q - \widehat{Q} W_{ \widehat{Q} })\Vert_{2\to \infty} + \frac{\Vert S Q \Vert_{2\to \infty} \Vert E \Vert }{2  \sigma_{2r}(S)},
\end{align*}
where the first inequality follows from the triangular inequality, the second inequality follows by the relation between the two-to-infinity norm and the sin theorem (see e.g., \cite{cape2019two}). The third inequality follows from Davis-Kahan's theorem. The fourth inequality follows under the condition $\Vert E \Vert \le \sigma_{2r}(S)/2$. We finally obtain 
\begin{align} \label{eq:decomp-error}
    \| Q - \widehat{Q} W_{\widehat{Q}} \|_{2\to\infty} 
    \leq \frac{1}{\sigma_{2r}(S)} \bigg( \frac{5\| S Q  \|_{2\to\infty} \| E\|}{\sigma_{2r}(S)} + 3\| E Q \|_{2\to\infty} + 2\| E ( Q - \widehat{Q} W_{\widehat{Q}})\|_{2\to\infty} \bigg).
\end{align}
Note that in the above inequality, we can control  $\Vert  E \Vert $ using Proposition \ref{prop:reward-concentration-1} and $\Vert E Q\Vert_{2\to \infty}$ using Proposition \ref{prop:reward-concentration-2}. However, the term $\| E ( Q - \widehat{Q} W_{\widehat{Q}})\|_{2\to\infty}$ is not easy to control because $E$ and $( Q - \widehat{Q} W_{\widehat{Q}})$ are dependent on each other in a non-trivial way. To control this term, we use the leave-one-out analysis.

\textbf{Step 3: Leave-one-out analysis.}  We  define a matrix $\widetilde{S}^{(\ell)} \in\RR^{(n+m)\times (n+m)}$ as follows:
\begin{align*}
    \widetilde{S}^{(\ell)}_{i,j} = \begin{cases}
        \widetilde{S}_{i,j},\quad &\text{if } i\neq \ell \text{ or } j \neq \ell \\
        S_{i,j}, \quad &\text{otherwise}
    \end{cases}
\end{align*}
Then define $\widehat{Q}^{(\ell)} \in \RR^{n\times 2r}$ as a matrix of eigenvectors corresponding to the $2r$ greatest (in absolute value) eigenvalues of matrix $\widetilde{S}^{(\ell)}$. Define $W_{\widetilde{U}^{(\ell)}}$ accordingly. We have 
 \begin{align*}
    \Vert E (Q - \widehat{Q}W_{\widehat{Q}})\Vert_{2 \to \infty} \le \max_{\ell \in [n+m]} \Vert  E_{\ell,:} ( Q - \widehat{Q}^{(\ell)} W_{\widehat{Q}^{(\ell)}} ) \Vert_2 + \Vert E \Vert_2 \Vert \widehat{Q} W_{\widehat{Q}}  - \widehat{Q}^{(\ell)} W_{\widehat{Q}^{(\ell)}}\Vert_F.    
\end{align*}
We have by Proposition \ref{prop:reward-concentration-1} that 
\begin{align*}
    \PP\left( \Vert E \Vert \lesssim \Vert M \Vert_\infty \cG\right) \ge 1 - \delta
\end{align*}
provided that 
\begin{align*}
    \textbf{(C1)}\qquad \qquad  T \ge c_1 \frac{mn}{m+n} \log^3\left( \frac{e(m+n)}{\delta}  \right) 
\end{align*}
and where we define 
\begin{align*}
    \cG = \sqrt{\frac{mn}{T}} \left(  \sqrt{(m+n) \log\left( \frac{e(m+n)}{\delta}\right)} + \log^{3/2}\left( \frac{e(m+n)}{\delta}\right)\right).
\end{align*}
Let us now introduce the event $\cE_1$ as follows
\begin{align*}
    \cE_1 = \left\lbrace \Vert E \Vert \le \Vert M \Vert_\infty \cG \right\rbrace. 
\end{align*}
Note that if the following condition holds 
\begin{align*}
        \textbf{(C2)} \qquad \qquad T \ge c_2 (\mu \kappa r)^2 \left((m+n)\log\left( \frac{e(m+n)}{\delta}\right) + \log^3\left( \frac{e(m+n)}{\delta}\right) \right)
\end{align*}
for $c_2$ large enough then $16 \Vert E \Vert \le \sigma_r(M)$. Hence, under the event $\cE_1$, using Lemma \ref{lem:aux-derivations},  we have
\begin{align*}
    \Vert \widehat{Q} W_{\widehat{Q}}  - \widehat{Q}^{(\ell)} W_{\widehat{Q}^{(\ell)}}\Vert_F \le \frac{16 \Vert E_{\ell,:} \widehat{Q}^{(\ell)} W_{\widehat{Q}^{(\ell)}}  \Vert_2 + 16 \Vert E \Vert \Vert \widehat{Q} W_{\widehat{Q}}\Vert_{2 \to \infty} }{\sigma_{2r}(M)},
\end{align*}

which further gives by triangular inequality 
\begin{align*}
    \Vert \widehat{Q} W_{\widehat{Q}}  - \widehat{Q}^{(\ell)} W_{\widehat{Q}^{(\ell)}}\Vert_F & \le \frac{16\Vert E_{\ell, :} ( Q - \widehat{Q}^{(\ell)} W_{\widehat{Q}^{(\ell)}} ) \Vert_2 }{\sigma_{r}(M)}  \\ 
    & \qquad + \frac{16\left( \Vert E_{\ell,:} Q \Vert_2 + \Vert E \Vert \Vert Q - \widehat{Q} W_{\widehat{Q}}\Vert_{2 \to \infty}  + \Vert E \Vert \Vert Q \Vert_{2 \to \infty} \right)  }{\sigma_r(M)}.
\end{align*}
Now, by Proposition \ref{prop:reward-concentration-2}, 
\begin{align*}
    \PP\left(\Vert E_{\ell,:}( Q - Q^{(\ell)} W_{\widehat{Q}^{(\ell)}} ) \Vert_2 \lesssim \Vert M \Vert_\infty \Vert Q - Q^{(\ell)} W_{\widehat{Q}^{(\ell)}} \Vert_{2 \to \infty} \cG \right) \ge 1-\delta
\end{align*}
as long as the same condition \textbf{(C1)} holds with $c_1$ large enough. So let us introduce the event 
\begin{align*}
    \cE_2 = \left\lbrace  \Vert E_{\ell,:}( Q - Q^{(\ell)} W_{\widehat{Q}^{(\ell)}} ) \Vert_2 \lesssim \Vert M \Vert_\infty \Vert Q - Q^{(\ell)} W_{\widehat{Q}^{(\ell)}} \Vert_{2 \to \infty} \cG  \right\rbrace. 
\end{align*}
We further upper bound under the event $\cE_1 \cap \cE_2$, 
\begin{align*}
    \Vert E_{\ell,:}( Q - \widehat{Q}^{(\ell)} W_{\widehat{Q}^{(\ell)}} ) \Vert_2 \lesssim \Vert M \Vert_\infty \left(\Vert Q - \widehat{Q} W_{\widehat{Q}} \Vert_{2 \to \infty}   + \Vert  \widehat{Q} W_{\widehat{Q}} - \widehat{Q}^{(\ell)} W_{\widehat{Q}^{(\ell)}}\Vert_F \right) \cG.   
\end{align*}
Note that, under the condition \textbf{(C2)} with $c_2$ large enough, we can also obtain 
\begin{align*}
    \frac{16\Vert M \Vert_\infty}{\sigma_r(M)} \cG  \le \frac{1}{2}
\end{align*}
which entails that 
\begin{align*}
    \Vert \widehat{Q} W_{\widehat{Q}}  - \widehat{Q}^{(\ell)} W_{\widehat{Q}^{(\ell)}}\Vert_F & \le \frac{32 \Vert M \Vert_\infty \Vert Q - \widehat{Q} W_{\widehat{Q}}\Vert_{2 \to \infty} }{\sigma_r(M)}  \cG \\
    &  + \frac{32( \Vert E_{\ell,:} Q \Vert + \Vert E \Vert \Vert Q - \widehat{Q} W_{\widehat{Q}}\Vert_{2 \to \infty}  + \Vert E \Vert \Vert Q \Vert_{2 \to \infty} )  }{\sigma_r(M)}.
\end{align*}
To simplify the notation, let us define the three errors as  
\begin{align*}
    x & = \Vert Q - \widehat{Q} W_{\widehat{Q}} \Vert_{2 \to \infty}, \\
    y & = \Vert E Q \Vert_{2 \to \infty} \ge \Vert E_{\ell,:} Q \Vert_2,\\
    z & =  \Vert E \Vert \Vert Q \Vert_{2 \to \infty}.
\end{align*}
We have 
\begin{align*}
    \Vert \widehat{Q} W_{\widehat{Q}}  - \widehat{Q}^{(\ell)} W_{\widehat{Q}^{(\ell)}}\Vert_F \lesssim \left(  \frac{\Vert M \Vert_\infty \cG}{\sigma_r(M)}  + \frac{\Vert E \Vert}{\sigma_r(M)} \right) x + \frac{1}{\sigma_r(M)} (y+z).
\end{align*}
By plugging the above in the previous inequality, we get 
\begin{align*}
    \Vert E_{\ell,:}( Q - \widehat{Q}^{(\ell)} W_{\widehat{Q}^{(\ell)}} ) \Vert_2  \lesssim \Vert M \Vert_\infty \cG  \left(  \left(1 + \frac{\Vert M\Vert_\infty \cG}{\sigma_r(M)} + \frac{\Vert E \Vert}{\sigma_r(M)} \right)x + \frac{1}{\sigma_r(M)} (y+z)    \right) 
\end{align*}
which entails finally 
\begin{align}\label{eq:loo-final}
    \Vert E(Q - \widehat{Q} W_{\widehat{Q}})  \Vert_{2 \to \infty} & \lesssim \left( \frac{\Vert E \Vert}{\sigma_r(M)} + \frac{\cG \Vert M \Vert_\infty }{\sigma_1(M)}\right) (y + z)  \nonumber \\
    & \qquad + (\Vert E \Vert + \cG \Vert M \Vert_\infty) \left( 1 + \frac{\cG \Vert M \Vert_\infty}{\sigma_r(M)}+ \frac{\Vert \cE\Vert}{\sigma_r(M)}\right) x.
\end{align}

\textbf{Step 4: Putting everything together.} Combining the inequalities  \eqref{eq:decomp-error} and \eqref{eq:loo-final} gives 
\begin{align*}
    x & \le C_1 \left(  \frac{\Vert E \Vert}{\sigma_r(M)} + \frac{\Vert M \Vert_\infty  \cG}{\sigma_r(M)} \right) \left( 1 + \frac{\Vert M \Vert_\infty \cG}{\sigma_r(M)} + \frac{\Vert E \Vert}{\sigma_r(M)} \right)x  \\
    & \quad + \frac{C_2}{\sigma_r(M)} \left( 1 + \frac{\Vert E \Vert}{\sigma_r(M)} + \frac{\cG \Vert M \Vert_\infty}{\sigma_r(M)}  \right) y \\
    & \qquad +\frac{C_3}{\sigma_r(M)} \left(    \frac{\Vert M\Vert}{\sigma_r(M)} + \frac{\Vert E \Vert}{\sigma_r(M)} + \frac{\cG \Vert M \Vert_\infty}{\sigma_r(M)} \right) z.
\end{align*}
Under the events $\cE_1$ and $\cE_2$ and provided that the conditions \textbf{(C1)} and \textbf{(C2)} hold, for $c_1$ and $c_2$ are large enough, we have 
\begin{align*}
    C_1 \left(  \frac{\Vert E \Vert}{\sigma_r(M)} + \frac{\Vert M \Vert_\infty  \cG}{\sigma_r(M)} \right) \left( 1 + \frac{\Vert M \Vert_\infty \cG}{\sigma_r(M)} + \frac{\Vert E \Vert}{\sigma_r(M)} \right) & \le \frac{1}{2},  \\
     \left( 1 + \frac{\Vert E \Vert}{\sigma_r(M)} + \frac{\cG \Vert M \Vert}{\sigma_r(M)}  \right) & \le 3, \\
     \left(    \frac{\Vert M\Vert}{\sigma_r(M)} + \frac{\Vert E \Vert}{\sigma_r(M)} + \frac{\cG \Vert M \Vert_\infty}{\sigma_r(M)} \right) & \le \left(\frac{\Vert M \Vert}{\sigma_r(M)} + 2\right). 
\end{align*}
Thus, we obtain
\begin{align*}
    x & \le \frac{1}{\sigma_r(M)}\left( y + \frac{\Vert M \Vert}{\sigma_r(M)} z\right).
\end{align*}
We note that, under a similar conditions as before , we also have by Proposition \ref{prop:reward-concentration-1} and Proposition \ref{prop:reward-concentration-2} that 
\begin{align*}
    y & \lesssim \Vert M \Vert_\infty \Vert Q \Vert_{2\to \infty} \cG \\
    z & \lesssim  \Vert M \Vert_\infty \Vert Q \Vert_{2 \to \infty}  \cG  
\end{align*}
with probability at least $1-\delta$. Thus, we conclude after further simplifications that for some $C>0$ large enough, we have 
\begin{align*}
    \PP\left( \Vert Q - \widehat{Q} W_{\widehat{Q}} \Vert_{2 \to \infty} \le C \frac{ \Vert M \Vert \Vert M \Vert_\infty}{\sigma_r(M)^2}  \Vert Q \Vert_{2 \to \infty} \cG \right) \ge 1 - \delta
\end{align*}
provided 
\begin{align*}
    T \ge  c (\mu^4 \kappa^2 r^2 + 1) (m+n) \log^{3}\left( \frac{e(m+n)}{\delta}
    \right),
\end{align*}
with 
\begin{align*}
    \cG(n,m,T,\delta) = \sqrt{\frac{nm}{T}} \left( \sqrt{(n+m) \log\left( \frac{e(n+m)}{\delta}\right)  }  + \log^{3/2}\left( \frac{e(n+m)}{\delta}\right) \right).
\end{align*}

\textbf{Step 5: Poisson approximation.} To conclude, we now invoke Lemma \ref{lem:poisson-approx-rewards} which entails that under the true model \eqref{eq:bandit-matrix-noise-model}, we have 
\begin{align*}
    \PP\left( \Vert Q - \widehat{Q} W_{\widehat{Q}} \Vert_{2 \to \infty} > C \frac{ \Vert M \Vert \Vert M \Vert_\infty}{\sigma_r(M)^2}  \Vert Q \Vert_{2 \to \infty} \cG(n,m,T,\delta) \right) \le  e\sqrt{T} \delta
\end{align*}
provided $ T \ge  c (\mu^4 \kappa^2 r^2 + 1) (m+n) \log^{3}\left(e(m+n) / \delta
    \right)$. By re-parametrizing  with $\delta' = e \sqrt{T} \delta$, we obtain 
    \begin{align*}
        \PP\left( \Vert Q - \widehat{Q} W_{\widehat{Q}} \Vert_{2 \to \infty} > C \frac{ \Vert M \Vert \Vert M \Vert_\infty}{\sigma_r(M)^2}  \Vert Q \Vert_{2 \to \infty} \cG(n,m,T,\delta'/e\sqrt{T}) \right) \le  \delta',
    \end{align*}
    again provided that $ T \ge  c (\mu^4 \kappa^2 r^2 + 1) (m+n) \log^{3}\left(e^2(m+n)\sqrt{T} / \delta'
    \right)$. Recalling that $\Vert Q \Vert_{2\to \infty} = \max(\Vert V\Vert_{2\to \infty}, \Vert U \Vert_{2\to\infty})$, we immediately obtain the final result.
\end{proof}

\begin{lem}\label{lem:aux-derivations}
Under the notation used in the proof of Lemma \ref{lem:ssr-rewards}, provided the condition $\Vert E \Vert \le \sigma_{2r}(S)/16$, the following inequality holds:
    \begin{align*}
         \Vert \widehat{Q} W_{\widehat{Q}} - \widehat{Q}^{(\ell)}  W_{\widehat{Q}^{(\ell)}}\Vert_F \le \frac{16 \Vert E_{\ell,:} \widehat{Q}^{(\ell)} W_{  \widehat{Q}^{(\ell)}}\Vert_{2} + 16 \Vert E  \Vert \Vert \widehat{Q} W_{\widehat{Q}} \Vert_{2 \to \infty}  }{\sigma_{2r}(S)}
    \end{align*}
\end{lem}

\begin{proof}[Proof of Lemma \ref{lem:aux-derivations}]
We have 
\begin{align*}
    \Vert \widehat{Q} W_{\widehat{Q}} - \widehat{Q}^{(\ell)}  W_{\widehat{Q}^{(\ell)}}\Vert_F &  \le \Vert \widehat{Q} \widehat{Q}^\top  - \widehat{Q}^{(\ell)} (\widehat{Q}^{(\ell)})^\top  \Vert_F  \Vert Q \Vert \le \frac{2  \Vert (\widetilde{S} - \widetilde{S}^{(\ell)} ) \widehat{Q}^{(\ell)}  \Vert_F  }{\vert \sigma_{2r}(\widetilde{S}^{(\ell)}) - \sigma_{2r+1}(\widetilde{S}^{(\ell)})  \vert }
\end{align*}
where the first inequality follows the elementary fact that $\Vert A B \Vert_F \le \Vert A \Vert_F \Vert B \Vert $, and the second inequality follows by Davis-Kahan. Now, by Weyl's inequality, we have for all $k \in [n+m]$, $
    \vert \sigma_{k}(\widetilde{S}^{(\ell)}) - \sigma_{k} (S) \vert \le \Vert E^{(\ell)} \Vert \le \Vert E \Vert 
$, where the error matrix $E^{(\ell)} = \widetilde{S}^{\ell} - S$, and more precisely is defined as follows:
\begin{align*}
    E^{(\ell)}_{i,j} = \begin{cases}
        E_{i,j} & \text{if } i\neq \ell \text{ or }j\neq \ell, \\
        0 & \text{otherwise.}
    \end{cases}
\end{align*}
    The crude inequality $\Vert E^{(\ell)} \Vert \le \Vert E \Vert$ follows from the fact that $\Vert E^{(\ell)} \Vert$ is equal to the operator norm of a submatrix of $E$ which will always be smaller than $\Vert E \Vert$. Therefore, under the condition that $\Vert E \Vert \le \sigma_{2r}(S)/4$, we have $\vert \sigma_{2r} (\widetilde{S}^{(\ell)}) - \sigma_{2r+1}(\widetilde{S}^{(\ell)}) \vert \ge \sigma_{2r}(S)/2$. In summary, we obtain that 
    \begin{align*}
        \Vert \widehat{Q} W_{\widehat{Q}} - \widehat{Q}^{(\ell)}  W_{\widehat{Q}^{(\ell)}}\Vert_F & \le \frac{4 \Vert (\widetilde{S} - \widetilde{S}^{(\ell)}) \widehat{Q}^{(\ell)}\Vert_F }{\sigma_{2r}(S)}. 
    \end{align*}
    Now, we further have by triangular inequality and by definition of $\widetilde{S}^{(\ell)}$:
    \begin{align*}
         \Vert (\widetilde{S} - \widetilde{S}^{(\ell)}) \widehat{Q}^{(\ell)}\Vert_F & = \Vert (e_\ell E_{\ell, :} + (E_{:, \ell} - E_{\ell, \ell} e_\ell)e_{\ell}^\top ) \widehat{Q}^{(\ell)} \Vert_F \\ 
         & \le \Vert E_{\ell,:} \widehat{Q}^{(\ell)} \Vert_2 + \Vert E_{:,\ell} - E_{\ell, \ell }e_\ell\Vert_2 \Vert \widehat{Q}^{(\ell)}\Vert_{2 \to \infty}  \\
         & \le \Vert E_{\ell,:} \widehat{Q}^{(\ell)} \Vert_2 + \Vert E \Vert_2 \Vert \widehat{Q}^{(\ell)} \Vert_{2 \to \infty} \\
         & \le \Vert E_{\ell,:} \widehat{Q}^{(\ell)} \Vert_2 + 2 \Vert E \Vert_2 \Vert \widehat{Q}^{(\ell)} (\widehat{Q}^{(\ell)})^\top Q \Vert_{2 \to \infty}
    \end{align*}
    where the last inequality follows under the condition that $\Vert E \Vert  \le 2 \sigma_{2r}(S)$. Indeed, we have under such condition that $\Vert \widehat{Q}^{(\ell)} \Vert_{2 \to \infty}  = \Vert \widehat{Q}^{(\ell)} (\widehat{Q}^{(\ell)})^\top Q \Vert_{2 \to \infty}  + \Vert \widehat{Q}^{(\ell)} (\sgn((\widehat{Q}^{(\ell)})^\top Q^\top) - (\widehat{Q}^{(\ell)})^\top  Q) \Vert_{2 \to \infty}$, and by Davis-Kahan's inequality $\Vert \sgn((\widehat{Q}^{(\ell)})^\top Q^\top) - (\widehat{Q}^{(\ell)})^\top  Q \Vert \le \frac{2\Vert E^{(\ell)}\Vert^2 }{(\sigma_{2r}(S))^2} \le\frac{2\Vert E\Vert^2 }{(\sigma_{2r}(S))^2} \le \frac{1}{2}$. 
    Similarly, we also have $\Vert E_{\ell,:} \widehat{Q}^{(\ell)} \Vert_2 \le 2 \Vert E_{\ell,:} \widehat{Q}^{(\ell)} (\widehat{Q}^{(\ell)})^\top Q \Vert_2$.
    
    Hence, we obtain: 
    \begin{align*}
        \Vert (\widetilde{S} - \widetilde{S}^{(\ell)}) \widehat{Q}^{(\ell)}\Vert_F \le 2 \Vert E_{\ell,:} \widehat{Q}^{(\ell)} W_{\widehat{Q}^{(\ell)}}\Vert_{2} + 2 \Vert E  \Vert (\Vert \widehat{Q} W_{\widehat{Q}} \Vert_{2 \to \infty} +  \Vert \widehat{Q} W_{\widehat{Q}} - \widehat{Q}^{(\ell)}W_{\widehat{Q}^{(\ell)}}   \Vert_{2 \to \infty} )  
    \end{align*}
    Which entails under the condition that $\Vert E \Vert \le  \sigma_{2r}(S)/16 $  that 
    \begin{align*}
         \Vert \widehat{Q} W_{\widehat{Q}} - \widehat{Q}^{(\ell)}  W_{\widehat{Q}^{(\ell)}}\Vert_F \le \frac{8 \Vert E_{\ell,:} \widehat{Q}^{(\ell)} W_{\widehat{Q}^{(\ell)}} \Vert_{2 } + 8 \Vert E  \Vert \Vert \widehat{Q} W_{\widehat{Q}} \Vert_{2 \to \infty}  }{\sigma_{2r}(S)} +\frac{\Vert \widehat{Q} W_{\widehat{Q}} - \widehat{Q}^{(\ell)}W_{\widehat{Q}^{(\ell)}}   \Vert_{2 \to \infty}}{2}  
    \end{align*}
    After rearranging, we obtain 
    \begin{align*}
         \Vert \widehat{Q} W_{\widehat{Q}} - \widehat{Q}^{(\ell)}  W_{\widehat{Q}^{(\ell)}}\Vert_F \le \frac{16 \Vert E_{\ell,:} \widehat{Q}^{(\ell)} W_{\widehat{Q}^{(\ell)}} \Vert_{2 } + 16 \Vert E  \Vert \Vert \widehat{Q} W_{\widehat{Q}} \Vert_{2 \to \infty}  }{\sigma_{2r}(S)}
    \end{align*}
\end{proof}

\subsection{Subspace recovery for transition matrices}

\begin{lem}
Let $Y\in\RR^{n\times n}$ be a matrix of independent Poisson entries with $Y_{i,j} \sim \frac{1}{T}\mathrm{Poisson}(T M_{i,j})$, and let $\hU,\hV$ be the matrices of left and right singular vectors of best $r$-rank approximation of $Y$. Let $g_{\delta}$ be the function defined in \eqref{eq:g_def}. Conditioned on the events where $\|Y-M\| \leq c_1 \sigma_r(M), g_{\delta}(TM)\log(ne/\delta)\leq c_2 T\sigma_r(M), \allowbreak \sqrt{\|M\|_{\infty}\log (ne/\delta)}\leq c_3 \sqrt{T}\sigma_r(M)$ for some sufficiently small universal constants $c_1,c_2,c_3>0$, we have, with probability at least $1-\delta$,
    \label{lemma:U_to_EU}
    \begin{align*}
        &\max\Big\{  \|U   - \hU (\widehat{U}^\top U) \|_{2\to\infty}, \|V - \hV (\widehat{V}^\top V) \|_{2\to\infty} \Big\}\\
        &\lesssim
        \frac{1}{\sigma_r(M)}\bigg[  \mu\sqrt{\frac{r}{n}}  \bigg(\frac{\sigma_1(M)}{\sigma_r(M)}\|Y-M\| + \frac{1}{T}g_{\delta}(TM) \log\left(\frac{ne}{\delta}\right)\bigg)  + \sqrt{\frac{r\|M\|_{\infty}}{T}\log\left( \frac{ne}{\delta}\right)}\bigg].
    \end{align*}
\end{lem}

\begin{proof}
The proof follows similar steps as the proof of Theorem 4.2 in \cite{chen2021spectral}. In order to apply the leave-one-out technique, we first repeat the symmetric dilation trick as in Step 1 of proof of Lemma \ref{lem:ssr-rewards}. We define
\begin{align}
S = \begin{bmatrix}
    0 & M\\
    M^\top & 0
\end{bmatrix}    
\label{eq:symmetrized_dilation_def}
\end{align}
and note that for matrix $M$ with SVD $M=U \Sigma V^\top$, we have:
\begin{align*}
    S = \frac{1}{\sqrt{2}} \begin{bmatrix}
        U & U \\ V & -V
    \end{bmatrix}
    \begin{bmatrix}
        \Sigma & 0\\
        0 & -\Sigma
    \end{bmatrix}
    \frac{1}{\sqrt{2}} \begin{bmatrix}
        U & U \\ V & -V
    \end{bmatrix}^\top := Q D Q^\top.
\end{align*}
We define $\tS$ as the symmetrized version of matrix $Y$, and let $\widehat{Q} \in \RR^{n\times 2r}$ be the matrix of eigenvectors of the best $2r$-rank approximation of $\tS$. Note that:
\begin{align*}
    \|Q - \widehat{Q} (\widehat{Q}^\top Q) \|_{2\to\infty} = \max\left\{ \|U  - \hU (\hU^\top U)\|_{2\to\infty} , \|V - \hV (\hV^\top V) \|_{2\to\infty}\right\}.
\end{align*}
We will also repeatedly use the properties \eqref{eq:properties_symm_matrix}. To keep the notation simple, define $ \WhQ = \widehat{Q}^\top Q$. Thus, proving Lemma \ref{lemma:U_to_EU} is equivalent to showing:
\begin{align*}
    &\|Q - \widehat{Q} \WhQ  \|_{2\to\infty} \\ &\lesssim
   \frac{1}{\sigma_r(M)}\bigg[  \|Q\|_{2\to\infty} (\frac{\sigma_1(M)}{\sigma_r(M)}\|\tS-S\| + \frac{1}{T} g_{\delta}(TM) \log\left(\frac{ne}{\delta}\right) ) 
   + \sqrt{\frac{r\|M\|_{\infty}}{T}\log \left(\frac{ne}{\delta}\right)} \bigg]
\end{align*}
with high probability.
Define $E = \tS - S$. Now, as in Lemma \ref{lem:decomp-2-infinity}, we have:
\begin{align}
\label{eq:U_UWU_lemma416_poisson_transitions}
    \|Q - \widehat{Q} \WhQ \|_{2\to\infty} 
    \leq
    \frac{1}{\sigma_r(M)} \bigg( \frac{4\|\tS Q\|_{2\to\infty} \|E\|}{\sigma_r(M)} + \|E Q\|_{2\to\infty} + 2\|\tS(Q - \widehat{Q} \WhQ)\|_{2\to\infty} \bigg)
\end{align}
under the assumption that $\|E\| \leq c_1 \sigma_r(M)$. Indeed, it is straightforward to show the same bounds as in Lemma 4.14 in \cite{chen2021spectral} - note that the boundedness assumption is not used in these lemmas. We bound the three terms in Equation \eqref{eq:U_UWU_lemma416_poisson_transitions} as follows:

    1. To bound the first term, we use:
    \begin{align}
        \|\tS Q\|_{2\to\infty} \leq \|S Q\|_{2\to\infty} + \|E Q\|_{2\to\infty} \leq \|Q\|_{2\to\infty} \|S\| + \|E Q\|_{2\to\infty} 
        \label{eq:proof_U_bound_first}
    \end{align}
    where we first used the triangle inequality and then $\|S Q\|_{2\to\infty} = \|Q D\|_{2\to\infty} \leq \|Q\|_{2\to\infty}\|S\|$.
    
    2. For the second term, according to Lemma \ref{lemma:El_A_conc}, we obtain with probability at least $1-\delta$:
    \begin{align}
    \|E  Q\|_{2\to\infty} \lesssim
     \frac{1}{T}\left[\|Q\|_F \sqrt{ T\|M\|_{\infty} \log (ne/\delta)} + g_{\delta}(TM)\log (ne/\delta)\|Q\|_{2\to\infty}  \right].
     \label{eq:proof_U_bound_second}
     \end{align}
    Moreover, we will use $\|Q\|_F \leq \sqrt{2r}$ and $\|Q\|_{2\to\infty} \leq \mu\sqrt{\frac{r}{n}}$, which follow from the low-rank and incoherence assumptions.
    
    3. Finally, regarding the last term in Equation \eqref{eq:U_UWU_lemma416_poisson_transitions}, we split it using the triangle inequality as follows: 
\begin{align*}
    \|\tS(Q - \widehat{Q} \WhQ)\|_{2\to\infty} \leq \|S(Q - \widehat{Q} \WhQ )\|_{2\to\infty} + \|E(Q - \widehat{Q} \WhQ)\|_{2\to\infty},
\end{align*}
and from Step 3 of proof of Theorem 4.2 in \cite{chen2021spectral} we have:
\begin{align}
    \|S(Q - \widehat{Q} \WhQ)\|_{2\to\infty}
    \leq
    \|Q\|_{2\to\infty} \|S\| \|Q^\top(Q - \widehat{Q} \WhQ)\|
    % &=
    % \|Q\|_{2\to\infty} \|S\|  \|\sin\Theta(Q,\widehat{Q})\|^2 \nonumber \\
    &\lesssim
    \|Q\|_{2\to\infty} \|S\| 
    \frac{\|E\|^2}{\sigma_r^2(M)},
    \label{eq:proof_U_bound_third1}
\end{align}
where we used $ \|Q^\top(Q - \widehat{Q} \WhQ)\| =  \|\sin\Theta(Q,\widehat{Q})\|^2$. The remaining of the proof consists in bounding $\|E(Q - \widehat{Q} \WhQ)\|_{2\to\infty} = \max_{\ell=1,...,n} \|E_{\ell,:}  (Q -\widehat{Q} \WhQ )\|$. First note that the matrix $Q - \widehat{Q} \WhQ$ depends on $E$ and thus we cannot apply Lemma \ref{lemma:El_A_conc} immediately. Instead, we will use the leave-one-out method, and define a matrix $\Stl\in\RR^{n\times n}$ as follows:
\begin{align*}
    \widetilde{S}^{(\ell)}_{i,j} = \begin{cases}
        \widetilde{S}_{i,j},\quad &\text{if } i\neq \ell \text{ or } j \neq \ell \\
        S_{i,j}, \quad &\text{otherwise}
    \end{cases}
\end{align*}
Then define $\Qhl \in \RR^{n\times 2r}$ as a matrix of eigenvectors corresponding to $2r$ greatest (in absolute value) eigenvalues of matrix $\Stl$. Define $\WhQl$ accordingly. Then we have:
\begin{align}
    \|E(Q - \widehat{Q} \WhQ)\|_{2\to\infty}
    \leq 2
    \max_{1\leq \ell \leq n} \left\{ \|E_{\ell,:}  (Q - \Qhl \WhQl)\|, \|E_{\ell,:}  (\Qhl\WhQl - \widehat{Q}\WhQ)\| \right\}.
    \label{eq:split_El}
\end{align}

  (3a) Since $E_{\ell,:}$ is statistically independent of $Q - \Qhl\WhQl$, the first term from \eqref{eq:split_El} can be bounded according to Lemma \ref{lemma:El_A_conc} for any $1\leq \ell\leq n$ as follows:
%Moreover, for all $1\leq \ell\leq n$ we have 
    \begin{align*}
    \|E_{\ell,:}  (Q - \Qhl\WhQl )\| \lesssim
     \frac{1}{T}\Big[&\|Q - \Qhl\WhQl\|_F \sqrt{ T\|M\|_{\infty}  \log \left(\frac{ne}{\delta}\right)} \\ &+ g_{\delta}(TM)\log \left(\frac{ne}{\delta}\right)\|Q - \Qhl\WhQl\|_{2\to\infty} \Big]
     \end{align*}
     with probability at least $1-\delta$.
After applying the triangle inequality to the second term and using $\|\cdot\|_{2\to\infty} \leq \|\cdot\|_F$, we get:
\begin{align}
    \|Q-\Qhl\WhQl\|_{2\to\infty} \leq \|Q-\widehat{Q}\WhQ\|_{2\to\infty} + \|\widehat{Q}\WhQ-\Qhl\WhQl\|_F.
    \label{eq:proof_triangle_ulU_hulUlulWhUl}
\end{align}
Thus, combining the last two inequalities yields:
\begin{align}
    &\|E_{\ell,:}  (Q - \Qhl\WhQl )\| 
    \nonumber \\ &\lesssim
    \|Q - \widehat{Q}\WhQ \|_F \sqrt{T^{-1}\|M\|_{\infty} \log \left(\frac{ne }{\delta}\right)} + T^{-1} g_{\delta}(TM) \log \left(\frac{ne}{\delta}\right) \|Q - \widehat{Q}\WhQ \|_{2\to\infty} \nonumber  \\
    &+
    \|\widehat{Q}\WhQ - \Qhl\WhQl\|_F \left( \sqrt{ T^{-1} \|M\|_{\infty} \log \left(\frac{ne}{\delta}\right)} + T^{-1} g_{\delta}(TM)\log \left(\frac{ne}{\delta}\right)\right)
\label{eq:bound_poisson_error_main_term_final}
\end{align}
for all $1\leq \ell\leq n$.

(3b) The second term from \eqref{eq:split_El} can be bounded very roughly as follows:
\begin{align}
    \|E_{\ell,:}  (\Qhl\WhQl  - \widehat{Q}\WhQ ))\|
    &\leq
    \|E\| \| \Qhl\WhQl - \widehat{Q} \WhQ\|_F.
    \label{eq:E_UW_UW_2}
\end{align}
Similar to the Step 2.2 in the proof of Theorem 4.2 in \cite{chen2021spectral}, we have:
\begin{align*}
    \| \Qhl\WhQl - \widehat{Q}  \WhQ \|_F
    \leq
    \frac{16}{\sigma_r(M)} &( \|E_{\ell,:}(Q-\Qhl\WhQl)\| + \|E Q\|_{2\to\infty} \nonumber \\
    &+ \|E\| \|Q-\Qhl\WhQl\|_{2\to\infty} + \|E\|\|Q\|_{2\to\infty} ).
\end{align*}
Applying again the inequality \eqref{eq:proof_triangle_ulU_hulUlulWhUl} and moving the term $\| \Qhl\WhQl - \widehat{Q}  \WhQ \|_F$ to the left side of inequality, we get under assumption $\|E\|  \leq c_1 \sigma_r(M)$ that:
\begin{align}
    \| \Qhl\WhQl - \widehat{Q}  \WhQ \|_F
    \lesssim
    \frac{1}{\sigma_r(M)} &( \|E_{\ell,:}(Q-\Qhl\WhQl)\| + \|E Q\|_{2\to\infty} \nonumber \\
    &+ \|E\| \|Q-\widehat{Q}\WhQ\|_{2\to\infty} + \|E\|\|Q\|_{2\to\infty} ).
    \label{eq:frob_norm_diff_hUl_hU}
\end{align}

After substitution of the results from \eqref{eq:proof_U_bound_first}, \eqref{eq:proof_U_bound_second}, \eqref{eq:proof_U_bound_third1}, \eqref{eq:split_El}, \eqref{eq:bound_poisson_error_main_term_final}, \eqref{eq:E_UW_UW_2} and \eqref{eq:frob_norm_diff_hUl_hU} into Equation \eqref{eq:U_UWU_lemma416_poisson_transitions} and using assumptions stated in the lemma, we obtain the statement of the lemma.
\end{proof}

\subsection{Error decomposition in the two-to-infinity norm}

Below, we present a decomposition for the error of subspace recovery in the norm $\Vert \cdot \Vert_{2 \to \infty}$. We borrow this result from \cite{chen2021spectral} and provide its proof for completeness.

\begin{lem}[Lemma 4.16 in \cite{chen2021spectral}]\label{lem:decomp-2-infinity}
    Let $S, \widetilde{S} \in \RR^{n \times n}$ be symmetric matrices and assume that $S$ and $\widetilde{S}$ are of rank $r$. Let $Q, \widehat{Q} \in \cO_{n \times r}$ be the corresponding $r$ singular vectors of $S$ and $\widetilde{S}$, respectively. Denote $E = \widetilde{S} - S$. Under the condition $\Vert E \Vert \le \sigma_r(S)/2$, we have: 
    \begin{align*}
        \Vert Q - \widehat{Q} \widehat{Q}^\top Q \Vert_{2 \to \infty} \le \frac{4 \Vert \widetilde{S} Q \Vert_{2 \to \infty} \Vert E \Vert }{(\sigma_{r}(S))^2} + \frac{\Vert E Q\Vert_{2 \to \infty}}{\sigma_r(S)} + \frac{2 \Vert \widetilde{S} (Q - \widehat{Q} \widehat{Q}^\top  Q )\Vert_{2 \to \infty}}{ \sigma_r(S)}
    \end{align*}

%          \begin{align*}
% %\label{eq:U_UWU_lemma416_poisson_rewards}
%     \| Q - \widehat{Q} W_{\widehat{Q}} \|_{2\to\infty} 
%     \leq 
%     \frac{1}{\sigma_{2r}(S)} \bigg( \frac{4\| \widetilde{S} Q  \|_{2\to\infty} \| E\|}{\sigma_{2r}(S)} + \| E Q \|_{2\to\infty} + 2\| \widetilde{S}( Q - \widehat{Q} W_{\widehat{Q}})\|_{2\to\infty} \bigg), 
% \end{align*}
\end{lem}

\begin{proof}[Proof of Lemma \ref{lem:decomp-2-infinity}]
Since $S$ is a symmetric matrix of rank $r$, by SVD we write $S = Q \Sigma Q^\top$, where the matrix $\Sigma = \diag(\sigma_1(S), \dots, \sigma_r(S))$. For ease of notations, let us further denote $W = \widehat{Q} Q^\top$. We have 
\begin{align*}
    \Vert Q - \widehat{Q}\widehat{Q}^\top Q \Vert_{2 \to \infty} & =  \Vert S Q \Sigma^{-1} - \widehat{Q}W \Vert_{2 \to \infty}  \\
    & \le \Vert \widetilde{S} Q \Sigma^{-1} - \widehat{Q}W \Vert_{2 \to \infty} + \Vert E Q \Sigma^{-1}\Vert_{2 \to \infty} \\
    & \le \frac{\Vert \widetilde{S} Q  - \widehat{Q}W \Sigma \Vert_{2 \to \infty}}{\sigma_r(S)} + \frac{\Vert E Q \Vert_{2 \to \infty}}{\sigma_r(S)} 
\end{align*}
Now, we focus on the term $\Vert \widetilde{S} Q  - \widehat{Q}W \Sigma \Vert_{2 \to \infty}$. To that end, we first establish the identity  
\begin{align*}
    \widehat{Q}W \Sigma &  = \widehat{Q}\widehat{Q}^\top Q  \Sigma \\
    & = \widehat{Q}\widehat{Q}^\top S Q \\
    & = \widehat{Q}\widehat{Q}^\top \widehat{S} Q + 
 \widehat{Q}\widehat{Q}^\top E Q \\
 & = \widehat{Q} \widehat{\Sigma} \widehat{Q}^\top Q  + 
 \widehat{Q}\widehat{Q}^\top E Q \\
 & = \widetilde{S} \widehat{Q}\widehat{Q}^\top Q + 
 \widehat{Q}\widehat{Q}^\top E Q
\end{align*}
where we use the identities $S Q = Q \Sigma$,  $\widehat{Q}^\top \widetilde{S} = \widehat{\Sigma} \widehat{Q}^\top$, and $\widehat{Q} \widehat{\Sigma} \widehat{Q}^\top =  \widetilde{S} \widehat{Q} \widehat{Q}^\top$. Then, we observe that 
\begin{align*}
    \Vert \widetilde{S} Q  - \widehat{Q}W \Sigma \Vert_{2 \to \infty} & = \Vert \widetilde{S} Q  -  \widetilde{S} \widehat{Q}\widehat{Q}^\top Q + 
 \widehat{Q}\widehat{Q}^\top E Q \Vert_{2 \to \infty} \\
 & \le \Vert \widetilde{S} (Q  -  \widehat{Q}\widehat{Q}^\top Q ) \Vert_{2 \to \infty}    +  \Vert 
 \widehat{Q}\widehat{Q}^\top E Q \Vert_{2 \to \infty}
\end{align*}

Next, we note that when $\Vert E \Vert \le \sigma_r(S)/2$, we have 
\begin{align*}
    \Vert \widehat{Q} \widehat{Q}^\top E Q \Vert_{2 \to \infty} & = \Vert \widetilde{S} \widehat{Q} \widehat{\Sigma}^{-1} \widehat{Q}^\top E Q \Vert_{2 \to \infty}  \\
    & \le \Vert \widetilde{S} \widehat{Q} \Vert_{2 \to \infty} \Vert \widehat{\Sigma}^{-1} \Vert \Vert \widehat{Q}^\top \Vert  \Vert E \Vert \Vert Q \Vert  \\
    & \le \frac{\Vert \widetilde{S} \widehat{Q} \sgn(\widehat{Q}^\top  Q)\Vert_{2 \to \infty}  \Vert E \Vert  }{\sigma_r(\widetilde{S})}.   
\end{align*}
At this point, we try to bound $\Vert \widetilde{S} \widehat{Q} \Vert_{2 \to \infty}$ and $\sigma_r(\widetilde{S})$, under the condition $\Vert E \Vert \le \sigma_r(S)/2$. First, we can easily see by Weyl's inequality we have 
$ \vert \sigma_r(\widetilde{S}) - \sigma_r(S)\vert \le \Vert E \Vert $, which entails under the assumed condition that $\sigma_r(\widetilde{S}) \ge \sigma_r(S)/2$. Next, we observe: 
\begin{align*}
    \Vert \widetilde{S} \widehat{Q} \Vert_{2 \to \infty} & = \Vert \widetilde{S} \widehat{Q}  \sgn(\widehat{Q}^\top  Q) \Vert_{2 \to \infty} \\ 
    & \le \Vert \widetilde{S} \widehat{Q} \widehat{Q}^\top Q \Vert_{2 \to \infty} +  \Vert \widetilde{S} \widehat{Q} \Vert_{2 \to \infty}  
  \Vert \sgn(\widehat{Q}^\top  Q) - \widehat{Q}^\top Q \Vert \\
  & \le \Vert \widetilde{S} \widehat{Q} \widehat{Q}^\top Q \Vert_{2 \to \infty} +    \frac{2 \Vert \widetilde{S} \widehat{Q} \Vert_{2 \to \infty} \Vert E \Vert^2 }{\sigma_r(S)^2} \\
  & \le \Vert \widetilde{S} \widehat{Q} \widehat{Q}^\top Q \Vert_{2 \to \infty} +    \frac{\Vert \widetilde{S} \widehat{Q} \Vert_{2 \to \infty} }{2} 
\end{align*}
where we used the Davis-Kahan's inequality and properties of the $\sgn(\cdot)$, to upper bound $ \Vert \sgn(\widehat{Q}^\top  Q) - \widehat{Q}^\top Q \Vert \le \Vert \sin (\widehat{Q}, Q)\Vert^2 \le 2 \Vert E \Vert^2 / (\sigma_r(S))^2$. Thus, leading to:
\begin{align*}
    \Vert \widetilde{S} \widehat{Q} \Vert_{2 \to \infty} \le 2 \Vert \widetilde{S} \widehat{Q} \widehat{Q}^\top Q \Vert_{2 \to \infty}
\end{align*}
 Moving forward we obtain:
\begin{align*}
     \Vert \widehat{Q} \widehat{Q}^\top E Q \Vert_{2 \to \infty} 
 & \le  \frac{4 \Vert \widetilde{S} \widehat{Q} \widehat{Q}^\top  Q\Vert_{2 \to \infty}  \Vert E \Vert  }{\sigma_r(S)} \\
    & \le \frac{4 \left(\Vert \widetilde{S} (\widehat{Q} \widehat{Q}^\top  Q  - Q) \Vert_{2 \to \infty} + \Vert \widetilde{S} Q \Vert_{2 \to \infty} \right) \Vert E \Vert  }{\sigma_r(S)} \\
    & \le 2 \Vert \widetilde{S} (\widehat{Q} \widehat{Q}^\top  Q  - Q) \Vert_{2 \to \infty} + \frac{4 \left( \Vert \widetilde{S} Q \Vert_{2 \to \infty} \right) \Vert E \Vert  }{\sigma_r(S)}
\end{align*}
Now, putting everything together we conclude that:
\begin{align*}
    \Vert Q - \widehat{Q}\widehat{Q}^\top Q \Vert_{2 \to \infty} \le   \frac{2 \Vert \widetilde{S} (\widehat{Q} \widehat{Q}^\top  Q  - Q) \Vert_{2 \to \infty}}{\sigma_r(S)} + \frac{4 \Vert \widetilde{S} Q \Vert_{2 \to \infty}  \Vert E \Vert  }{(\sigma_r(S))^2} + \frac{\Vert EQ \Vert_{2 \to \infty} }{\sigma_r(S)}
\end{align*}    
\end{proof}

% \begin{lem}
%     We have the following
%     \begin{align*}
%         \Vert \widehat{Q}^\top Q - \sgn(\widehat{Q}^\top Q ) \Vert \le  \frac{2 \Vert E \Vert^2 }{(\sigma_r(S))^2}
%     \end{align*}
% \end{lem}

\newpage

\section{Row-wise and entry-wise matrix estimation errors}
% \section{Applications of singular subspace bound to different norms }
\label{sec:different_norms_from_sing_subspace}

In this appendix, we provide a series of results about quantifying the matrix estimation error using different norms. It is important to note that all these results require a control of the error in the two-to-infinity norm, which in turn requires a subspace recovery guarantee in the two-to-infinity norm. Lemmas \ref{lem:P-one-to-infty} and \ref{lemma:hP_infty_bound} are specific to our analysis for the estimation of the transition matrices. Lemmas \ref{lemma:P_to_U} and \ref{lem:M-infinity} are common to the analysis of both the estimation of reward matrices and transition matrices. The results presented in this appendix are used in the proofs of the main results, presented in Appendix \ref{app:main}.

\subsection{Bounding \texorpdfstring{$\| M-\hM \|_{2\to\infty}$}{M2inf}} 
\begin{lem}
Let $M,\hM$ be as in \textsection \ref{sec:sampling}. Assume that there exists a sufficiently small universal constant $c_1>0$ such that $\|M-\tM\| \leq c_1 \sigma_r(M)$. Then, there exists a universal constant $c_2>0$ such that
\label{lemma:P_to_U}
    \begin{align*}
    \|\hM - M\|_{2\to\infty} & \leq
    c_2\sigma_1(M) \left[
    \|U-\hU (\hU^\top U)\|_{2\to\infty} + \|U\|_{2\to\infty}\frac{\|\tM-M\|}{\sigma_r(M)} \right].
\end{align*}
\end{lem}
\begin{proof}
We start by using definition of $\hM$ as a projection of matrix $\tM$, and then use the triangle inequality and the inequality \eqref{eq:2_to_inf_2_norm_ineq} to obtain: 
\begin{align}
    \|\hM - M \|_{2 \to \infty} 
    & = \| \Pi_{\widehat{U}}\tM - \Pi_U M \|_{2 \to \infty} \nonumber\\
    & = \| (\Pi_{\widehat{U}} - \Pi_U)(\tM - M) + \Pi_U (\tM - M ) + (\Pi_{\widehat{U}}- \Pi_U) M \|_{2 \to \infty} \nonumber\\
    & \leq \| (\Pi_{\widehat{U}} - \Pi_U)(\tM - M)\|_{2 \to \infty} + \| \Pi_U (\tM - M ) \|_{2 \to \infty} + \|(\Pi_{\widehat{U}}- \Pi_U) M \|_{2 \to \infty} \nonumber\\
    & \le \| \Pi_{\widehat{U}} - \Pi_U \|_{2 \to \infty} ( \| \tM - M \| + \| M  \| )  + \| \Pi_U \|_{2 \to \infty} \| \tM - M  \|. 
    \label{ineq:from_hF_to_proj_hU}
\end{align}
Moreover, we note that $\Vert \Pi_{U} \Vert_{2 \to \infty} = \Vert U  \Vert_{2 \to \infty}$ (refer to Proposition 6.6 in \cite{cape2019two}). In the remaining of the proof, we upper bound $\| \Pi_{\widehat{U}} - \Pi_U \|_{2 \to \infty}$ from \eqref{ineq:from_hF_to_proj_hU}. For any orthogonal matrix $R \in {\cal O}^{r\times r}$, we have 
\begin{align}
    \| \Pi_{\widehat{U}} - \Pi_U\|_{2 \to \infty} & = \| \widehat{U} \widehat{U}^\top - U U^\top  \|_{2 \to \infty} \nonumber\\
    & = \| \widehat{U} R R^\top \widehat{U}^\top - U R^\top \widehat{U}^\top + UR^\top \widehat{U}^\top    - U U^\top  \|_{2 \to \infty} \nonumber\\
    & \le \|\widehat{U} R R^\top \widehat{U}^\top - U R^\top \widehat{U}^\top\|_{2 \to \infty} + \| UR^\top \widehat{U}^\top    - U U^\top \|_{2 \to \infty} \nonumber\\
    & \le \| \widehat{U}R  - U \|_{2 \to \infty} \Vert R^\top \widehat{U}^\top \| + \Vert U \Vert_{2 \to \infty}  \Vert R^\top \hat{U}^\top -  U^\top \|   \nonumber \\
    & \le \| U - \widehat{U}R \|_{2 \to \infty} + \Vert U \Vert_{2 \to \infty} \Vert U -  \hU R \|. 
    \label{eq:puphu_uuh}
\end{align}
Recall the definition of $\sgn$ function given in the notation presented in \textsection\ref{sec:intro} and choose the matrix $R$ as $R = \sgn(\hU^\top U)$. For this choice of $R$ we have according to Davis-Kahan's theorem (Corollary 2.8 in \cite{chen2021spectral}):
\begin{align}
    \Vert U -  \hU R \| \leq \sqrt{2}  \|\sin \Theta(\hU,U)\| \leq \frac{2\|M-\tM\|}{\sigma_r(M)}.
    \label{eq:UhUR}
\end{align}
Define the matrix $\WhU=\hU^\top U$. We use the facts that
\begin{align}
\label{eq:U_hUR}
    \|U -  \hU R\|_{2\to\infty} \leq \|U - \hU\WhU\|_{2\to\infty} + \|\hU\|_{2\to\infty} \|\WhU-R\|
\end{align}
and that $\WhU$ is very close to $R$ according to the proof of Lemma 4.15 in \cite{chen2021spectral} to show:
\begin{align}
\label{eq:WhU_R}
    \|\WhU-R\| = \|\hU^\top U - \sgn(\hU^\top U)\| = \|\sin \Theta(\hU,U)\|^2 \leq \frac{2\|M-\tM\|^2}{\sigma_r^2(F)}.
\end{align}
We also have $\sigma_i(R)=1$ for $i\in[r]$ and according to Weyl's inequality $\sigma_{\min}(\WhU) \geq \sigma_{\min}(R) - \|\WhU - R\| = 1 - \|\WhU - R\|  $. Combining these results under assumption $\|M-\tM\| <  \sigma_r(M)/\sqrt{2}$ we obtain:
\begin{align*}
    \|\WhU^{-1}\|= \frac{1}{\sigma_{\min}(\WhU)} \leq \frac{1}{1-\|\WhU-R\|} \leq \frac{1}{1-\frac{2\|M-\tM\|^2}{\sigma_r^2(M)}}.
\end{align*}
Thus:
\begin{align}    \|\hU\|_{2\to\infty} 
    \leq
    \|\hU \WhU\|_{2\to\infty} \|\WhU^{-1}\|
    \leq
    \frac{1}{1-\frac{2\|M-\tM\|^2}{\sigma_r^2(M)}} (\|U\|_{2\to\infty} + \|U-\hU \WhU\|_{2\to\infty}).
    \label{eq:hU}
\end{align}
Combining Equations \eqref{eq:U_hUR}, \eqref{eq:WhU_R}, \eqref{eq:hU} we get:
\begin{align*}
    \|U -  \hU R\|_{2\to\infty} \leq
    \frac{1}{1-\frac{2\|M-\tM\|^2}{\sigma_r^2(M)}} (
    \|U-\hU \WhU\|_{2\to\infty} + \frac{2\|M-\tM\|^2}{\sigma_r^2(M)} \|U\|_{2\to\infty})
\end{align*}
and combining the last equality with \eqref{eq:puphu_uuh} and \eqref{eq:UhUR}, we have
\begin{align}
\label{eq:2inf_projU_diff}
    \| \Pi_{\widehat{U}} - \Pi_U\|_{2 \to \infty} \leq 
    \frac{ \|U-\hU \WhU\|_{2\to\infty}}{1-\frac{2\|M-\tM\|^2}{\sigma_r^2(M)}} + \left(
     \frac{\frac{2\|M-\tM\|^2}{\sigma_r^2(M)}}{1-\frac{2\|M-\tM\|^2}{\sigma_r^2(M)}} + \frac{2\|M-\tM\|}{\sigma_r(M)} \right) \|U\|_{2\to\infty}.
\end{align}
Finally, substituting the obtained bound into Equation \eqref{ineq:from_hF_to_proj_hU} and using assumption $\|M-\tM\| \leq c_1 \sigma_r(M)$ for simplification, we obtain the statement of the lemma.
% \begin{align*}
%     \| \hM -  M& \|_{2\to\infty} \leq
%     \sigma_1(M) \Bigg[ \frac{ 1 + \frac{ \|\tM - M \| }{\sigma_1(M)}}{1-\frac{2\|M-\tM\|^2}{\sigma_r^2(M)}} 
%     \|U-\hU \WhU\|_{2\to\infty} \\&+ \|U \|_{2 \to \infty}  \Bigg( \frac{\| \tM - M  \| }{\sigma_1(M)} + (\frac{\frac{2\|M-\tM\|^2}{\sigma_r^2(M)}}{1-\frac{2\|M-\tM\|^2}{\sigma_r^2(M)}} + \frac{\sqrt{2}\|M-\tM\|}{\sigma_r(M)} ) ( 1+\frac{\| \tM - M \|}{\sigma_1(M)}) \Bigg)   \Bigg]
% \end{align*}
% The statement of the lemma follows after using the assumption that $\sqrt{2}\|M-\tM\| \leq c_1 \sigma_r(M) \leq c_1 \sigma_1(M)$ and definition of incoherence constant.
\end{proof}

\subsection{Bounding \texorpdfstring{$\| P-\hP \|_{1\to\infty}$}{P1inf}}
\begin{lem}\label{lem:P-one-to-infty}
Let $P,\hP$ be as in Model II in \textsection \ref{sec:sampling}. We have:
    \begin{align*}
        \| \hP - P\|_{1\to\infty} \leq 2\frac{\sqrt{n} \|\hM - M\|_{2\to\infty}}{\min_{j\in[n]} \|M_{j,:}\|_1}.
    \end{align*}
\end{lem}
\begin{proof}
Starting with the definition of $\hP$, we get:
\begin{align*}
    \| \hP - P\|_{1\to\infty} & = \max_{i\in[n]} \| \hP_{i,:} -  P_{i,:} \|_1  =
    \max_{i\in[n]} \left\Vert \frac{(\hM_{i,:})_+}{\|(\hM_{i,:})_+\|_1 } - \frac{M_{i,:}}{\|M_{i,:}\|_1 } \right\Vert_1 \\
    & \leq
    2\max_{i\in[n]} \frac{ \|\hM_{i,:} - M_{i,:}\|_1}{\|M_{i,:}\|_1}\leq
    2\frac{\sqrt{n} \max_{i\in[n]} \|\hM_{i,:} - M_{i,:}\|}{\min_{j\in[n]} \|M_{j,:}\|_1 },
    %& =
    %2\sqrt{n}\|\nu^{-1}\|_{\infty} \|\hM - M\|_{2\to\infty}
    %\label{eq:phatp_fhatf}
\end{align*}
where the first inequality follows from Lemma 2 in \cite{zhang2019spectral} and the second by equivalence of norms. Moreover, note that the above inequality holds even in the case when $\|(\hM_{i,:})_+\|_1 = 0$ (and thus $\hP_{i,:} = \frac{1}{n} \onevec_n$), but the bound is vacuous in this case.
\end{proof}

\subsection{Bounding \texorpdfstring{$\| M-\hM \|_{\infty}$}{Minf}}
\begin{lem}\label{lem:M-infinity}
Let $M,\hM$ be as in \textsection \ref{sec:sampling}. Assume that there exists a sufficiently small universal constant $c_1>0$ such that $\|M-\tM\| \leq c_1 \sigma_r(M)$. Then, there exists a universal constant $c_2>0$ such that
    \begin{align*}
        \Vert \widehat{M} - M \Vert_{\infty}  \leq
        c_2 \| M\|_{2\to\infty} & \left(\frac{\|M-\tM\|}{\sigma_r(M)} \| V\|_{2\to\infty} + \|V-\hV\WhV \|_{2\to\infty} \right) \\
        &+ c_2\| M - \hM \|_{2\to\infty} (\| V\|_{2\to\infty} + \|V-\hV\WhV \|_{2\to\infty} ). 
    \end{align*}
\end{lem}
\begin{proof}
Similarly to the decomposition leading to Equation \eqref{ineq:from_hF_to_proj_hU}, we can upper bound the infinity norm error easily from the following decomposition: 
\begin{align*}
    \Vert \widehat{M} - M \Vert_{\infty} & = \Vert  \widehat{M} \Pi_{\widehat{V}} -  M  \Pi_V \Vert_\infty  \\
    & \le \Vert \widehat{M} \Vert_{2\to \infty} \Vert \Pi_{\widehat{V}} - \Pi_{V}\Vert_{2\to \infty}  +  \Vert \widehat{M} - M \Vert_{2\to\infty} \Vert \Pi_{V}\Vert_{2\to \infty} \\
    & \le  (\Vert \widehat{M} - M \Vert_{2 \to \infty} + \Vert M \Vert_{2\to\infty}) \Vert \Pi_{\widehat{V}} - \Pi_{V}\Vert_{2 \to  \infty} + \Vert \widehat{M} - M \Vert_{2\to\infty}\Vert V\Vert_{2\to \infty}, 
\end{align*}
where we used the inequality \eqref{eq:inf_ineq_2_to_inf_norm} together with the triangle inequalities and the fact that projection matrices are symmetric. To bound $\Vert \Pi_{\widehat{V}} - \Pi_{V}\Vert_{2 \to  \infty}$, we use the same approach as that used in \eqref{eq:2inf_projU_diff} (just replacing $U$ by $V$), and we obtain:
\begin{align*}
    \| \Pi_{\widehat{V}} - \Pi_V\|_{2 \to \infty} \leq 
    \frac{ \|V-\hV \WhV\|_{2\to\infty}}{1-\frac{2\|M-\tM\|^2}{\sigma_r^2(M)}} + \left(
     \frac{\frac{2\|M-\tM\|^2}{\sigma_r^2(M)}}{1-\frac{2\|M-\tM\|^2}{\sigma_r^2(M)}} + \frac{2\|M-\tM\|}{\sigma_r(M)} \right) \|V\|_{2\to\infty}.
\end{align*}
\end{proof}

\subsection{Bounding \texorpdfstring{$\| P-\hP \|_{\infty}$}{Pinf}}
\begin{lem}
\label{lemma:hP_infty_bound}
Let $P,\hP$ be as in Model II in \textsection \ref{sec:sampling}. Assume that ${\cal D} = \min_{i \in[n]} \Vert M_{i,:} \Vert_1 > 0$. If $\Vert  \hM -  M \Vert_{1 \to \infty} \le \frac{1}{2}{\cal D}$, then
\begin{align*}
    \|\hP-P\|_{\infty}
    % &\leq \max_{i \in[n]}\frac{1}{\Vert M_{i,:} \Vert_1} \Vert \hM - M\Vert_{\infty}  +  \max_{i\in[n]}\frac{2 \|M\|_{\infty}  }{\Vert M_{i,:} \Vert_{1}^2} \Vert \hM - M\Vert_{1 \to \infty}   + \max_{i\in[n]}\frac{2}{\Vert M_{i,:} \Vert_{1}^2 } \Vert \hM - M\Vert_{\infty} \Vert \hM - M\Vert_{1 \to \infty} \\
    % &\lesssim  \max_{i \in[n]}\frac{1}{\Vert M_{i,:} \Vert_1} \Vert \hM - M\Vert_{\infty}  +  \max_{i\in[n]}\frac{ \|M\|_{\infty}  \sqrt{n}}{\Vert M_{i,:} \Vert_{1}^2} \Vert \hM - M\Vert_{2 \to \infty}   + \max_{i\in[n]}\frac{\sqrt{n}}{\Vert M_{i,:} \Vert_{1}^2 } \Vert \hM - M\Vert_{\infty} \Vert \hM - M\Vert_{2 \to \infty} 
    &\leq  2\frac{\| \hM - M\|_{\infty} }{{\cal D}}   +  2\frac{ \sqrt{n}\|M\|_{\infty}  }{{\cal D}^2} \Vert \hM - M\Vert_{2 \to \infty}.
\end{align*}
\end{lem}
\begin{proof}
First note that for any $i\in[n]$:
\begin{align}
    \left \vert \Vert (\hM_{i,:})_{+}\Vert_{1} - \Vert M_{i,:}\Vert_1 \right \vert \le  \Vert  (\hM_{i,:})_{+} -  M_{i,:} \Vert_1 \le  \Vert  \hM_{i,:} -  M_{i,:} \Vert_1 \le \frac{\Vert M_{i,:} \Vert_1}{2},
    \label{eq:Pinf_proof_Mhat_lbound}
\end{align}
where the first inequality follows from the reverse triangle inequality, the second from $\vert \max(0, x) - y \vert \le \vert x - y\vert  $ for all $y > 0$ and $x \in \RR$, and the last inequality follows from the assumption in the lemma. This implies that $\Vert (\hM_{i,:})_{+}\Vert_{1} > 0$ for all $i\in [n]$, which further implies that  $\hP$ is defined by: 
for all $i\in [n]$,
\begin{align}
\hP_{i,:} = (\hM_{i,:})_+/ \| (\hM_{i,:})_+ \|_1.
\end{align}

Furthermore, we have for all $i,j = 1,\dots,n$,\begin{align*}
    \vert \widehat{P}_{i,j} - P_{i,j} \vert & = \left\vert \frac{(\hM_{i,j})_{+}}{\Vert (\hM_{i,:})_{+}\Vert_{1}} - \frac{M_{i,j}}{ \Vert M_{i,:}\Vert_1} \right\vert 
    \le \left\vert \frac{\hM_{i,j}}{\Vert (\hM_{i,:})_{+}\Vert_{1}} - \frac{M_{i,j}}{ \Vert M_{i,:}\Vert_1} \right\vert \\
    & \le  \frac{1}{\Vert M_{i,:} \Vert_1}\left\vert \hM_{i,j} - M_{i,j}   \right\vert +   \vert \hM_{i,j} \vert   \left\vert \frac{1}{\Vert (\hM_{i,:})_{+}\Vert_{1}} - \frac{1}{\Vert M_{i,:}\Vert_{1}}  \right\vert \\
    & =  \frac{1}{\Vert M_{i,:} \Vert_1} \left\vert \hM_{i,j} - M_{i,j}   \right\vert + \frac{\vert \hM_{i,j} \vert}{ \Vert M_{i,:} \Vert_1  }   \left\vert \frac{1}{ 1 + \frac{\Vert (\hM_{i,:})_{+}\Vert_{1} - \Vert M_{i,:}\Vert_1}{\Vert M_{i,:}\Vert_1}   } - 1 \right\vert \\
    &  = \frac{1}{\Vert M_{i,:} \Vert_1} \left\vert \hM_{i,j} - M_{i,j}   \right\vert + \frac{\vert \hM_{i,j} \vert}{ \Vert M_{i,:} \Vert_1}  \varphi\left( \frac{\Vert (\hM_{i,:})_{+}\Vert_{1} - \Vert M_{i,:}\Vert_1}{\Vert M_{i,:}\Vert_1} \right) 
\end{align*}
where we define $\varphi(x) = \vert x/(1+x) \vert $ for all $x \in \RR \backslash \lbrace   - 1 \rbrace$. Note that if $\vert x \vert < 1/2$, then $\varphi(x) \le 2 \vert x\vert$, which combined with \eqref{eq:Pinf_proof_Mhat_lbound} gives
% Thus, if  
% \begin{align*}
%     \left \vert \Vert (\hM_{i,:})_{+}\Vert_{1} - \Vert M_{i,:}\Vert_1 \right \vert \le  \Vert  (\hM_{i,:})_{+} -  M_{i,:} \Vert_1 \le  \Vert  \hM_{i,:} -  M_{i,:} \Vert_1 \le \frac{\Vert M_{i,:} \Vert_1}{2} 
% \end{align*} 
\begin{align*}
    \vert \widehat{P}_{i,j} - P_{i,j} \vert & \le \frac{1}{\Vert M_{i,:} \Vert_1} \left\vert \hM_{i,j} - M_{i,j}   \right\vert + \frac{ 2 \vert \hM_{i,j} \vert}{ \Vert M_{i,:} \Vert_1} \left \vert  \frac{\Vert (\hM_{i,:})_{+}\Vert_{1} - \Vert M_{i,:}\Vert_1}{\Vert M_{i,:}\Vert_1} \right \vert  \\
    & \le \frac{1}{\Vert M_{i,:} \Vert_1} \left\vert \hM_{i,j} - M_{i,j}   \right\vert + \frac{2}{ \Vert M_{i,:} \Vert_{1}^2} \left( \left\vert \hM_{i,j} - M_{i,j}   \right\vert + \left\vert M_{i,j}   \right\vert\right) \Vert \hM_{i,:} - M_{i,:} \Vert_{1}.  
\end{align*}
Using the assumption $ \Vert  \hM -  M \Vert_{1 \to \infty} \le \frac{1}{2} \min_{i\in[n]} \Vert M_{i,:} \Vert_1$ again, we can group first two terms, and then use $\| \hM - M\|_{1\to\infty} \leq \sqrt{n}\| \hM - M\|_{2\to\infty}$ to get the statement of the lemma.
% \begin{align*}
%     \Vert \hat{P} - P \Vert_\infty \le \max_{i \in[n]}\frac{2}{\Vert M_{i,:} \Vert_1} \Vert \hM - M\Vert_{\infty}  +  \max_{i\in[n]}\frac{2 \|M\|_{\infty} }{\Vert M_{i,:} \Vert_{1}^2} \Vert \hM - M\Vert_{1 \to \infty}  
% \end{align*}

\end{proof}
\newpage
\section{Low-rank bandits: proofs of results from Section \ref{sec:bandits}}\label{app:bandits}

\subsection{Gap-dependent guarantees} \label{app:bandits-1}

\begin{proof}[Proof of Theorem \ref{thm:regret}] First, we prove the result corresponding the best entry  identification problem. We proceed in several steps.

\textbf{Step 1: entry-wise concentration.} We can easily verify that for all $\ell \ge 1$, for all $(i,j) \in [m] \times [n]$, we have 
\begin{align*}
    \vert \widehat{\Delta}_{i,j}^{(\ell)} - \Delta_{i,j} \vert \le 2 \Vert \widehat{M}^{(\ell)} - M^\star \Vert_\infty. 
\end{align*}
Therefore, applying Theorem \ref{thm:reward}, we have, for $\delta > 0$, and $T_\ell \ge c_1 (m+n) \log^3( (e^2(m+n) (mn) / \delta_\ell  )$, 
\begin{align*}
    \PP\left(  \vert \widehat{\Delta}_{i,j} - \Delta_{i,j} \vert > 2 C_1 \sqrt{ \frac{e(m+n)}{T_\ell} \log^3\left(\frac{e(m+n)mn T_\ell}{\delta_\ell}\right) }    \right) \le \frac{\delta_\ell}{mn} 
\end{align*}
for some $c_1, C_1 >0$ sufficiently large. In particular, we can choose  $C_1 = C (\mu^{11/2}\kappa^2 r^{1/2} + \mu^3 \kappa r^{3/2} (m+n)/\sqrt{mn})$, and $c_1 = c \mu^4 \kappa^2 r^2$, but under a homogeneous reward matrix these constants are $\Theta(1)$. Thus, by a union bound and always under the same conditions, we have 
\begin{align*}
    \PP\left(  \max_{(i,j) \in [m]\times [n]}\vert \widehat{\Delta}_{i,j} - \Delta_{i,j} \vert > 2 C_1 \sqrt{ \frac{e(m+n)}{T_\ell} \log^3\left(\frac{e(m+n)mn T_\ell}{\delta_\ell}\right) }    \right) \le \delta_\ell.
\end{align*}
Next, we wish to choose $T_\ell$ so that we have 
\begin{align}\label{eq:good-conc}
    \PP\left(  \max_{i,j}  \vert \widehat{\Delta}_{i,j} - \Delta_{i,j} \vert \le  2^{-(\ell + 2)}   \right) \ge 1 -\delta_\ell. 
\end{align} 
Note that in order for the above guarantee to hold, it is sufficient to have: 
\begin{align*}
     T_\ell & \ge c_1 (m+n) \log^3\left( \frac{e^2(m+n) (mn)}{\delta_\ell} \right), \\
     T_\ell & \ge 2\sqrt{e} C_1^2 (m+n) 2^{2(\ell - 2)} \log^3\left( \frac{e(m+n)(mn)}{\delta_\ell} \right).
\end{align*}
This can be achieved if we choose 
\begin{align}\label{eq:good-condition}
    T_\ell = \left\lceil  C_3 2^{2(\ell - 2)}(m+n) \log^3\left( \frac{2^{2(\ell - 2)} (m+n)}{\delta_\ell}\right)   \right\rceil, 
\end{align}
for some positive constant $C_3>0$ large enough which can be determined explicitly and only depend on $c_1, C_1$. Indeed, this can be deduced from the basic fact that if $T_\ell^{1/3} \ge 2 a \log(2a) + 2b$, then $T_\ell^{1/3} \ge a\log(T_\ell^{1/3}) + b$. We spare the reader these tedious calculations and only argue that such $C_3$ exists and can be computed explicitly. 

\textbf{Step 2: Good events.} We define $S_\ell = \left\lbrace (i,j)\in [n]\times[m]: \Delta_{i,j} \le 2^{-\ell} \right\rbrace$ and the good events under which we correctly find the best entry as
\begin{align*}
    \cE_\ell = \lbrace \cA_{\ell+1} \subseteq S_{\ell+1} \rbrace \cap \lbrace  (i^\star, j^\star) \in \cA_{\ell+1} \rbrace. 
\end{align*}
We show that the good event $\cE_\ell$ happens with high probability conditionally on $\cE_1, \dots, \cE_{\ell-1}$. 
Observe that by independence of the entries sampled at epoch $\ell$ from those of the previous epochs, we have based on \eqref{eq:good-conc} 
\begin{align*}
    \PP\left(  \max_{i,j}  \vert \widehat{\Delta}_{i,j} - \Delta_{i,j} \vert \le  2^{-(\ell + 2)}   \Big\vert \cE_{\ell - 1}, \dots, \cE_{1} \right) \ge 1 - \delta_\ell 
\end{align*}
Now, conditionally on $\cE_{\ell - 1}, \dots, \cE_{1}$, under the event that $\max_{i,j}  \vert \widehat{\Delta}_{i,j} - \Delta_{i,j} \vert \le  2^{-(\ell + 2)}$, if $(i,j) \in S_{\ell+1}^c \cap \cA_{\ell +1}$ then 
$$
\widehat{\Delta}_{i,j}^{(\ell)} \ge \Delta_{i,j} - 2^{-(\ell +2)} > 2^{-(\ell+1)} - 2^{-(\ell +2)} = 2^{-(\ell +2)}.
$$
Thus, we have 
\begin{align*}
      \PP\left( \cA_{\ell+1} \subseteq S_{\ell+1}  \Big\vert \cE_{\ell - 1}, \dots, \cE_{1} \right) \ge \PP\left(  \max_{i,j}  \vert \widehat{\Delta}_{i,j} - \Delta_{i,j} \vert \le  2^{-(\ell + 2)}   \Big\vert \cE_{\ell - 1}, \dots, \cE_{1} \right) \ge 1 - \delta_\ell. 
\end{align*}
Furthermore, note that under the event $\max_{i,j}  \vert \widehat{\Delta}_{i,j} - \Delta_{i,j} \vert \le  2^{-(\ell + 2)}$, we clearly have that $\widehat{\Delta}_{i^\star,j^\star} \le 2^{-(\ell + 2)}$ and since $(i^\star, j^\star) \in \cA_\ell $ conditionally on $\cE_{\ell -1}$ we conclude that 
\begin{align*}
    \PP\left( \cE_{\ell}  \Big\vert \cE_{\ell - 1}, \dots, \cE_{1} \right) \ge \PP\left(  \max_{i,j}  \vert \widehat{\Delta}_{i,j} - \Delta_{i,j} \vert \le  2^{-(\ell + 2)}   \Big\vert \cE_{\ell - 1}, \dots, \cE_{1} \right) \ge 1 - \delta_\ell 
\end{align*}

\textbf{Step 3: Sample complexity.} First, we remark that when $\ell = \lceil \log_2(1/\Delta_{\min}) \rceil$, we have  $S_\ell = \lbrace (i^\star, j^\star)\rbrace$. Therefore, under the event 
\begin{align*}
    \cE_1 \cap \dots \cap \cE_{\lceil \log_2(1/\Delta_{\min}) \rceil}
\end{align*}
the algorithm will stop after $\tau$ rounds, and recommend the optimal $(i^\star, j^\star)$, where
\begin{align*}
    \tau & \le \sum_{\ell=1}^{\lceil \log_2(1/\Delta_{\min})  \rceil } T_\ell \\
    & \le \sum_{\ell=1}^{\lceil \log_2(1/\Delta_{\min})   \rceil }  \left\lceil  C_3 2^{2(\ell - 2)}(m+n) \log^3\left( \frac{2^{2(\ell - 2)} (m+n)}{\delta_\ell}\right)   \right\rceil \\
    & \le \sum_{\ell=1}^{\lceil \log_2(1/\Delta_{\min})   \rceil }\left\lceil  C_3 \frac{(m+n)}{\Delta_{\min}^2} \log^3\left( \frac{ (m+n) \lceil \log_2\left(1/\Delta_{\min}\right) \rceil^2 }{\Delta_{\min}^2\delta}\right)   \right\rceil \\
    & \le \log_2\left(\frac{1}{\Delta_{\min}}\right)\left\lceil  C_3 \frac{(m+n)}{\Delta_{\min}^2} \log^3\left( \frac{ (m+n) \lceil \log_2\left(1/\Delta_{\min}\right) \rceil^2 }{\Delta_{\min}^2\delta}\right)   \right\rceil \\
    & \le \psi(n,m,\delta):=  C_4  \frac{(m+n) \log\left(e/\Delta_{\min}\right) }{\Delta_{\min}^2} \log^3\left( \frac{ e(m+n)  \log\left(e/\Delta_{\min}\right)  }{\Delta_{\min}\delta}\right)    
\end{align*}
where we recall the definition of $T_\ell$ in \eqref{eq:good-condition}, that $\delta_\ell = \delta / \ell^2$, and where $C_4$ is a large enough universal constant. Hence, we have  
\begin{align}\label{eq:guarantee-bei}
     \PP\left( (i_\tau, j_\tau) = (i^\star,j^\star),  \tau \le  \psi(n,m,\delta) \right)  \ge \PP\left( \bigcap_{\ell = 1}^{\lceil \log_2(1/\Delta_{\min})  \rceil} \cE_{\ell}\right) \ge 1-\delta. 
\end{align}
This conclude the proof of the guarantee for the best entry identification. Note that we can immediately conclude from the above guarantee \eqref{eq:guarantee-bei} that the sample complexity of SME-AE$(1/T^\alpha)$ for all $T \ge 1$, satisfies $\EE[\tau \wedge T] \le \psi(n,m, T^{-\alpha}) + T^{1-\alpha}$. Indeed, we have 
\begin{align*}
    \EE[\tau \wedge T] & = \EE[(\tau \wedge T) \indicator_{\left\lbrace \tau \le \psi(n,m,T^{-\alpha}) \right\rbrace} ] + \EE[(\tau \wedge T) \indicator_{\left\lbrace \tau > \psi(n,m,T^{-\alpha}) \right\rbrace} ] \\
    & \le \psi(n,m,T^{-\alpha}) + T \PP(\tau > \psi(n,m,T^{-\alpha}) ) \\
    &\le \psi(n,m,T^{-\alpha}) + T^{1-\alpha},
\end{align*}
where the upper bound on the probability follows from \eqref{eq:guarantee-bei} with $\delta = 1/T^\alpha$. 

\medskip 

Next, we turn our attention to proving the regret upper bound. We define $\cE_{good} = \lbrace (\hat{\imath}_\tau, \hat{\jmath}_\tau) = (i^\star,j^\star) ,\tau \le \psi(n,m,1/T^2) \rbrace$. We have 
\begin{align*}
    R^\pi(T) & = T M_{i^\star, j^\star} - \EE\left[ \sum_{t=1}^T M_{i_t^\pi,j_t^\pi} \right]  \\ 
    & =  \EE\left[   \sum_{t=1}^T (M_{i^\star, j^\star} - M_{i_t^\pi,j_t^\pi}) \indicator_{\lbrace \cE_{good} \rbrace}  \right] + \EE\left[   \sum_{t=1}^T (M_{i^\star, j^\star} - M_{i_t^\pi,j_t^\pi}) \indicator_{\lbrace \cE_{good}^c \rbrace}  \right] \\
    & \le \EE\left[   \sum_{t=1}^T (M_{i^\star, j^\star} - M_{i_t^\pi,j_t^\pi}) \indicator_{\lbrace \tau \le \psi(n,m,T^{-2}) \rbrace}  \right] +  \Delta_{\max}T \PP( \cE_{good}^c) \\
    & \le \EE\left[   \sum_{t=1}^{\infty} (M_{i^\star, j^\star} - M_{i_t^\pi,j_t^\pi})  \indicator_{\lbrace \tau \wedge \psi(n,m,T^{-2}) > t\rbrace} \right] + \frac{\Delta_{\max}}{T}  \\
    & \le \EE\left[   \sum_{t=1}^{\infty} \bar{\Delta} \indicator_{\lbrace \tau \wedge \psi(n,m,T^{-2}) > t\rbrace} \right] + \frac{\Delta_{\max}}{T}\\
    & \le \bar{\Delta} \psi(n,m,T^{-2}) + \frac{\Delta_{\max}}{T}
\end{align*}
where in the second to last inequality, we used the tower rule together with the observation that  $\EE[ (M_{i^\star, j^\star} - M_{i_t^\pi,j_t^\pi})  \indicator_{\lbrace \tau \wedge \psi(n,m,T^{-2}) > t\rbrace} \vert \cF_{t-1} ] =  \bar{\Delta}\indicator_{\lbrace \tau \wedge \psi(n,m,T^{-2}) > t\rbrace}$ where $\cF_{t-1}$ is the $\sigma$-algebra defined by the observations up to time $t-1$. This concludes the proof.
\end{proof}

\subsection{Gap-independent guarantees}

An immediate consequence of the regret bound in Theorem \ref{thm:regret} is that we can have a gap-independent bound under some additional assumption. Let us define $\zeta = \Delta_{\max}/\Delta_{\min}$, then the regret bound becomes 
\begin{align}\label{eq:reg-bound1}
    R^{\pi}(T) \le    \frac{\zeta C_4(m+n) \log\left(e/\Delta_{\min}\right) }{\Delta_{\min}} \log^3\left( \frac{ e(m+n)  \log\left(e/\Delta_{\min}\right)  T^2 }{\Delta_{\min}}\right)  + \frac{\Delta_{\max}}{T}.
\end{align}
At the same time, we also have the worst case bound 
\begin{align}\label{eq:reg-bound2}
    R^{\pi}(T) \le \zeta \Delta_{\min} T. 
\end{align}
Taking the best of the two bounds \eqref{eq:reg-bound1} and \eqref{eq:reg-bound2} with the worst case choice for $\Delta_{\min}$, we get 
\begin{align*}
    R^{\pi}(T) =  \tilde{O}\left(\zeta \sqrt{(n+m) T} \log^2((n+m)T)\right)
\end{align*}
where the $\tilde{O}$ hides additional log-log terms in $m,n$ and $T$.

% Indeed, the above bound can be achieved by choosing $\Delta_{\min} = \sqrt{(m+n)\log^3((m+n)T)/T}$

\newpage
\section{Related work}
\label{app:related}

In this section, we first discuss the results for the estimation of a low-rank transition matrix presented in \cite{zhang2019spectral}. We then give a more detailed account of the related work for low-rank bandits.

\subsection{Low-rank transition matrix estimation}

In \cite{zhang2019spectral}, the authors try to estimate a low-rank transition matrix from the data consisting of a single trajectory of the corresponding Markov chain. In a sense, this objective is similar to ours in Model II(b). The main results of \cite{zhang2019spectral} are presented in Theorem 1. First observe that our results are more precise since we manage to get entry-wise guarantees. Then it is also worth noting that, in the case of homogenous transition matrices, the upper bound on $\| \hP-P\|_{1\to\infty}$ stated in Theorem 1 in \cite{zhang2019spectral} are similar to the upper bounds we establish in Corollary \ref{corr:homogeneous_model3}. However, to obtain such bounds, we believe that it is necessary to first derive guarantees for the singular subspace recovery in the $\ell_{2\to\infty}$ norm, as we do. The authors of \cite{zhang2019spectral} do not present any step with such guarantees for the estimation of the singular subspaces. We explain below why this step is missing and where the analysis towards the upper bound $\| \hP-P\|_{1\to\infty}$ breaks in \cite{zhang2019spectral}. 

{\bf Proof of the guarantees for $\| \hP-P\|_{1\to\infty}$ in \cite{zhang2019spectral}.} Note that in \cite{zhang2019spectral}, the authors use $F$ in lieu of $M$. We keep our notation $M$ below to be consistent with the rest of the manuscript. In the proof of Theorem 1 in \cite{zhang2019spectral}, the authors use the following decomposition:
\begin{align}
    \|\hM_{i,:} - M_{i,:}\| \leq \|(\hM_{i,:} - M_{i,:})V\| + (\|\hM_{i,:} - M_{i,:}\| + \|M_{i,:}\|) \frac{C\| \tM-M\| } {\sigma_r(M)}.
    \label{eq:Mengdi_decomposition_initial}
\end{align}
They apply concentration results on $\|(\tM - M)V\|_{2\to\infty}$ (Lemma 8) and $\|M-\tM\|$ (Lemma 7) to bound the two terms from above. More precisely, their proof includes (33) page 3217, a sequence of inequalities where these concentration results are used. In the fifth line of (33), the authors apply (31), the concentration result on $\|(\tM - M)V\|_{2\to\infty}$, but to bound $\|(\hM - M)V\|_{2\to\infty}$ instead. Replacing $\hM$ by $\tM$ is however not possible, and the analysis breaks here.  

{\bf Is there a simple solution?} We argue below that it is not easy to solve the aforementioned issue in the proof. We first claim that the two concentration bounds on $\|(\tM - M)V\|_{2\to\infty}$ and $\|M-\tM\|$ are not sufficient for bounding the first term from Equation \eqref{eq:Mengdi_decomposition_initial}. Specifically, for any row $i$:
\begin{align*}
    \|(\hM_{i,:} - M_{i,:})V\| = \|(\tM_{i,:}\hV\hV^\top - M_{i,:})V \| = \| (\tM_{i,:} - M_{i,:})V + \tM_{i,:}(\hV\hV^\top - V V^\top) V \| ,
\end{align*}
and in order to analyze the second term inside the norm, we need to deal with dependence between $\tM$ and $\hV$. Doing this naively using the triangle inequality and Cauchy-Schwarz inequality yields:
\begin{align}
    \| (\hM - M)V \|_{2\to\infty} &\leq \| (\tM - M)V \|_{2\to\infty}  + \| \tM(\hV\hV^\top - V V^\top) V \|_{2\to\infty} \nonumber \\ &\leq \| (\tM - M)V \|_{2\to\infty}  + \| \tM\|_{1\to\infty} \| V-\hV(\hV^\top V)\|_{2\to\infty}
    \label{eq:mengdi_explanation_hF}.
\end{align}
It is not clear how bounds on $\| (\tM - M)V \|_{2\to\infty}$ and $\| M-\tM \|$ imply a bound on $\| (\hM- M)V \|_{2\to\infty}$ since term $\| V-\hV(\hV^\top V)\|_{2\to\infty}$ does not seem to be directly bounded by these two terms. We can think of bounding $\| V-\hV(\hV^\top V) \|_{2}$ using Davis-Kahan's inequality:
\begin{align*}
    \| V-\hV(\hV^\top V)\|_{2\to\infty} \leq \| V-\hV(\hV^\top V)\|_{2} \lesssim \frac{\| M-\tM\|}{ \sigma_r(M)},
\end{align*}
where we neglect the higher order term  (see Equations \eqref{eq:UhUR},\eqref{eq:U_hUR},\eqref{eq:WhU_R}). Then, with the upper bound on $\| M-\tM\|$, we may obtain an upper bound on $\| \hP-P\|_{1\to\infty}$ but that does not have a fast decaying rate as that claimed in Theorem 1 in \cite{zhang2019spectral} or in our main theorems.

It is also worth noting that assuming proof of Theorem 1 in \cite{zhang2019spectral} holds or that more specifically, the series of inequalities leading to Equation (33) holds, one could greatly simplify the singular subspace recovery problem. In particular, since
\begin{align*}
    \| \tM (V-\hV\hV^\top V) \|_{2\to\infty} = \|(\tM-\hM)V\|_{2\to\infty}
    \leq 
    \| ({\hM} - M)V \|_{2\to\infty} + \|EV \|_{2\to\infty}
\end{align*}
we can rewrite Equation \eqref{eq:U_UWU_lemma416_poisson_transitions} (wlog for symmetric matrix $M$ with eigenvector matrix $V$) as:
\begin{align}
   \|V&- \hV(\hV^\top V) \|_{2\to\infty} 
    \nonumber \\ &\leq
    \frac{1}{\sigma_r(M)} \left( (2+\frac{4\| E\|}{\sigma_r(M)}) \| EV \|_{2\to\infty} + 2\| ({\hM} - M)V \|_{2\to\infty} + \frac{4\| MV\|_{2\to\infty} \|E\|}{\sigma_r(M)}\right).
    \label{eq:mengdi_explanation_subspace}
\end{align}
Now if (33) in \cite{zhang2019spectral} was true, we could use the correspoding bound of the critical term $\|({\hM} - M)V \|_{2\to\infty}$. This would not only greatly simplify proofs given in literature based on leave-one-out-technique, but also extend their work to Markov dependent random variables (which has not been done before). Lastly, note that we cannot skip estimation of singular subspaces by combining Equation \eqref{eq:mengdi_explanation_hF} and \eqref{eq:mengdi_explanation_subspace} since inequality $2\| \tM \|_{1\to\infty} < \sigma_r(M)$ does not hold in general.

\subsection{Low rank bandits}

Here we survey models for low-rank bandits that have emerged recently in the literature but that are not directly related to our model. Nonetheless our guarantees can be exported there.

\cite{jun2019bilinear} considers a bi-linear bandit model which seems more general than that of considered \cite{bayati2022speed}. Indeed, they assume that the observed reward in round $t$ after selecting a pair $(x,z) \in \cX \times \cZ$, is $x^\top \Theta z + \xi_t$ where $\cX \subset \RR^{m}$ and $\cZ \subset \RR^{n}$ are finite. They assume that $\Theta \in \RR^{m\times n}$ is low rank. If we assume that $\cX = \lbrace e_1, \dots, e_m\rbrace$  and $\cZ = \lbrace e_1, \dots, e_n\rbrace$, we then recover our model and that of \cite{bayati2022speed} with $M = \Theta$. However, we can also argue that if we restrict our attention to $m$ vectors from $\cX$, say $\cX' = \lbrace x_1, \dots, x_m\rbrace$, that span $\RR^m$, and $n$ vectors from $\cZ$, say $\cZ' = \lbrace z_1, \dots, z_n\rbrace$, that span $\RR^n$, then in our model and that of \cite{bayati2022speed}, $M_{i,j} = x_i^\top \Theta z_j$, for all $(i,j)\in [m]\times [n]$. Note that in this case, the rank of $M$ is equal to that of $\Theta$. In fact, in the first phase of  the algorithm proposed by \cite{jun2019bilinear}, the authors also restrict their attention to sets $\cX'$ and $\cZ'$ such that $\lambda_{\min}(\sum_{i=1}^m x_i x_i^{\top})$ and $\lambda_{\min}(\sum_{i=1}^m z_i z_i^{\top})$ are maximized. To simplify our exposition, we do not use the model presented by \cite{jun2019bilinear}, instead we use that of \cite{bayati2022speed}.

\cite{kang2022efficient} considers a generalized bandit framework with low rank structure which is rather different than the bandit framework we consider. There, the algorithm is based on the two stage idea introduced by  \cite{jun2019bilinear}, which consists in first estimating the subspace, then reducing the problem to a low-dimensional linear bandit with ambient dimension $nm$ but with roughly $n+m$ relevant dimensions. They are able obtain a minimax regret scaling as $(n+m)\sqrt{T}$. It is worth noting that both these works do not have gap-dependent bounds.

\cite{jain2022online} is another relevant work. There, the authors consider a low-rank bandit problem similar to ours but slightly more restrictive. At time $t$, they recommend an arm $\rho(j)_t$ for each user $j$, and they observe the corresponds rewards. In other words they observe $m$ entries per round, while in our case we only observe one entry per round. They show that with an explore-then-commit algorithm, they attain a regret of order $\textrm{polylog}(n+m) T^{2/3}$. Their regret guarantees require an entry-wise matrix estimation guarantee with scaling comparable to ours. They use the result of Chen et al. \cite{chen2020} which again is only valid for independent entries and does not account for repetitive sampling. To remedy that they rely on ad-hoc pre-processing steps (see remarks 2, 3 and 4 in \cite{jain2022online}).  In our case, we believe that our matrix estimation guarantees can be immediately used in their setting and this would lead to a regret scaling of order $(n+m)^{1/3}T^{2/3}$ with the more reasonable constraint that we only observe one entry at each round.  The authors also obtain an $\textrm{polylog}(n+m)\sqrt{T}$ guarantee but for rank-1 reward matrices only.

%\cite{kang2022efficient} consider a more general framework but rely on similar ideas as in \cite{jun2019bilinear}, and in comparison,  manage to improve the regret to an order of $(n+m)\sqrt{T}$.  

%Finally,  \cite{jain2022online} also consider a bandit framework with low rank structure for user item recommendation. Their algorithm design also requires entry-wise matrix estimation guarantees.  They use the result of  \cite{chen2020} meant for matrix estimation with nuclear norm penalization. However, this method does not account for repeated sampling. As such, they end up relying on ad-hoc procedures resulting in a complex model (see remark 2, 3 and 4 in \cite{jain2022online}). However, their data collection requires $m$ observations per round which differs from our setting. 

% The embarrassing truth about these algorithms \cite{jun2019bilinear} and \cite{kang2022efficient} is that if used in our setting, then there minimax regret guarantee is comparable to that of a UCB algorithm that does not explore further the structure which would suffer $\sqrt{nm T}$ when $n = \Theta(m)$.        

%\newpage
%\input{rebuttal}

\end{document}